%% file: main.tex
\documentclass{article}

\usepackage{microtype}
\usepackage{graphicx}
\usepackage{booktabs} 

\usepackage{hyperref}

\usepackage[accepted]{icml2023}

\usepackage{amsmath}
\usepackage{amssymb}
\usepackage{mathtools}

\usepackage[capitalize,noabbrev]{cleveref}

\usepackage{amsmath,wrapfig}
\input{latex_resources/mysymbol.sty}
\input{latex_resources/juan_defs.sty}
\usepackage{caption}
\usepackage{subcaption}
\usepackage{pgfplots}
\usepackage{enumitem}

\newcommand{\QED}{\hfill\blacksquare}

\newcommand{\elambda}{\hat{\lambda}}
\newcommand{\emu}{\hat{\mu}}
\newcommand{\eL}{\hat{L}}
\newcommand{\eD}{\hat{D}}
\newcommand{\etheta}{\hat{\theta}}
\newcommand{\eTheta}{\hat{\Theta}}

\newcommand{\fP}{\tilde{P}}

\newcommand{\gradLip}{G}
\newcommand{\gradBound}{R}

\usepackage[textsize=tiny]{todonotes}

\icmltitlerunning{Learning Globally Smooth Functions on Manifolds}

\begin{document}

\twocolumn[
\icmltitle{Learning Globally Smooth Functions on Manifolds}




\begin{icmlauthorlist}
\icmlauthor{Juan Cervi\~no}{upenn}
\icmlauthor{Luiz F.O. Chamon}{stut}
\icmlauthor{Benjamin D. Haeffele}{jhu}
\icmlauthor{Rene Vidal}{upenn}
\icmlauthor{Alejandro Ribeiro}{upenn}
\end{icmlauthorlist}
\icmlaffiliation{upenn}{University of Pennsylvania}
\icmlaffiliation{stut}{University of Stuttgart}
\icmlaffiliation{jhu}{Johns Hopkins University}

\icmlcorrespondingauthor{Juan Cervi\~no}{jcervino@seas.upenn.edu}

\icmlkeywords{Machine Learning, ICML}

\vskip 0.3in
]



\printAffiliationsAndNotice{\icmlEqualContribution} 

\begin{abstract} 
Smoothness and low dimensional structures play central roles in improving generalization and stability in learning and statistics. This work combines techniques from semi-infinite constrained learning and manifold regularization to learn representations that are globally smooth on a manifold. To do so, it shows that under typical conditions the problem of learning a Lipschitz continuous function on a manifold is equivalent to a dynamically weighted manifold regularization problem. This observation leads to a practical algorithm based on a weighted Laplacian penalty whose weights are adapted using stochastic gradient techniques. It is shown that under mild conditions, this method estimates the Lipschitz constant of the solution, learning a globally smooth solution as a byproduct. Experiments on real world data illustrate the advantages of the proposed method relative to existing alternatives.
\end{abstract}



\input{alejandro_intro.tex}


\subsection*{Related Work}
This paper is at the intersection of learning with Lipschitz constant constraints \citep{oberman2018lipschitz, finlay2018improved, couellan2021coupling, pauli2021training, bungert2021clip, miyato2018spectral, zhao2020spectral, krishnan2020lipschitz, shi2019neural, lindemann2021learning, arghal2021robust} and manifold regularization \citep{belkin2005manifold, niyogi2013manifold, li2022neural, hein2005graphs, belkin2005towards}. Relative to learning with Lipschitz constraints, we offer the ability to leverage data manifolds. Since data manifolds are often obtained from unlabeled data, as in \cite{kejani2020graph, belkin2004semi, jiang2019semi, kipf2016semi, yang2016revisiting, zhu2005semi, lecouat2018semi, ouali2020overview, cabannes2021overcoming} we also use point-cloud Laplacian techniques to compute the integral of the norm of the gradient. 

Relative to the literature on manifold regularization, we offer global smoothness assurances instead of an average penalty of large manifold gradients.
Similar to us, \citep{krishnan2020lipschitz} poses the problem of minimizing a Lipschitz constant. However, they utilize a softer surrogate (i.e. $p$-norm loss) which is a smoother version of the Lipschitz constant. Their approach therefore, tradeoffs numerical stability (small $p$) with accurate Lipschitz constant estimation ($p=\infty$). We do not work with surrogates and seek to minimize the maximum norm of the gradient utilizing an epigraph technique.

\section{Globally Constraining Manifold Lipschitz Constants}\label{sec_problem_formulation}

We consider data pairs $(x,y)$ in which the input features $x\in \ccalM \subset \reals^D$ lie in a compact oriented Riemannian manifold $\ccalM$ and the output features are real valued $y\in\reals$. We study the regression problem of finding a function $f_\theta:\ccalM \to \reals$, parameterized by $\theta\in\Theta\subset\reals^Q$ that minimizes the expectation of a nonnegative loss $\ell:\reals\times\reals\to\reals_+$, where $\ell(f_\theta(x),y)$ represents the loss of predicting output $f_\theta(x)$ when the world realizes the pair $(x,y)$. Data pairs $(x,y)$ are drawn according to an unknown probability distribution $p(x,y)$ on $\ccalM \times \reals$ which we can factor as  $p(x,y)=p(x)p(y|x)$.

We are interested in learning smooth functions, i.e. functions with \emph{controlled} variability over the manifold $\ccalM$. We therefore let $\nabla_\ccalM f_\theta(x)$ represent the manifold gradient of $f_\theta$ and introduce the following definition.



\begin{definition}[Manifold Lipschitz Constant]\label{def_manifold_lipschitz}
	Given a Riemannian manifold $\ccalM$, the function $f_\theta : \ccalM \to \reals$ is said to be $L$-Lipschitz continuous if there exists a strictly positive constant $L>0$ such that for all pairs of points $x_1,x_2 \in \ccalM$, 
	\begin{align}\label{eqn_def_manifold_lipschitz_10}
		|f_\theta(x_1)-f_\theta(x_2)| \leq L\, d_{\ccalM} (x_1,x_2),
	\end{align}
	where $d_\ccalM (x_1,x_2)$ denotes the distance between $x_1$ and $x_2$ in the manifold $\ccalM$. If the function $f_\theta$ is differentiable on the manifold, \eqref{eqn_def_manifold_lipschitz_10} is equivalent to requiring the gradient norm to be bounded by~$L$,
	\begin{align}\label{eqn_def_manifold_lipschitz_20}
		\| \nabla_\ccalM f_\theta(x) \| 
		~ &:=  ~ \lim_{\delta \to 0} ~
		\sup_{ \substack{ x' \in \ccalM ~:~ x' \neq x, \\ d_\ccalM (x,x') \leq \delta}} ~
		\frac{|f_\theta(x)-f_\theta(x')|}{d_\ccalM(x,x')} \nonumber \\
		 &\leq ~ L , \text{\quad for all~}  x \in \ccalM.
	\end{align}
\end{definition}


With Definition \ref{def_manifold_lipschitz} in place and restricting attention to differentiable functions $f_\theta$, our stated goal of learning functions $f_\theta$ with controlled variability over the manifold $\ccalM$ can be written as
\begin{align}
	\label{prob:manifold_lipschitz}
	P^* = \ \ \ \ \underset{\theta\in\Theta,\rho\geq 0}{\min}  \quad   
	& \rho, \\
	\text{subject to}  \quad 
	& \mbE_{p(x,y)} [\ell \left(f_\theta(x),y\right)]\leq \epsilon,\nonumber \\
	& \| \nabla_\ccalM f_\theta(z) \|^2 \leq \rho , p(z)\text{-a.e.},  z \in \ccalM \nonumber. 
\end{align}
In this formulation, the statistical loss $\mbE_{p(x,y)} [\ell \left(f_\theta(x),y\right)]$ is required to be below a target level $\eps$. Of all functions $f_\theta$ that satisfy this loss requirement, Problem \eqref{prob:manifold_lipschitz} defines as optimal those that have the smallest Lipschitz constant $L=\sqrt{\rho}$.

The goal of this paper is to develop methodologies to solve~\eqref{prob:manifold_lipschitz} when the data distribution and the manifold are unknown. To characterize the distribution, we are given sample pairs $(x_n,y_n)$ drawn independently from the joint distribution $p(x,y)$. To characterize the manifold, we are given samples $z_n$ drawn from the marginal distribution $p(x)$. This includes the samples $x_n$ from the (labeled) data pairs  $(x_n,y_n)$, but may also include additionally~(unlabeled) samples $z_n$.

Observe from \eqref{prob:manifold_lipschitz} that the problems of manifold regularization \citep{belkin2005manifold, niyogi2013manifold, kejani2020graph, li2022neural} and Lipschitz constant control \citep{oberman2018lipschitz, finlay2018improved,couellan2021coupling, finlay2018lipschitz, pauli2021training, krishnan2020lipschitz, shi2019neural, lindemann2021learning, arghal2021robust} are related. This connection is important to understand the merit of \eqref{prob:manifold_lipschitz}. To explain this better observe that there are three motivations for the problem formulation in \eqref{prob:manifold_lipschitz}: (i) it is often the case that if samples $x_1$ and $x_2$ are close, then the conditional distributions $p(y \mid x_1)$ and $p(y \mid x_2)$ are close as well. A function $f_\theta$ with small Lipschitz constant leverages this property, (ii) the Lipschitz constant of $f_\theta$ is \emph{guaranteed} to be smaller than $L=\sqrt{\rho}$, this provides advantages in, e.g., physical systems where Lipschitz constant guarantees translates to stability and safety assurances, (iii) it leverages the intrinsic low-dimensional structure of the manifold $\ccalM$ embedded in the ambient space. In particular, this permits taking advantage of unlabeled data. 

Motivations (i) and (ii) are tropes of the Lipschitz regularization literature; e.g., \cite{oberman2018lipschitz, finlay2018improved,couellan2021coupling, finlay2018lipschitz, pauli2021training, krishnan2020lipschitz, shi2019neural, lindemann2021learning, arghal2021robust}. Indeed, the problem formulation in \eqref{prob:manifold_lipschitz} is inspired in similar problem formulations in which the Lipschitz constant is regularized in the ambient space,
\begin{align}
	\label{prob:ambient_lipschitz}
	\underset{\theta\in\Theta,\rho\geq 0}{\minimize}  \quad   
	& L, \\
	\text{subject to}  \quad 
	& \mbE_{p(x,y)} [\ell \left(f_\theta(x),y\right)]\leq \epsilon,\nonumber\\
	&  |f_\theta(w)-f_\theta(z)| \leq L\, \| w - z \|,\nonumber \\
    &\quad(w , z)\sim p(w)\times p(z).\nonumber
\end{align}
A key difference between \eqref{prob:manifold_lipschitz} and \eqref{prob:ambient_lipschitz} is that the latter uses a Lipschitz condition that does not require differentiability. A more important difference is that in \eqref{prob:ambient_lipschitz}, the Lipschitz constant is regularized in the ambient space. The distance between features $w$ and $z$ in \eqref{prob:ambient_lipschitz} is the Euclidean distance $\| w - z \|$. This is disparate from the manifold metric $d_\ccalM(w,z)$ that is implicit in the manifold gradient constraint in \eqref{prob:manifold_lipschitz}. Thus, the formulation in \eqref{prob:manifold_lipschitz} improves upon \eqref{prob:ambient_lipschitz} because it leverages the structure of the manifold $\ccalM$ [cf. motivation~(iii)].

Motivations (i) and (iii) are themes of the manifold regularization literature \citep{belkin2005manifold, niyogi2013manifold, kejani2020graph, li2022neural}. And, indeed, it is ready to conclude by invoking Green's first identity (see Section \ref{sec_laplacian}) that the formulation in \eqref{prob:manifold_lipschitz} is also inspired in the manifold regularization problem,
\begin{align}\label{prob:manifold_regularization}
	\underset{\theta\in\Theta}{\minimize}  \quad&  
	\mbE_{p(x,y)} [\ell \left(f_\theta(x),y\right)] \\
	&+ \gamma \int_\ccalM \| \nabla_\ccalM f_\theta(z) \|^2 p(z) dV(z).\nonumber
\end{align}
The difference between \eqref{prob:manifold_lipschitz} and \eqref{prob:manifold_regularization} is that the latter adds the manifold Lipschitz constant as a regularization penalty. This is disparate from the imposition of a manifold Lipschitz constraint in \eqref{prob:manifold_lipschitz}. The regularization in \eqref{prob:manifold_regularization}  \emph{favors} solutions with small Lipschitz constant by penalizing large Lipschitz constants, while the constraint in \eqref{prob:manifold_lipschitz}  \emph{guarantees} that the Lipschitz constant is bounded by $\sqrt{\rho}$. This is the distinction between \emph{regularizing} a Lipschitz constant and \emph{constraining} a Lipschitz constant. The constraint in \eqref{prob:manifold_regularization} is also imposed at \emph{all} points in the manifold, whereas the regularization in \eqref{prob:manifold_regularization} is an average over the manifold. Taking an average allows for large Lipschitz constants at some specific points if this is canceled out by small Lipschitz constants in other points of the manifold \citep{bubeck2021a,bubeck2021law}. Both of these observations imply that \eqref{prob:manifold_lipschitz} improves upon \eqref{prob:manifold_regularization} because it offers global smoothness guarantees that are important in, e.g., physical systems [cf. motivation~(iii)].


\begin{remark}[Manifold Lipschitz formulations] \normalfont
	There are three arbitrary choices in \eqref{prob:manifold_lipschitz}: (a) We choose to constrain the average statistical loss $\mbE_{p(x,y)} [\ell \left(f_\theta(x),y\right)]\leq \epsilon$; (b) we choose to constrain the pointwise Lipschitz constant $\| \nabla_\ccalM f_\theta(z) \|^2 \leq \rho$; and (c) we choose as our objective to require a target loss $\eps$ and minimize the Lipschitz constant $L=\sqrt{\rho}$. We can alternatively choose to constrain the pointwise loss $\ell \left(f_\theta(x),y\right)\leq \epsilon$, to constrain the average Lipschitz constant $\int_\ccalM \| \nabla_\ccalM f_\theta(z) \|^2 p(z) dV(z) \leq \rho$ [cf \eqref{prob:manifold_regularization}], or to require a target smoothness $L=\sqrt{\rho}$ and minimize the loss $\eps$. All of the possible eight combinations of choices are of interest. We formulate \eqref{prob:manifold_lipschitz} because it is the most natural intersection between the regularization of Lipschitz constants in ambient spaces [cf. \eqref{prob:ambient_lipschitz}] and manifold regularization [cf. \eqref{prob:manifold_regularization}]. The techniques we develop in this paper can be adapted to any of the other seven alternative formulations.
\end{remark}


\section{Learning with Global Lipschitz Constraints} \label{sec_algorithm}

Problem~\eqref{prob:manifold_lipschitz} is a constrained learning problem that we will solve in the dual domain~\citep{chamon2020probably}. To that end, observe that \eqref{prob:manifold_lipschitz} has statistical and pointwise constraints. The loss constraint $\mbE_{p(x,y)} [\ell \left(f_\theta(x),y\right)]\leq \epsilon$ is said to be statistical because it restricts the expected loss over the data distribution. The Lipschitz constant constraints $\| \nabla_\ccalM f_\theta(z) \|^2 \leq \rho$ are said to be pointwise because they are imposed for all individual points in the manifold except perhaps for a set of zero measure. Consider then a Lagrange multiplier $\mu$ associated with the statistical constraint $\mbE_{p(x,y)} [\ell \left(f_\theta(x),y\right)]\leq \epsilon$ and a Lagrange multiplier distribution $\lambda(z)$ associated with the set of pointwise constraints $\| \nabla_\ccalM f_\theta(z) \|^2 \leq \rho$. We define the Lagrangian $L(\theta,\mu,\lambda) $ associated with \eqref{prob:manifold_lipschitz} as 
\begin{align}
L(\theta,\mu,\lambda) :=&\mu \Big(\mbE [\ell \big(f_\theta(x),y\big)]-\epsilon \Big)\nonumber \\
		&+ \int_\ccalM \lambda(z) \|\nabla_\ccalM f_\theta(z)\|^2 p(z) dV(z).
\end{align}
The dual problem associated with \eqref{prob:manifold_lipschitz} can then be written as
	\begin{align}
		D^*= \underset{\mu,\lambda \geq 0}{\text{max}}\ \min_{\theta}\ &L(\theta,\mu,\lambda) ,\label{prob:manifold_dual} \\
		\text{subject to }&\ \int_\ccalM \lambda(z) p(z) dV(z) = 1. \nonumber
	\end{align}
We point out that in \eqref{prob:manifold_dual} we remove $\rho$ from the Lagrangian by incorporating the dual variable constraint $\int_\ccalM \lambda(x)  p(x) dV(x) = 1$ (see Appendix \ref{sec:dual_problem_formulation} for details). 

We henceforth use the dual problem \eqref{prob:manifold_dual} in lieu of \eqref{prob:manifold_lipschitz}. Since we are interested in situations in which we do not have access to the data distribution $p(x,y)$, we further consider empirical versions of \eqref{prob:manifold_dual}. Given $N$ i.i.d. samples $(x_n,y_n)$ drawn from $p(x,y)$, define the empirical Lagrangian $\eL(\theta,\emu,\elambda)$ as
\begin{align}
    \eL(\theta,\emu,\elambda) :=& \emu \bigg(\dfrac{1}{N} \sum_{n = 1}^N \ell \big( f_\theta(x_n),y_n \big) -\epsilon \bigg)\nonumber\\
    &+\frac{1}{N}\sum_{n=1}^N \elambda(x_n) \|\nabla_\ccalM f_\theta(x_n)\|^2.
\end{align}
and the empirical dual problem as
	\begin{align}
		\eD^\star = \max_{\emu,\elambda \geq 0}\ \min_{\theta}\ &	\eL(\theta,\emu,\elambda) \label{prob:emp_dual}
			\\
			{} \text{subject to }& \frac{1}{N}\sum_{n=1}^N \elambda(x_n)  =1\nonumber
			\text{,}
	\end{align}
where, to simplify notation, we assume no unlabeled samples are available. If unlabeled samples are given, the modification is straightforward (see Appendix \ref{sec_empricial_unlabeled_samples}). 

The remainder of this section provides three technical contributions:


In section \ref{sec_empirical_dual} we address contribution C$2$: To justify the use of \eqref{prob:emp_dual} we must show statistical consistency with respect to the primal problem \eqref{prob:manifold_lipschitz}. This is challenging for two reasons: (i) since we do not assume the use of a linear parameterization, \eqref{prob:manifold_lipschitz} is not a linear problem in $\theta$. Thus, the primal \ref{prob:manifold_lipschitz} and dual \ref{prob:manifold_dual} are not necessarily equivalent, (ii) since we are maximizing over the dual variables $\mu$ and $\lambda(z)$, we do not know if the empirical dual formulation in \eqref{prob:emp_dual} is close to the statistical dual formulation in \eqref{prob:manifold_dual}. We overcome these two challenges and show that the empirical dual problem \eqref{prob:emp_dual} is a consistent approximation of the statistical primal problem (Proposition \ref{P:statistical_empirical_bound}). 

In section \ref{sec_laplacian} we address contribution C$3$: Solving \eqref{prob:emp_dual} requires evaluating the gradient norm sum $\sum_{n=1}^N \elambda(x_n) \|\nabla_\ccalM f_\theta(x_n)\|^2$. We will generalize results from the manifold regularization literature to show that under regularity conditions on~$\lambda$, the gradient norm integral can be computed in a more amenable form utilizing a weighted point-cloud Laplacian (Proposition \ref{P:pointcloud}). (Contribution C$3$).

In section \ref{subsec:dual_ascent} we address contribution C$4$: We introduce a primal-dual algorithm to solve \eqref{prob:emp_dual}. 


\subsection{Statistical consistency of the empirical dual problem}\label{sec_empirical_dual}

In this section, we show that \eqref{prob:emp_dual} is close to \eqref{prob:manifold_lipschitz} under the following assumptions:
\begin{assumption}\label{A:loss}
	The loss~$\ell$ is $M$-Lipschitz continuous, $B$-bounded, and convex.
\end{assumption}
%
%
\begin{assumption}\label{A:phi}
	Let $\ccalH = \{ f_\theta \mid \theta \in \Theta\subset\reals^Q \}$ with compact~$\Theta$ be the hypothesis class, and let $\bar{\ccalH} = \overline{\textup{conv}}(\ccalH)$ be the closure of its convex hull. For each $\nu>0$ and $\varphi \in \bar{\ccalH}$, there exists~$\theta \in \Theta$ such that simultaneously $\sup_{z \in \ccalM} \vert \varphi(z) - f_\theta(z) \vert \leq \nu$ and $\sup_{z \in \ccalM} \| \nabla_\ccalM \varphi(z) - \nabla_\ccalM f_\theta(z) \| \leq \nu$.
\end{assumption}
\begin{assumption}\label{A:phi_gradient}	
The functions in the hypothesis class $\ccalH$ have $\gradLip$-Lipschitz gradients, i.e. $ \|\nabla_\ccalM f_{\theta_1}(z)-\nabla_\ccalM f_{\theta_2}(z)\| \leq G|f_{\theta_1}(z) - f_{\theta_2}(z)| $, for all $\theta_1,\theta_2\in\Theta$.
\end{assumption}
\begin{assumption}\label{A:rate}
	There exists~$\zeta(N,\delta) \geq 0$, and $\hat\zeta(N,\delta)\geq 0$ monotonically decreasing with~$N$, such that
	\begin{align}\label{E:emp_bound}
		&\bigg\vert \mbE [\ell \left(f_\theta(x),y\right)] - \dfrac{1}{N} \sum_{n = 1}^N \ell \big( f_\theta(x_n),y_n \big) \bigg\vert
		\leq \zeta(N,\delta)
		\text{,}\nonumber\\
		&\bigg\vert \mbE [f_\theta(x)] - \dfrac{1}{N} \sum_{n = 1}^N  f_\theta(x_n)  \bigg\vert
		\leq \hat\zeta(N,\delta)
		\text{,}
	\end{align}
	for all~$\theta \in \Theta$, with probability~$1-\delta$ over independent draws~$(x_n,y_n) \sim p$.
\end{assumption}
\begin{assumption}\label{A:feasible}
	There exists a feasible solution~$\tilde{\theta} \in \Theta$ such that~$\mbE [ \ell(f_{\tilde{\theta}}(x),y) ] < \epsilon - M\nu$.
\end{assumption}
Assumption \ref{A:loss} holds for most losses utilized in practice. Assumption~\ref{A:phi} is a requirement on the richness of the parametrization. In the particular case of neural networks, the covering constant $\nu$ is upper bounded by the universal approximation bound of the neural network. Assumption~\ref{A:phi_gradient} also holds for neural networks with smooth nonlinearities -- e.g., hyperbolic tangents. The \emph{uniform convergence} property~\eqref{E:emp_bound} is customary in learning theory to prove PAC learnability and is implied by bounds on complexity measures such as VC dimension or Rademacher complexity~\citep{mohri2018foundations,vapnik1999nature,shalev2014understanding}. In fact, if it has bounded Rademacher complexity, both $\zeta$ and $\hat\zeta$ are bounded given that $\ell$ is Lipschitz continuous (Assumption~\ref{A:loss}). The following proposition provides the desired bound. 
\begin{proposition}\label{P:statistical_empirical_bound}
	Let~$\emu^\star,\elambda^\star$ be solutions of the empirical dual problem \eqref{prob:emp_dual}. Under assumptions~\ref{A:loss}--\ref{A:feasible}, there exists~$\etheta^\star \in \argmin_{\theta} \eL(\theta,\emu^\star,\elambda^\star)$ such that, with probability~$1-5\delta$, 
	\begin{equation}\label{E:statistical_empirical_bound}
		\begin{gathered}
			\vert P^\star -\eD^\star \vert \leq \ccalO(\nu) + (1+\Delta) \zeta(N,\delta) + \ccalO( \hat \zeta(N,\delta))\text{,}
		\end{gathered}
	\end{equation}
	where~$\Delta = \max(\emu^\star,\mu^\star) \leq C$ for a constant~$C < \infty$ and~$\mu^\star,\emu^\star$ are solutions of~\eqref{prob:manifold_dual}, \eqref{prob:emp_dual} respectively.
\end{proposition}
Proposition~\ref{P:statistical_empirical_bound} shows that the empirical dual problem \ref{prob:emp_dual} is statistically consistent. That is to say, for any realization of $N$ according to $p$, the difference between the empirical dual problem~\eqref{prob:emp_dual} and the statistical smooth learning problem~\eqref{prob:manifold_lipschitz} decreases as $N$ increases. This difference is bounded in terms of the richness of the parametrization~($\nu$), the difficulty of the fit requirement~(as expressed by the optimal dual variables~$\mu^\star,\emu^\star$), and the number of samples~($N$). The guarantee has a form typical for constrained learning problems~\citep{chamon2022constrained}. Proposition \ref{P:statistical_empirical_bound} states that we are able to predict what is the minimum norm of the gradient that a function class can have while achieving an expected loss of at most~$\epsilon$. This is important because we do not require access to the distribution $p$, only a set of $N$ samples from this distribution. On the other hand, Proposition \ref{P:statistical_empirical_bound} does not state that by solving the dual problem \eqref{prob:emp_dual} we will obtain a solution of the primal problem \eqref{prob:manifold_lipschitz}. The following proposition provides a bound on the near feasibility of the solution of \eqref{prob:emp_dual} with respect to the solution of \eqref{prob:manifold_lipschitz}.
\begin{proposition}\label{P:solution_bound}
	Let~$\emu^\star,\elambda^\star$ be solutions of the empirical dual problem~\eqref{prob:emp_dual}. Under assumptions~\ref{A:loss}--\ref{A:feasible}, there exists~$\etheta^\star \in \argmin_{\theta} \eL(\theta,\emu^\star,\elambda^\star)$ such that, with probability~$1-5\delta$, 
	\begin{align}\label{E:solution_bound}
		\max_{z\in\ccalM} \|\nabla_\ccalM f_{\etheta^\star} (z)\|^2 \leq & P^* +\ccalO(\nu)+\ccalO(\hat\zeta(N,\delta)) \\
   &+ \emu^*\bigg\vert\frac{1}{N}\sum_{n=1}^N \ell(f_{\etheta^*} (x_n),y_n) -\epsilon\bigg\vert \nonumber\\
		\text{and }\mbE [ \ell(f_{\etheta^\star}(x),y) ] \leq & \epsilon + \zeta(N,\delta) .
	\end{align}
\end{proposition}
Proposition \ref{P:solution_bound} provides near optimality and near feasibility conditions for solutions $\hat \theta^*$ obtained through the empirical dual problem \eqref{prob:emp_dual}. The difference  between the maximum gradient of the obtained solution $\theta$ and the optimal value $P^*$ is bounded by the number of samples $N$ as well as  the empirical constraint violation. Notice that even though the optimal dual variable $\hat \mu$ is not known, the constraint violation can be evaluated in practice as it only requires to evaluate the obtained function over the $N$ given samples. 
\begin{remark}[Interpolators]
	In practice, the number of parameters in a parametric function (e.g., Neural Network) tends to exceed the dimension of the input, which allows functions to interpolate the data, i.e., to attain zero loss on the dataset. Proposition \ref{P:solution_bound} presents a connection to interpolating functions. By setting $\epsilon=0$, if the function achieves zero error over the empirical distribution i.e. $\ell(f_{\hat\theta^*}(x_n),y_n)=0,\ \forall \ n \in [N]$, then the dependency on $\mu^*$ disappears. This implies that within the classifiers that interpolate the data, the one with the minimum Lipschitz constant over the available samples, will probably be the one with the minimum Lipschitz constant over the whole manifold.
\end{remark}

\subsection{From Manifold Gradient to Discrete Laplacian}\label{sec_laplacian}

We derive an alternative way of computing the integral of the norm of the gradient utilizing samples. To do so, we define the normalized point-cloud Laplacian according a probability distribution $\lambda$.

\begin{definition}[Point-cloud Laplacian]\label{def:point_cloud_laplacian}
	Consider a set of points~$x_1,\dots,x_N \in \ccalM$, sampled according to probability $\lambda:\ccalM\to\reals$. The normalized graph Laplacian of~$f_\theta$ at~$z \in \ccalM$ is defined as
	\begin{equation}\label{E:graph_laplacian}
		\bbL_{\lambda,N}^t f_\theta(z) = \frac{1}{N} \sum_{n=1}^N W(z,x_n) \big( f_\theta(z) - f_\theta(x_n) \big)
		\text{,}
	\end{equation}
	\begin{align*}
		\text{for }W(z,x_n) &= \frac{1}{t} \frac{G_t(z,x_n)}{ \sqrt{\hat{w}(z) \hat{W}(x_n)} }
		\text{,} \quad \\
  &\text{with} \quad
		G_t(z,x_n) = \frac{1}{(4 \pi t)^{d/2}} \, e^{-\frac{\| z - x_n \|^2}{4t}}
		\text{,}\nonumber
		\\
		\hat{w}(z) &= \frac{1}{N} \sum_{n = 1}^N G_t(z,x_n)
		\text{,} \quad\\
  &\text{and} \quad
		\hat{W}(x_n) = \frac{1}{N-1}\sum_{m \neq n} G_t(x_m,x_n)
		\text{.} \nonumber
	\end{align*}
\end{definition}

As long as the function considered is smooth enough, the following convergence result holds:


\begin{proposition}[point-cloud Estimate]\label{P:pointcloud}
	Let $\Lambda$ be the set of probability distributions defined a compact $d$-dimensional differentiable manifold $\ccalM$ isometrically embedded in $\reals^D$ such that $\Lambda=\{\lambda: 0<a\leq\lambda(z)\leq b<\infty, | \frac{\partial \lambda}{\partial x}|\leq c < \infty \text{ and } |\frac{\partial^2 \lambda}{\partial x^2}|\leq d <\infty \ \forall\ z \in \ccalM \}$, and let $f_\theta$, with $\theta \in \Theta$, be a family of functions with uniformly bounded derivatives up to order $3$ vanishing at the boundary $\partial\ccalM$.  For any $\epsilon>0,\delta>0$, for all $N>N_0$,
	\begin{align}\label{P:pointcloud_equation}
 \begin{split}
		&P\bigg[\sup_{\lambda\in\Lambda,\theta \in \Theta}\bigg|\int \|\nabla_\ccalM f_\theta(z)\|^2 \lambda(z)dV(z) \\
  &\quad-\frac{1}{N}\sum_{i=1}^N f_\theta(z_i)\bbL^t_{\lambda, N} f_\theta(z_i)\lambda(z_i)\bigg|>\epsilon \bigg]\leq \delta 
  \end{split}
	\end{align}
	where the point-cloud laplacian $\bbL^t_{\lambda, N} f_\theta $ is as defined in \ref{def:point_cloud_laplacian}, with $t=N^{-\frac{1}{d+2+\alpha}}$ for any $\alpha>0$. 
\end{proposition}
The proof of Proposition \ref{P:pointcloud} relies on two steps. First, we relate the integral of the norm of the gradient over the manifold with the integral of the continuous Laplace-Beltrami operator by virtue of Green's identity. Second, we approximate the value of the Laplace-Beltrami operator by the point-cloud Laplacian. Proposition \ref{P:pointcloud} connects the integral of the norm of the gradient of function $f_{\theta}$ with a point-cloud Laplacian operator. This result connects the dual problem \eqref{prob:emp_dual}, with the primal \eqref{prob:manifold_lipschitz} while allowing for a more amenable way of computing the integral. 
\begin{remark}
[Laplacian Regularization]
   The dual problem~\eqref{prob:manifold_dual} is closely related to the manifold regularization problem~\eqref{prob:manifold_regularization}. In particular, the two become equivalent by substituting $\rho=\mu^{-1}$ and utilizing a uniform distribution for $\lambda$. The key difference between the problems is given by the dual variable $\lambda$, which can be though as a probability distribution over the manifold that penalizes regions of the manifold where the norm of the gradient of $f_\theta$ is larger. The standard procedure in Laplacian regularization is to calculate the graph-Laplacian of the set of points $\bbL$ utilizing the heat kernel and compute the integral by $\bbf^T \bbL \bbf$, where $\bbf=[f_\theta(x_1),\dots,f_\theta(x_n)]^T$. In the case of Manifold Lipschitz, the same product can be computed, but utilizing the re-weighted point-cloud laplacian.
\end{remark}
\subsection{Dual Ascent Algorithm}\label{subsec:dual_ascent}
%


We outline an iterative and empirical primal-dual algorithm to solve the dual problem. Upon the initialization of $\theta_0,\lambda_0$~and~$\mu_0$, we set out to minimize the dual function using, 
\begin{align}\label{eqn:primal_update}
\theta_{k+1} = \theta_{k} - \eta_\theta\nabla_\theta L(\theta_k,\mu_k,\lambda_k),
\end{align}
where $\eta_\theta$ is a positive stepsize. Note that to compute~\eqref{eqn:primal_update}, we can either utilize the gradient version of the Lagrangian~\eqref{prob:emp_dual} or the point-cloud Laplacian~\eqref{P:pointcloud_equation}. Consequent to updating $\theta$, we update the dual variables by 
\begin{align}\label{eqn:dual_update}
    \mu_{k+1}&=\bigg[\mu_{k} + \eta_\mu \bigg(\frac{1}{N}\sum_{n=1}^N \ell(f_\theta(x_n),y_n)-\epsilon \bigg)\bigg]_+,\\
    \tilde\lambda_{k+1}(x_n)&=\lambda_{k}(x_n) + \eta_\lambda \| \nabla_\ccalM f_\theta(x_n)\|^2,n=1,\dots,N\label{eqn:dual_step_lambda_tilde}\\
    \lambda_{k+1}&=\argmin_{\lambda} \|\tilde \lambda_{k+1}-\lambda_{k+1}\|,\\
    &\quad \text{ subject to }\sum_{n=1}^N\lambda_{k+1}(x_n)=N, \nonumber
\end{align}
where $\eta_\mu,\eta_\lambda$ are again a positive stepsize. Note that we require a convex projection over $\tilde \lambda$ to satisfy the normalizing constraint. In step \eqref{eqn:dual_step_lambda_tilde}, we need to estimate the norm of the gradient at data point $x_n$, which we do using neighboring data points. Intuitively, the primal-dual procedure increases the value of $\lambda(x_n)$ at points in which the norm of the gradient is larger. The role of $\mu$ is to enforce the loss $\ell$ to be smaller than $\epsilon$, by adjusting the relative importance of the loss $\ell$ over the norm of the integral (see Appendix \ref{appendix:algorithms}).

\section{Experiments} \label{sec_numerical_results}
To demonstrate the effectiveness of our method, we conduct two real world experiments with physical systems. In Section \ref{ssec:ground_robot}, we tackle a regression problem in which we try to predict the error made when estimating the state of a ground robot making turns in both pavement and grass. In Section \ref{ssec:quadrotor} we learn the dynamics of a quadrotor when taking off and making circles in open air. In both experiments, data is acquired from real world robots and we compare our method against standard Empirical Risk Minimization (ERM), Ambient Regularization, and Manifold Regularization.



%
\subsection{Ground Robot Error Prediction}\label{ssec:ground_robot}
In this experiment, we seek to learn the error that a model would make in predicting the dynamics of a ground robot by looking at the states throughout a trajectory \cite{7759118}.

 The data acquisition of this experiment involves an \textit{iRobot} Packbot equipped with high resolution camera. The setting of the data acquisition is a robot making turns on both pavement and grass. The dynamics of the system are govern by the discrete-time nonlinear state-space system of equations
\begin{align}
x_{k+1} =     f(x_k,u_k)+g(u_k),
\end{align}
where $x_k$ is the state of the system,  $u_k$ is the control input, $f(x_k,u_k)$ is the model prediction, and $g(u_k)$ is the non-modelable error of the prediction. In this setting, the robot state involves the position and the control inputs are linear and angular velocities. The dynamics of the system are modeled by $f(x_k,u_k)$ and they involve the mass of the robot, the radius of the wheel, and other known parameters of the robot. 
In practice, the model $f(x_k,u_k)$ is not perfect and the model mismatch is given by the difference in friction with the ground, delays in communications with the sensors, and discrepancies in the robot specifications. To make matters worse, some of the discrepancies are difficult to model or intractable to compute in practice. 
 
We possess trajectories in the form of time series of the control signals, i.e., linear and angular velocity of the robot. For each trajectory, the average and variance of the mismatch between the real and the model-predicted states is quantified. Given a dataset of trajectories and errors made by the model, the objective is to predict the error that the model will make on a given trajectory. 

In order to construct the point-cloud Laplacian, we measured the euclidean distance between samples and computed \eqref{E:graph_laplacian}. In this case we added a threshold, i.e. minimum magnitude of $\bbL_{ij}$ below which the value of $\bbL_{ij}$ becomes $0$. This is to disregard the effect of samples that are far away, by only considering neighborhoods of each sample. By doing this we are constructing a manifold in which samples that are sufficiently close will be forced to have similar outputs. However, if samples are not sufficiently close, the similarity between the outputs of the function will not be forced to be small (to exemplify this point we have a toy example in Appendix \ref{Appendix:NavigationControls}).
The manifold in this case is constructed by time series that are sufficiently similar to each other. 
For further details of the experiment and sample trajectories can be found in Appendix \ref{Appendix:GroundRobot}. 
We show the results of the ground robot prediction experiment in Table \ref{table:ground_robot}.

\begin{table}[tbh]
     \centering
	\begin{tabular}{||c c c||} 
		\hline
		Method & Grass & Pavement  \\ [0.5ex] 
		\hline\hline
		ERM & $0.42$  & $0.0120$ \\ 
		\hline
		Ambient Regularization &  $0.31$  & $0.0065$ \\
		\hline
		Manifold Regularization & $0.38$  & $0.0045$ \\
		\hline
		Manifold Lipschitz (ours) & $0.25$  & $0.0032$ \\ 
		\hline
	\end{tabular}
	\caption{Error prediction accuracy for the Ground Robot Experiment.\label{table:ground_robot}}
 \end{table}

As seen on Table \ref{table:ground_robot}, regularization improves the accuracy over the standard ERM framework. 
This is related to the fact that given that the underlying predictive model is the same in all trajectories, similar trajectories will have similar errors.
However, our method improves upon both regularization techniques. 
This is related to the idea that between experiments only the velocities change (the environment is fixed), so we can intuitively imagine a continuous transition in the error made by the model. This explains why our method (Manifold Lipschitz) by forcing the function to be smooth over the trajectory space is significantly better in both experiments. Our method also improves upon Ambient regularization because the euclidean distance is able to approximate the distance on the manifold \textit{locally}, but it is not able to approximate it \textit{globally}. For a better clarification on this last point see Appendix \ref{Appendix:NavigationControls}.
In all, in this experiment 
we show that our method improves the generalization capabilities of a function.


\subsection{Quadrotor Model Mismatch}\label{ssec:quadrotor}
In this section introduce
a state prediction problem based real world dynamics. The setting consists of a quadrotor taking off and flying in circles for $12$ seconds.
\begin{figure}[H]
    \centering
    \includegraphics[scale=0.5]{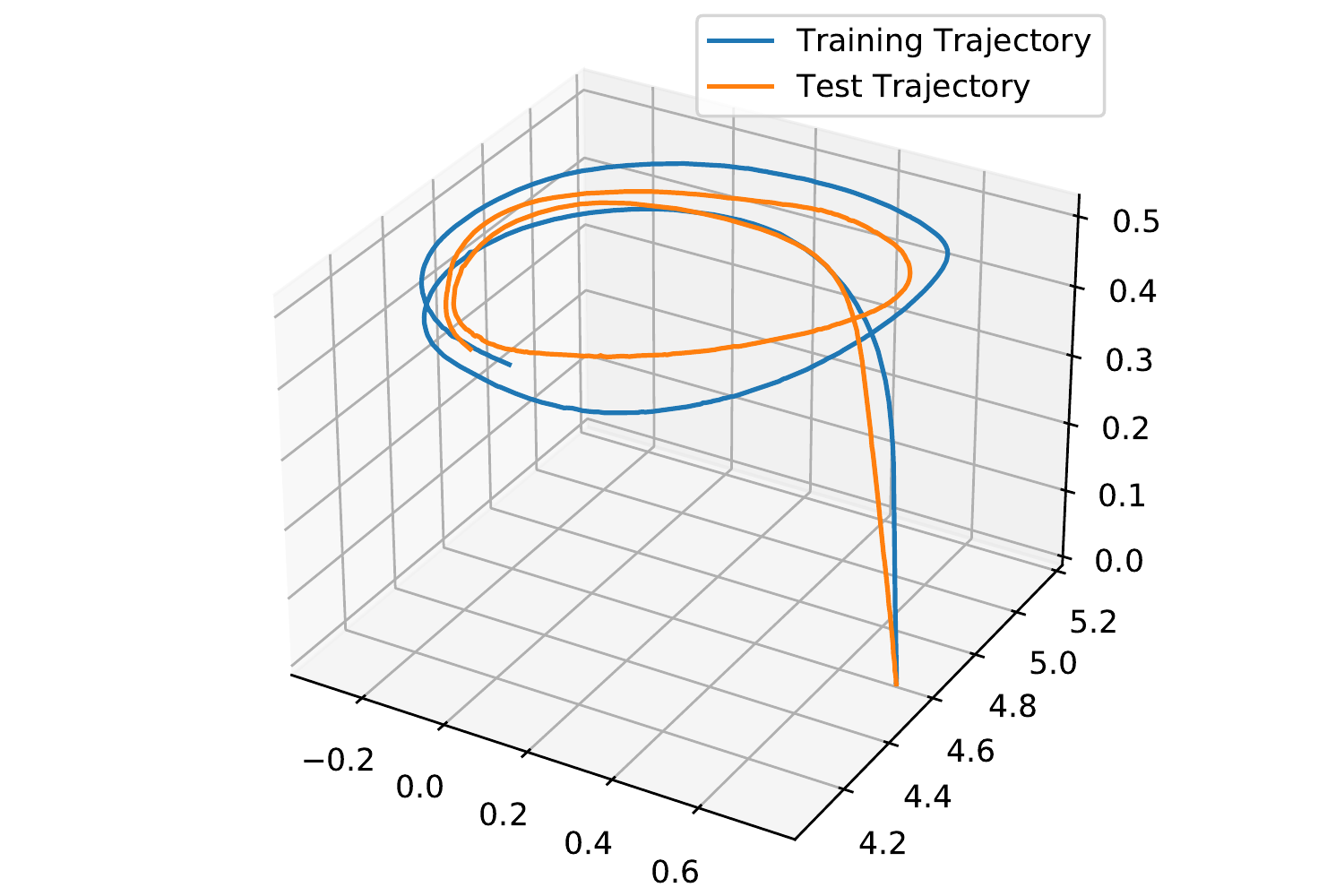}
    \caption{Quadrotor Sample Trajectories}
    \label{fig:quadrotor}
\end{figure}
The experimental setup is to target a speed of $0.4$m/s and to track a circular trajectory of radius $0.5$m. We consider $2$ trajectories of $12$ seconds each, and the starting position of the quadrotor is the same for both trajectories (the ground). For each time stamp $t$, we have measurements of position, velocity and acceleration in $\reals^3$, i.e. $[x_t,y_t,z_t,\dot x_t, \dot y_t, \dot z_t, \ddot x_t, \ddot y_t, \ddot z_t]$.

For the learning procedure we utilize the $6000$ samples, and we seek to minimize the mean square error loss between the next state and the prediction given the current state. We train a two layer neural network with different methods as seen in Table \ref{table:quadrotor}. To test the accuracy of the learned function, we compute the difference between the predicted state and the next state on the test trajectory. 

Analogous to experiment \ref{ssec:ground_robot}, to construct the point-cloud Laplacian, we compute the euclidean distances in a neighborhood of each sample. Intuitively, by looking at the trajectories \ref{fig:quadrotor}, we can expect next states to be similiar only on samples that are sufficiently close. The euclidean distance is able to approximate the manifold \textit{locally} but fails when it is too large. By looking at Figure \ref{fig:quadrotor} we can see that even though the training and testing trajectories are not the same, there is a close resemblance between the two of them. This allows a manifold model to be effective in predicting the next state. 
\begin{table}
\centering
	\begin{tabular}{||c c||} 
		\hline
		Method &  Error on Test Trajectory  \\ [0.5ex] 
		\hline\hline
        ERM & $0.00666$ \\ 
		\hline
		Ambient Regularization & $0.00734$ \\
		\hline
		Manifold Regularization & $0.00625$ \\
		\hline
		Manifold Lipschitz (ours) & $0.00237$ \\ 
		\hline
	\end{tabular}
	\caption{State prediction error of a quadrotor flying in circles on an unseen trajectory.\label{table:quadrotor}}
\end{table}
However, by looking at the results in Table \ref{table:quadrotor}, we see that adding regularization does not always help, as ambient regularization does not improve upon ERM. This is related to the fact that the simple euclidean distance is not necessarily a good measure of the how related two states are. 
Instead, one conclusion of the results shown in Table~\ref{table:quadrotor} is that adding regularization on the manifold space always improves upon ERM. In particular, our method give performance almost three times better than standard ERM and more than twice that of standard Laplacian Regularization. This is a consequence of adding extra information about the problem by grouping similar states together, which reduces the impact of the noise in a particular sample. 
To conclude, we show that our method obtains an improvement over all the techniques considered and that by predicting the next state of a quadrotor from the current state utilizing smooth functions improves generalization for noisy measurements.


%

\section{Conclusion}
In this work, we presented a constraint learning method to obtain smooth functions over manifold data. We showed that under mild conditions, the problem of finding smooth functions over a manifold can be reformulated as a weighted point-cloud Laplacian penalty over varying probability distributions whose dynamics are govern by the constraint violations. Two experiments on real world data validate the empirical advantages of obtaining functions that vary smoothly over the data.



\newpage
\bibliography{bib}
\bibliographystyle{icml2023}


\newpage

\appendix
\onecolumn

\section{Dual Problem Formulation}\label{sec:dual_problem_formulation}

In Section \ref{sec_algorithm} we introduce the optimization program in \eqref{prob:manifold_dual} as the dual of \eqref{prob:manifold_lipschitz}. Strictly speaking \eqref{prob:manifold_dual} is equivalent to the actual dual problem of \eqref{prob:manifold_lipschitz}. This follows from a ready reformulation of the dual problem as we show in the following proposition.

\begin{proposition}
The optimization program in \eqref{prob:manifold_dual} is equivalent to the Lagrangian dual of \eqref{prob:manifold_lipschitz}.    
\end{proposition}

\begin{proof}
The result is true because the dual problem is linear in $\rho$. To see this, recall that $\mu$ is the dual variable associated with the statistical constraint $\mbE_{p(x,y)} [\ell \left(f_\theta(x),y\right)]\leq \epsilon$ and that $\lambda(z)$ is the Lagrange multiplier distribution associated with the set of pointwise constraints $\| \nabla_\ccalM f_\theta(z) \|^2 \leq \rho$. The Lagrangian $\tilde\ccalL(\rho,\theta,\mu,\lambda)$ of \eqref{prob:manifold_lipschitz} is therefore written as 
\begin{align}\label{eqn_dual_problem_derivation_pf_10}
\tilde\ccalL(\theta,\rho,\mu,\lambda) =& \rho + \mu \bigg(\mbE [\ell \left(f_\theta(x),y\right)]-\epsilon \bigg) + \int_\ccalM \lambda (z)\bigg(\|\nabla f_\theta(z)\|^2 -  \rho \bigg) p(z)dV(z)
\end{align}
Reorder terms in \eqref{eqn_dual_problem_derivation_pf_10} to group the two summands that involve $\rho$ to write
\begin{align}\label{eqn_dual_problem_derivation_pf_20}
\tilde\ccalL(\theta,\rho,\mu,\lambda) =& \rho \bigg(1 - \int \lambda(z)p(z)dV(z) dz  \bigg)+ \mu \bigg(\mbE_{p(x,y)} [\ell \left(f_\theta(x),y\right)]-\epsilon \bigg) \nonumber \\
&+ \int  \lambda (z)\|\nabla_\ccalM f_\theta(z)\|^2 p(z)dV(z) .
\end{align}
An important observation to make is that the Lagrangian decomposes in a term that involves only $\rho$ and a term that involves only $f_\theta$. Define then
\begin{alignat}{2}\label{eqn_dual_problem_derivation_pf_25}
   &\tilde\ccalL_1(\rho, \lambda) 
        &&~:= \rho \bigg(1 - \int \lambda(z)p(z)dV(z) dz  \bigg),  \nonumber \\
   &\tilde\ccalL_2(\theta, \mu,\lambda) 
        &&~:= \mu \bigg(\mbE_{p(x,y)} [\ell \left(f_\theta(x),y\right)]-\epsilon \bigg) 
               + \int  \lambda (z)\|\nabla_\ccalM f_\theta(z)\|^2 p(z)dV(z), 
\end{alignat}
so that we can write the Lagrangian in \eqref{eqn_dual_problem_derivation_pf_20} as $\tilde\ccalL(\theta,\rho,\mu,\lambda) = \tilde\ccalL_1(\rho, \lambda) + \tilde\ccalL_2(\theta, \mu,\lambda)$. 

The dual problem can now be written as the maximization over multipliers of the minimum of the Lagrangian over primal variables
\begin{align}\label{eqn_dual_problem_derivation_pf_30}
   \tilde D^* 
       & = \max_{\mu,\lambda} 
              \min_{\theta,\rho} ~ \tilde\ccalL(\rho,\theta,\mu,\lambda)  \nonumber \\
       & = \max_{\mu,\lambda} 
              \bigg[
                  \min_{\rho} ~ \tilde\ccalL_1(\rho,\mu,\lambda) +
                  \min_{\theta} ~ \tilde\ccalL_2(\theta,\mu,\lambda) 
              \bigg] .
\end{align}
where we utilized the decomposition of the Lagrangian to write the second equality. 

The important observation to make is that the minimization over $\rho$ of $\tilde\ccalL_1(\rho,\mu,\lambda)$ has an elementary solution. Indeed, as per its definition we have
\begin{align}\label{eqn_dual_problem_derivation_pf_40}
   \min_{\rho} ~ \tilde\ccalL_1(\rho,\mu,\lambda) 
   = \min_{\rho} ~  \rho \bigg(1 - \int \lambda(z)p(z)dV(z) dz  \bigg).
\end{align}
This minimization yields $-\infty$ when $ \int \lambda(z)p(z)dV(z) dz \neq 1$ and $0$ when $ \int \lambda(z)p(z)dV(z) dz = 1$. Since in the dual problem we are interested in the maximum over all dual variables, we know that: (i) The maximum will be attained for a dual distribution that satisfies $\int \lambda(z)p(z)dV(z) dz = 1$. (ii) When dual variables satisfy this property we know that $\min_{\rho} ~  \rho (1 - \int \lambda(z)p(z)dV(z) dz )=0$. It then follows that the dual problem in \eqref{eqn_dual_problem_derivation_pf_30} is equivalent to
\begin{alignat}{2} \label{eqn_dual_problem_derivation_pf_50}
   \tilde D^* 
         = & \max_{\mu,\lambda} 
                  && \min_{\theta} ~ \tilde\ccalL_2(\theta,\mu,\lambda), \nonumber \\
  	       & \text{subject to } &&\ \int_\ccalM \lambda(z) p(z) dV(z) = 1.                 
\end{alignat}
This is the problem in \eqref{prob:manifold_dual} given the definition of $\tilde\ccalL_2(\theta,\mu,\lambda)$ in \eqref{eqn_dual_problem_derivation_pf_25} which is the same as the definition of $L(\theta,\mu,\lambda)$ in  \eqref{prob:manifold_dual}. $\QED$

\end{proof}

Notice that the constraint $\int_\ccalM \lambda(z) p(z) dV(z) = 1$ implies that $\lambda(z)$ is a probability distribution over the manifold $\ccalM$. This is an important observation for the connections we establish to manifold regularization in Section \ref{sec_laplacian}. It also implies that even though the dual problem \ref{prob:manifold_dual} is not an unconstrained problem, the constraint is easy to enforce as an orthogonal projection on the space of probability distributions over the manifold $\ccalM$. This is a simple normalization. 


\section{Empirical Dual Problem with unlabeled samples} \label{sec_empricial_unlabeled_samples}

In Section \ref{sec_algorithm} we introduce the empirical dual program in \eqref{prob:emp_dual} in the case in which unlabeled samples are unavailable. The ability to leverage unlabeled samples is an important feature of this work and ready to incorporate in \eqref{prob:emp_dual}. Indeed, if in addition to $N$ i.i.d. labeled samples $(x_n,y_n)$ with $N\in[1,N]$ drawn from $p(x,y)$ we are also given $\tilde N$ i.i.d. unlabeled samples $x_n$ with $n\in[N+1, N+\tdN]$ drawn from the input distribution $p(x)$ we redefine \eqref{prob:emp_dual} as
\begin{align}\label{eqn_emp_dual_unlabeled_samples}
		\eD^\star \!= \max_{\emu,\elambda \geq 0} \min_{\theta}  &	\eL(\theta,\emu,\elambda) := \emu \bigg(\dfrac{1}{N} \sum_{n = 1}^N \ell \big( f_\theta(x_n),y_n \big) -\epsilon \bigg)\!
		    +\frac{1}{N\!+\!\tdN}\sum_{n=1}^{N+\tdN} \elambda(x_n) \|\nabla_\ccalM f_\theta(x_n)\|^2,\nonumber
			\\
			{} \text{subject to }& \frac{1}{N+\tdN}\sum_{n=1}^{N+\tdN} \elambda(x_n)  =1
			\text{,}
	\end{align}
Results in Section \ref{sec_algorithm} hold with proper modifications.


\section{Proof of Proposition \ref{P:statistical_empirical_bound}}

The proof in this appendix is a generalization of the proof in \cite{chamon2022constrained}. In order to show Proposition \ref{P:statistical_empirical_bound}, we need to introduce an auxiliary problem formulation over a functional domain of functions. In this case, we take the optimization problem \eqref{prob:manifold_lipschitz} over the convex hull of the domain of parametric functions $\phi\in\bar \ccalH$,
\begin{align}\label{prob:manifold_lipschitz_functional}
	\fP^* = \ \ \ \ \ \ \underset{\phi\in \bar\ccalH,\rho\geq 0}{\min}  \quad   
	& \rho, \\
	\text{subject to}  \quad 
	& \mbE_{p(x,y)} [\ell \left(\phi(x),y\right)]\leq \epsilon, \nonumber\\
	& \| \nabla_\ccalM \phi(z) \|^2 \leq \rho ,\quad p(z)\text{-a.e.}, \quad z \in \ccalM . \nonumber
\end{align}
We can now show that problem \ref{prob:manifold_lipschitz_functional} is strongly dual as follows, 
%
\begin{lemma}\label{lemma:zero_duality_gap_functional}
	Under the assumptions of Proposition \ref{P:statistical_empirical_bound}, the functional smooth learning problem \eqref{prob:manifold_lipschitz_functional} has zero duality gap, i.e. $\tilde P^*=\tilde D^*$.
\end{lemma}
\myproof[Lemma \ref{lemma:zero_duality_gap_functional}]{
	To begin with, the non-parametric space $\bar\ccalH$ is convex as it is the convex hull of the space of parametric functions $\ccalH$, therefore the domain of the optimization problem is a convex set. Note that the objective is linear. Given that the gradient of $\phi$ is a linear function of $\phi$, and that by taking the norm we preserve convexity, the constraint $\|\nabla_\ccalM\phi(z)\|^2\leq \rho$ is convex. By assumption \ref{A:loss}, the loss $\ell$ is convex. Therefore, problem \eqref{prob:manifold_lipschitz} is a semi-infinite convex problem. Moreover, as $\tilde \theta$ belongs to the relative interior of the feasible set, and $\ccalH\subset\hat \ccalH$, it suffices to take $\phi^\dagger(\cdot)=f(\tilde \theta,\cdot)$, and $\rho^\dagger > \gradBound^2 $, and by Slater's condition as $\phi^\dagger,\rho^\dagger$ belongs to the interior of the feasible domain, problem \eqref{prob:manifold_lipschitz} has zero duality gap, i.e. it is strongly dual. 
}

Now we need to define the \textit{supergradient} of function $d(\mu,\lambda)$.
\begin{definition}[Supergradient and Superdifferential]\label{def:supergradient_superdifferential}
	We say that $c\in\reals$ is a supergradient of $d$ at $\mu$ if,
	\begin{align}
		d(\mu^{'})\leq d(\mu) + c(\mu^{'}-\mu) \forall \mu^{'} \in \reals. 
	\end{align}
	The  set of all supergradients of $d$ at $\mu$ is called the superdifferential, and we denote it $\partial d(\mu)$.
\end{definition}
\begin{lemma}[Danskin's Theorem]\label{lemma:danskins}
	Consider the function,
	\begin{equation}
		F(x) = \sup_{y\in Y}f(x,y).
	\end{equation}
	where $f:\reals^n\times Y\to \reals \cup \{-\infty,+\infty\}$, if the following conditions are satisfied
	\begin{enumerate}
		\item Function $f(\cdot,y)$ is convex for all $y\in Y$.
		\item Function $f(x,\cdot)$ is upper semicontinuous for all $x$ in a certain neighborhood of a point $x_0$.
		\item The set $Y\subset \reals ^m$ is compact
	\end{enumerate}
	Then, 
	\begin{align}
		\partial F(x_0)= \text{conv}\left( \underset{y\in \hat Y(x_0)}{\cup} \partial_x f(x_0,y) \right) 
	\end{align}
	where $\partial_x f(x_0,y)$ denotes the subdifferential of the function $f(\cdot,y)$ at $x_0$
\end{lemma}
\myproof[of Lemma \ref{lemma:danskins}]{
	The proof can be found in \cite{ruszczynski2011nonlinear}[Theorem 2.87].
}

We now need to define the primal problem \ref{prob:manifold_lipschitz} associated with Lagrangian $ L(\theta, \mu, \lambda)$ as follows,
\begin{equation}\label{prob:manifold_lipschitz_laplacian_primal}
	P^\star = \min_{\theta\in\Theta}\ \max_{\mu,\lambda \geq 0}\  L(\theta,\mu,\lambda)
	\text{.}
\end{equation}
note that \ref{prob:manifold_lipschitz} and \ref{prob:manifold_lipschitz_laplacian_primal} are equivalent problems, given that in the optimal solution, $\int_\ccalM \lambda(z) p(z)dV(z)$, to avoid an infinite result.

\begin{lemma}\label{lemma:duality_gap_grad_lip}
	Let $ \mu^*,\lambda^*$ be a solution of \eqref{prob:manifold_lipschitz_laplacian_primal}. Under the conditions of Proposition \ref{P:statistical_empirical_bound}, there exists a feasible $\theta^\dagger \in \argmin_{\theta} \ccalL(\theta,\mu^*,\lambda^*)$, and the value $D^*$ is bounded by,
	\begin{equation}\label{eqn:lemma:duality_gap}
		P^*- \nu(\mu^*_\nu M   -2 P_{\phi_\nu^*}^* -\nu)\leq D^*\leq P^*
	\end{equation}
	where $P_{\phi_\nu^*}^*$ and $\mu^*_\nu$ are the optimal value, and optimal dual variable of the functional version of problem \ref{prob:manifold_lipschitz} with constraint $\epsilon-\nu M$.
\end{lemma}
\myproof{
	First we need to show that there is a feasible $\theta^\dagger$ for problem \eqref{prob:manifold_lipschitz}. The constraint $\|\nabla_{\ccalM}f_{\check{\theta}^\dagger}(x) \|^2\leq\rho$ can be trivially satisfied by taking $\rho=\gradBound^2$, as $\|\nabla_\ccalM f_\theta (x) \|\leq \gradBound$ by Assumption \ref{A:phi}. However, we need to verify that there exists a $\theta^\dagger$ that is feasible. To do so,  We begin by considering the set of Lagrange minimizers as follows,
	\begin{equation}
		\Theta^\dagger( \mu^*,\lambda^*)= \argmin_{\theta}  L(\theta, \mu^*,\lambda^*)
	\end{equation}
	we can also define the constraint slack associated with parameters $\theta$ as follows, 
	\begin{equation}
		c(\theta) = [\mbE[\ell(f_\theta(x),y)]-\epsilon]_{+}.
	\end{equation}
	Therefore, to show that there exists a feasible solution, it is analogous to show that there is an element $\theta^\dagger$ of $\Theta^\dagger(\check \mu^*,\lambda^*)$ whose slack is equal to zero, i.e. $c(\theta^\dagger)=0$. To do so, we leverage Lemma \ref{lemma:danskins}, given that condition $(1)$ is satisfied by convexity of $\ell$, and linearity of integral, condition $(2)$ is satisfied by linearity of integral, condition $(3)$ is satisfied by the compactness of $\Theta$, and condition $(4)$ is satisfied by the smoothness of $\Theta$.  
	
	By contradiction, we can say that if no element of $\Theta(\mu^*,\lambda^*)$ is feasible, then $\mbE[\ell(f_\theta(x),y)]-\epsilon >0, \forall \ \theta \in \Theta( \mu^*,\lambda^*)$. From Lemma \ref{lemma:danskins}, we get that $\bbzero \notin \partial \Theta^\dagger( \mu^*,\lambda^*)$, which contradicts the optimality of $ \mu^*,\lambda^*$. Given that the dual variable $\mu^*$ is finite, considering the existence of a feasible solution by Assumption \ref{A:phi}. Hence, there must be one element of $\Theta( \mu^*,\lambda^*)$ that is feasible.  
	
	Now we need to show that the inequality \ref{eqn:lemma:duality_gap} holds. The upper bound is trivially verified by weak duality \cite{boyd2004convex}, i.e.,
	\begin{align}
		\check{D}^*\leq \check{P}^*
	\end{align}
	To show the lower bound, we consider the functional version of problem \ref{prob:manifold_lipschitz} with constraint $\epsilon-M\nu$ as follows, 
	\begin{align}\label{prob:manifold_lipschitz_functional_Mnu}
		\tilde P_{\nu}^* = \ \ \ \ \ \ \underset{\phi\in \bar\ccalH,\rho\geq 0}{\min}  \quad   
		& \rho, \\
		\text{subject to}  \quad 
		& \mbE_{p(x,y)} [\ell \left(\phi(x),y\right)]\leq \epsilon-M\nu, \nonumber\\
		& \| \nabla_\ccalM \phi(z) \|^2 \leq \rho ,\quad p(z)\text{-a.e.}, \quad z \in \ccalM . \nonumber
	\end{align}
	Now note that problem \ref{prob:manifold_lipschitz_functional_Mnu} is strongly dual by Lemma \ref{lemma:zero_duality_gap_functional}, i.e.,
	\begin{align}
		\tilde P_{\nu}^*= 	 \min_{\phi}\ \underset{\mu,\lambda \geq 0}{\text{maximize}} \ \tilde{L}_{\nu}(\phi,\rho,\mu,\lambda) = 	\underset{\mu,\lambda \geq 0}{\text{maximize}}\ \min_{\phi}\ \tilde{L}_{\nu}(\phi,\mu,\lambda) = 	\tilde D_{\nu}^*
	\end{align}
	with the Lagrangian defined as, 
	\begin{equation}\label{eqn:lagrangian_function_Mnu}
		\begin{aligned}
			\tilde{L}_{\nu}(\phi,\mu,\lambda) &= \mu \Big(\mbE [\ell \big(\phi(x),y\big)]-(\epsilon -M \nu)\Big)+ \int_\ccalM \lambda(x) \|\nabla_\ccalM \phi(x)\|^2 p(x) dV(x))	\\
			{}&\text{subject to }\int_\ccalM \lambda(x) p(x) dV(x) =1
			\text{,}
		\end{aligned}
	\end{equation}
	we define the optimal dual variables $\mu^*_{\nu}, \lambda^*_{\nu}$ of $\tilde D _{\nu}^*$ be such that,
	\begin{align}
		\tilde P_{\nu}^*= 	 \min_{\phi}\  \tilde{L}_{\nu}(\phi,\mu^*_{\nu},\lambda^*_{\nu}) 
	\end{align}
	Coming back to the parametric dual problem \eqref{prob:manifold_dual}, we know that,
	\begin{align}
		D^* \geq \min_{\theta } L(\theta,\mu,\lambda), \forall \ \mu, \lambda
	\end{align}
	We thus utilize the optimal dual variables of the functional problem with constraints $\epsilon-M\nu$, i.e. $\mu^*_{\nu}, \lambda^*_{\nu}$,  as follows, 
	\begin{align}\label{eqn:dual_parametric_greater_perturbed_non_parametric}
		D^* \geq \min_{\theta, \rho } L(\theta,\mu^*_{\nu},\lambda^*_{\nu})  \geq\min_{\phi } L(\phi,\mu^*_{\nu},\lambda^*_{\nu})  \geq  \min_{\phi } \tilde{L}_{\nu}(\phi,\mu^*_{\nu},\lambda^*_{\nu})  - \mu^*_\nu M\nu= \tilde{P}^*_{\nu}  - \mu^*_\nu M\nu.
	\end{align}
	Given that $\ccalH\subseteq\bar{\ccalH}$, and Problem \ref{prob:manifold_lipschitz_functional_Mnu} is strongly dual. To complete the proof we need to show that,
	\begin{align}
		\tilde{P}^*_{\nu}\geq P^* -2\nu P^*_{\phi_\nu^*} - \nu^2. 
	\end{align}
	Denoting $\phi^*_\nu$ the solution to problem \ref{prob:manifold_lipschitz_functional_Mnu}, by Assumption \ref{A:phi}, we know there is a parameterization $\tilde{\theta}_\nu^*$ such that simultaneously, 
	\begin{align}
		\sup_{z\in\ccalM} |\phi^*_\nu(z)-f_{\tilde{\theta}_\nu^*}(z)| \leq \nu\\
		\sup_{z\in\ccalM}  \|\nabla \phi^*_\nu(z)-\nabla f_{\tilde{\theta}_\nu^*}(z) \|  \leq \nu
	\end{align}
	therefore, 
	\begin{align}
		\left|  \mbE[\ell(\phi^*_\nu(x),y)]-\mbE[\ell(f_{\tilde{\theta}_\nu^*}(x),y)]  \right|&\leq\mbE\left[ | \ell(\phi^*_\nu(x),y)]-\ell(f_{\tilde{\theta}_\nu^*}(x),y)  | \right ]\\
		&\leq M \mbE \left[ | \phi^*_\nu(x) - f_{\tilde{\theta}_\nu^*}(x) | \right ]\leq M \nu
	\end{align}
	Since $\phi^*_\nu$ is feasible for Problem \ref{prob:manifold_lipschitz_functional_Mnu}, it implies that $f_{\tilde{\theta}_\nu^*}$ is feasible for Problem \ref{prob:manifold_lipschitz}. We can now define the $\rho_{\tilde{\theta}_\nu^*}^*$ as the minimum $\rho$ obtained with $f_{\tilde{\theta}_\nu^*}$ as follows, 
	\begin{align}\label{prob:manifold_lipschitz_particular_solution}
		\rho_{\tilde{\theta}_\nu^*}^*= \ \ \ \ \ \ \underset{\rho\geq 0}{\min}  \quad   
		& \rho, \\
		\text{subject to}  \quad 
		& \mbE_{p(x,y)} [\ell \left(f_{\tilde{\theta}_\nu^*}(x),y\right)]\leq \epsilon, \nonumber\\
		& \| \nabla_\ccalM f_{\tilde{\theta}_\nu^*}(z) \|^2 \leq \rho ,\quad p(z)\text{-a.e.}, \quad z \in \ccalM . \nonumber
	\end{align}
	Returning to \eqref{eqn:dual_parametric_greater_perturbed_non_parametric}, and by optimality, $P^*\leq  P_{\tilde{\theta}_\nu^*}^*$, 
	\begin{align}
		D^*&\geq P_{\phi_\nu^*}^* - \mu^*_\nu M\nu\\
		& \geq P_{\phi_\nu^*}^*+P^*-P_{\tilde{\theta}_\nu^*}^*- \mu^*_\nu M\nu \label{eqn:previous2last_parametric_to_functional}
	\end{align}
	Now we need to bound the difference $P_{\phi_\nu^*}^*-P_{\tilde{\theta}_\nu^*}^*$. We know that there exist $z_1,z_2 \in \ccalM$ such that $\|\nabla_\ccalM \phi_\nu^*(z_1) \|^2=P_{\phi_\nu^*}^*$ and $\| \nabla_\ccalM f_{\tilde{\theta}_\nu^*}  (z_2)\| ^2=P_{\tilde{\theta}_\nu^*}^*$, by optimality,
	
	\begin{align}
		P_{\phi_\nu^*}^*-P_{\tilde{\theta}_\nu^*}^*&=  \|\nabla_\ccalM \phi_\nu^*(z_1) \|^2-\| \nabla_\ccalM f_{\tilde{\theta}_\nu^*} (z_2)\| ^2\\
		&\geq  \|\nabla_\ccalM \phi_\nu^*(z_2) \|^2-\| \nabla_\ccalM f_{\tilde{\theta}_\nu^*} (z_2)\| ^2\\
		&\geq  \|\nabla_\ccalM \phi_\nu^*(z_2) \|^2 -(\| \nabla_\ccalM f_{\tilde{\theta}_\nu^*} (z_2)-\nabla_\ccalM \phi_\nu^*(z_2)\| + \|\nabla_\ccalM \phi_\nu^*(z_2) \| )\\
		&\quad\quad\quad(\| \nabla_\ccalM f_{\tilde{\theta}_\nu^*} (z_2)-\nabla_\ccalM \phi_\nu^*(z_2)\| + \|\nabla_\ccalM \phi_\nu^*(z_2) \|)\\
		&\geq -2\|  \nabla_\ccalM f_{\tilde{\theta}_\nu^*} (z_2)-\nabla_\ccalM \phi_\nu^*(z_2) \| \|\nabla_\ccalM \phi_\nu^*(z_2) \| -\|\nabla_\ccalM \phi_\nu^*(z_2) \| ^2\\
		&\geq -2\nu P_{\phi_\nu^*}^*-\nu^2  \label{eqn:last_parametric_to_functional}
	\end{align}
	where $B$ is a bound on the norm of $\|\nabla_\ccalM \phi(z)\|$. Putting \eqref{eqn:previous2last_parametric_to_functional} and \eqref{eqn:last_parametric_to_functional} together, we attain the desired result.
}

\myproof[of Proposition \ref{P:statistical_empirical_bound}]{
	
	This proof follows the lines of \cite{chamon2022constrained}[Proposition III.4]. We begin by considering $\mu^*,\lambda^*$, and $\hat \mu^*,\hat \lambda^*$, solutions of \ref{prob:manifold_dual}, and \ref{prob:emp_dual}, and we define the set of optimal dual minimizers as, 
	\begin{align}
		\Theta(\mu^*,\lambda^*)& =\argmin_{\theta\in\Theta} L (\theta,\mu^*,\lambda^*)\\
		\hat\Theta(\hat\mu^*,\hat\lambda^*)& =\argmin_{\theta\in\Theta} \hat L (\theta,\hat\mu^*,\hat\lambda^*)
	\end{align}
	where $L$, and $\hat L$ are as defined in \ref{prob:manifold_dual}, and \ref{prob:emp_dual}. We  can proceed to bound the difference between the values of the dual problems as follows, 
	\begin{align}
		D^* - \hat D^* &= \min_{\theta \in \Theta} L(\theta,\mu^*,\lambda^*) - \min_{\theta \in \Theta}\hat L(\theta,\hat\mu^*,\hat\lambda^*)\\
		&\leq \min_{\theta \in \Theta} L(\theta,\mu^*,\lambda^*) - \min_{\theta \in \Theta}\hat L(\theta,\mu^*,\bar\lambda^*)\text{\quad (by optimality)}\\
		&\leq  L(\theta^\dagger,\mu^*,\lambda^*) - \hat L(\hat\theta^\dagger,\mu^*,\bar\lambda^*)
	\end{align} 
	where $\bar \lambda ^*(z ) = \lambda^*(z) / \sum_{n=1}^N \lambda^*(z_n)$, and we define  $\hat\theta^\dagger\in\hat\Theta(\mu^*,\bar\lambda^*)$. By utilizing the definition of the Lagrangian, we can obtain,
	\begin{align}
		D^* - \hat D^* 
		&\leq |\mu^* | \bigg| \mbE [\ell (f_{\hat\theta^\dagger}(x),y)   ]- \frac{1}{N} \ell (f_{\hat\theta^\dagger}(x_n),y_n)\bigg| \\
		&  + \bigg| \mbE_{\lambda^*} \| \nabla_\ccalM f_{\hat\theta^\dagger} (x_n)\|^2 - \frac{1}{N} \sum_{n=1}^N \bar{\lambda}^*(x_n) \| \nabla_\ccalM f_{\hat\theta^\dagger}\|^2 \bigg|
	\end{align} 
	Note that the sampled dual variables, converge to the continuous values as follows, 
	\begin{align}
	    \lim_{N \to \infty}  \frac{\lambda^*(z)}  {\frac{1}{N}\sum_{n=1}^N \lambda^*(z_n) } = \lambda^*(z)
	\end{align}
	Given that $\lim_{N\to \infty} \frac{1}{N} \sum_{n=1}^N \lambda^*(z_n) = 1$
	Utilizing the same argument on the other direction of the inequality it yields, 
	\begin{align}
		D^* - \hat D^* 
		&\geq L(\theta^\dagger, \hat \mu ^*, \hat \lambda ^* ) - \hat L(\theta^\dagger, \hat \mu ^*, \hat \lambda ^* ) \\
		&\geq|\hat\mu^* | \bigg| \mbE [\ell (f_{\theta^\dagger}(x),y)   ]- \frac{1}{N} \ell (f_{\theta^\dagger}(x_n),y_n)\bigg| \\
		&  + \bigg| \mbE_{\bar\lambda^*} \| \nabla_\ccalM f_{\theta^\dagger} (x_n)\|^2 - \frac{1}{N} \sum_{n=1}^N \lambda^*(x_n) \| \nabla_\ccalM f_{\hat\theta^\dagger}\|^2 \bigg|
	\end{align} 
	Where $\bar \lambda^* = \frac{1}{N}\sum_{n=1}^N \hat\lambda^*(x_n) B_{\alpha}(x_n)$, where $B_{\alpha}(x_n)$ is a ball of center $x_n$, and radius $\alpha$. By tending $\alpha\to 0$, and $N\to \infty$, this integral converges, given the compactness of the manifold. Therefore, with probability $1-2\delta$, it holds, 
	\begin{align}
		\bigg|D^*-\hat D^ *\bigg| \leq \max \bigg\{   | L(\theta^\dagger,\mu^*,\lambda^*) - \hat L(\hat\theta^\dagger,\mu^*,\bar\lambda^*) |      ,  |L(\theta^\dagger, \hat \mu ^*, \hat \lambda ^* ) - \hat L(\theta^\dagger, \hat \mu ^*, \hat \lambda ^* ) |       \bigg\}
	\end{align}
	Now we utilize the Lipschitzness of the gradient of $\nabla_\ccalM f_\theta$ by assumption \ref{A:phi}, and given that the set $\Theta$ is compact, and the manifold $\ccalM$ is compact, 
	\begin{align}
	    &\bigg| \| \nabla_\ccalM f_{\theta_1} (x)\|^2 - \| \nabla_\ccalM f_{\theta_2}(x)\|^2 \bigg|\\
	    &\leq 2 \max_{\theta\in\Theta, z\in \ccalM} \| \nabla_\ccalM f_{\theta} (z)\|  \| \nabla_\ccalM f_{\theta_1}(x)- \nabla_\ccalM f_{\theta_2}(x)\| \text{\ by triangle inequality,}\\
	    &\leq 2 R G | f_{\theta_1}(x)-  f_{\theta_2}(x)| \text{\ by compactness and lipschitz.}
	\end{align}
	Where $\max_{\theta\in\Theta, z\in \ccalM} \| \nabla_\ccalM f_{\theta} (z)\| \leq R<\infty $ by compactness of $\Theta$, and $\ccalM$. To complete the proof, we leverage Talagrand’s lemma \cite{mohri2018foundations}[Lemma 5.7], finally obtaining, %
	\begin{align}
		\bigg|D^*-\hat D^ *\bigg| \leq \max \bigg\{   |\mu^*| , |\hat\mu^*|  \bigg\} \zeta(N,\delta) + 2R G\hat \zeta(N,\delta)
	\end{align}
	By leveraging Lemma \ref{lemma:duality_gap_grad_lip} we attain the desired result.
	
}


\newpage
\section{Proof Of Proposition \ref{P:solution_bound}}
\myproof[of Proposition \ref{P:solution_bound}]{
	
	To prove that $\etheta^*$ is approximately feasible, we use the same argument that in Lemma \ref{lemma:duality_gap_grad_lip}. By contradiction, if there exists no $\etheta^\dagger \in \eTheta(\emu^*,\elambda^*)$ that is feasible, the supergradient,
	\begin{align}
		\frac{1}{N} \sum_{n=1}^N \ell(f_{\etheta^\dagger}(x_n),y_n) > \epsilon
	\end{align}
	and therefore $\bb0 \notin \partial \eTheta(\etheta^*,\elambda^*)$ which contradicts the optimality of $\etheta^*$, and $\elambda^*$. Therefore, there exists $\etheta^* \in \eTheta(\emu^*,\elambda^*)$ such that,
	\begin{align}
		\mbE[	\ell(f_{\etheta^\dagger}(x),y) ]\leq \epsilon + \zeta(N,\delta)
	\end{align}
	with probability $1-\delta$. We can analyze the term given by the summation of the dual variables $\lambda$, noting that if there is no solution $\elambda^*$, such that 
	\begin{align}
		\frac{1}{N}\sum_{n=1}^N\lambda(x_n) \| \nabla_\ccalM f_{\theta }(x_n)\|^2 = \max_{n\in[N]} \| \nabla f_{\theta}(x_n)\|^2
	\end{align}
	then, utilizing the same argument, $\bb0\notin \partial \eTheta(\etheta^*,\elambda^*)$, contradicting the optimality of $\etheta^*$. 
	
	Now, we know that the norm of the gradient $\| \nabla_\ccalM f_{\etheta^\star}(x_n)\|$, is smaller than a value $\sqrt{\rho_{\etheta^\star}}$ for all $x_n$. To bound the maximum gradient over all $z \in \ccalM$, we can leverage the fact that the manifold is compact, and that the gradients are Lipschitz as follows, 
	
	\begin{align}
		\| \nabla_\ccalM f_{\etheta^\star}(x_n)\| &\leq  	\| \nabla_\ccalM f_{\etheta^\star}(z)\|  + 	\| \nabla_\ccalM f_{\etheta^\star}(x_n) - \nabla_\ccalM f_{\etheta^\star}(z)\|  \ \forall \ z \in \ccalN(x_n)\\
		&\leq  \sup_{z\in \ccalN(x_n)}	\| \nabla_\ccalM f_{\etheta^\star}(z)\|  + 	\| \nabla_\ccalM f_{\etheta^\star}(x_n) - \nabla_\ccalM f_{\etheta^\star}(z)\|   \\
		&\leq  \sup_{z\in \ccalM}	\| \nabla_\ccalM f_{\etheta^\star}(z)\|  + \sup_{z\in \ccalN(x_n)}		\| \nabla_\ccalM f_{\etheta^\star}(x_n) - \nabla_\ccalM f_{\etheta^\star}(z)\|    \\
		&\leq  \sup_{z\in \ccalM}	\| \nabla_\ccalM f_{\etheta^\star}(z)\|  + \sup_{z\in \ccalN(x_n)}		d( x_n, z)    \\
		&\leq  \sup_{z\in \ccalM}	\| \nabla_\ccalM f_{\etheta^\star}(z)\|  + \gradLip \tilde\zeta(N)\label{eqn:norm_gradient_over_xn}
	\end{align}
	where $\{\ccalN(x_n): x_n = \argmin_{n\in[N]} d_\ccalM(x_n,z)   \}$, which corresponds to the set of points in $\ccalM$ that are closer to $x_n$. 
	Function $\tilde \zeta(N)$ is a decreasing function that measures the maximum distance between a point sampled by $z\in\ccalM$, and a point in $\ccalN(x_n)$, this number decreases with $N$ given that $\ccalM$ is compact. Now, we can evaluate $\eD^*$ at $\etheta^*$ as follows, 
	\begin{align}
		\eD^*& = \mu^* \big(\sum_{n=1}^N \ell(f_{\etheta^*}(x_n) ,y_n) -\epsilon \big) + \frac{1}{N}\sum_{n=1}^N \elambda^*(x_n) \| \nabla_\ccalM f_{\etheta^*}(x_n)\|^2\\
		& = \mu^* \big(\sum_{n=1}^N \ell(f_{\etheta^*}(x_n) ,y_n) -\epsilon \big) + \max_{n\in[N]} \|\nabla_\ccalM f_{\etheta^*}(x_n)\|^2 \text{ (by optimality of }\elambda^*)
	\end{align}
	To conclude the proof, we leverage Proposition \ref{P:statistical_empirical_bound}, in conjunction with \eqref{eqn:norm_gradient_over_xn} and we get that, 
	\begin{align}
		|P^* - \sup_{z\in\ccalM} \|\nabla_\ccalM f_{\etheta^*}(z) |^2 |\leq& |\emu^*|  \big(\sum_{n=1}^N \ell(f_{\etheta^*}(x_n) ,y_n) -\epsilon \big) \\
		&+ G\tilde\zeta(N) + \ccalO(\nu) + \max \bigg\{   |\mu^*| , |\hat\mu^*|  \bigg\} \zeta(N,\delta) + 2GR\hat \zeta(N,\delta)
	\end{align}
	
}

\newpage
\section{Proof of Proposition \ref{P:pointcloud}}
The proof of Proposition \ref{P:pointcloud} utilizes the following Definitions.


\begin{definition}\label{def:Manifold_volume_V}
	The total volume of the smooth Manifold $\ccalM$ is $\bbV$,i.e.,
	\begin{align}
		\int_\ccalM vol(x) = \bbV.
	\end{align}
\end{definition}

\begin{definition}\label{def:upper_bound_F}
	The function $f_\theta$, parameterized by $\theta \in \Theta$, is universally upper bounded by $\bbF$, i.e. 
	\begin{align}
		\max_{x\in\ccalM,\theta \in \Theta}\| f_\theta(x)\|\leq \bbF
	\end{align}
\end{definition}

\begin{definition}\label{def:degree_probability}
	Given a probability distribution $\lambda \in \Lambda$ defined over the manifold $\ccalM$, we define the degree $d_t(x)$ as,
	\begin{align}
		d_t(p) = \int_\ccalM G_t(p,z) \lambda(z)vol(z) 
	\end{align}
\end{definition}

\begin{definition}\label{def:continuous_heat_kernel_laplacian}
	We define the continuous heat kernel laplacian $\tilde\bbL_{\lambda}^t f_\theta(p)$ associated with dual function $\lambda\in\Lambda$, function $ f_\theta,\theta \in \Theta$ and point $p\in\ccalM$ as, 
	\begin{align}
		\tilde\bbL_\lambda^t f_\theta (p) =\frac{1}{t} \int_{\ccalM} \frac{G_t(p,z)}{\sqrt{d_t(p)}\sqrt{d_t(z)}}(f_\theta(p)-f_\theta(z))\lambda(z)vol(z)
	\end{align}
	were $d_t$ is the degree function defined in \ref{def:degree_probability}. 
\end{definition}

\begin{lemma}\label{lemma:ball_integral}
	For a given $p\in \ccalM$, and any open set $\ccalB \subset \ccalM, p\in \ccalB$, for any function $f_\theta, \theta \in \Theta$, such that $\sup_{\theta \in \Theta,x \in \ccalM} |f(x)|\leq \bbF$, and any $\lambda \in \Lambda$ with $\Lambda$ as in Proposition \ref{P:pointcloud}, we have that
	
	\begin{align}
		\bigg| \int_{\ccalB} e^{-\frac{\|p-z\|^2}{4t}} \lambda(z) f_\theta (z) vol(z)- \int_\ccalM e^{-\frac{\|p-z\|^2}{4t}} \lambda(z) f_\theta (z) vol(z) \bigg| \leq b \bbV \bbF e^{-\frac{r^2}{4t}},
	\end{align}
	where $r=\inf_{x\notin \ccalB}\|x-p \|$, and $b$ in $\Lambda$. 
\end{lemma}
\myproof[of Lemma \ref{lemma:ball_integral}]{
	The proof is similar to Lemma 4.1 in \cite{belkin2005towards}. The proof follows from bounding $\lambda$ by $b$ from assumption, $\|f_\theta\|$ by $\bbF$ from definition \ref{def:upper_bound_F}, the volume of $\ccalM-\ccalB$ by the volume of $\ccalM$ given by $\bbV$, and bounding $e^{-\frac{\|p-z\|^2}{4t}}$ by $e^{-\frac{d}{4t}}$.$\QED$ }

\begin{lemma}\label{lemma:min_degree}
	The limit of the degree function $d_t(x)$ (defined in Definition \ref{def:degree_probability}) when  $t\to 0$ is lower bounded for probabilities $\lambda \in\Lambda$ as in Proposition \ref{P:pointcloud},i.e., 
	\begin{align}
		\min_{p\in \ccalM,\lambda \in \Lambda} \lim_{t\to 0} d_t(x) &= a,
	\end{align}
	with $a = \min_{z\in \ccalM} \lambda(z)$ as given by $\lambda \in \Lambda$.
\end{lemma}

\myproof[of Lemma \ref{lemma:min_degree}]{
	To begin with, we can write down the definition of the degree function $d_t(x)$,
	\begin{align}
		\min_{p\in \ccalM,\lambda \in \Lambda} \lim_{t\to 0} d_t(x) &=  \min_{p\in \ccalM,\lambda \in \Lambda}\lim_{t\to 0}  \int_\ccalM G_t(x,z) \lambda(z)vol(z)  \\
		&=  \min_{p\in \ccalM,\lambda \in \Lambda} \lim_{t\to 0}  \int_\ccalM     \frac{1}{(4 \pi t)^{d/2}} \, e^{-\frac{\| x - z  \|^2}{4t}} \lambda(z)vol(z) \\
		&=  \min_{p\in \ccalM} \lim_{t\to 0}   a \int_\ccalM     \frac{1}{(4 \pi t)^{d/2}} \, e^{-\frac{\| x - z  \|^2}{4t}} vol(z) \\
		&=   \min_{p\in \ccalM}  \lim_{t\to 0}  a  \int_\ccalM     \frac{1}{(4 \pi t)^{d/2}} \, e^{- \frac{( x - z)^{\intercal} (\frac{1}{2t}\bbI) (x-z)}{2}} vol(z) \\
		&= \min_{p\in \ccalM}   \lim_{t\to 0} a  \int_\ccalB     \frac{1}{(4 \pi t)^{d/2}} \, e^{- \frac{( x - z)^{\intercal} (\frac{1}{2t}\bbI) (x-z)}{2}} vol(z) \\
		&=  \min_{p\in \ccalM} \lim_{t\to 0}   a \int_{\tilde\ccalB}     \frac{1}{(4 \pi t)^{d/2}} \, e^{- \frac{(  v)^{\intercal} (\frac{1}{2t}\bbI) (v)}{2}} vol(z) \\
		& = \frac{a}{(4 \pi t)^{d/2}}  ( (2 \pi )^{d}   (2t)^d )^{\frac{1}{2}}   = a 
	\end{align}

}

\begin{lemma}\label{lemma:euclidean_vs_geodesic_distance}
	For any two sufficietly close points $p,q \in \ccalM$, such that $q=\text{exp}_p(v)=\reals^k$, the relationship between the Euclidean distance and geodesic distance is given by,
	\begin{align}
		d_\ccalM^2(p,q)= \|v\|^2_{\reals^k}
	\end{align}
\end{lemma}
\myproof[of Lemma \ref{lemma:euclidean_vs_geodesic_distance}]{
	Consider a geodesic curve $\gamma$, that goes from $\gamma(0)=p$ to $\gamma(1)=q$, this geodesic can be obtained by $\gamma(t)=\text{exp}_p(tv)$. The distance between $p,q$ can be expressed as, 
	\begin{align}
		d_\ccalM(p,q)= \int_0^1 \|\gamma^{'}(t)\|dt=\|v\|.
	\end{align}
	given that the geodesic has constant derivative, and it is given by $v$ by the definition of exponential map (see \cite{do1992riemannian}[Proposition 3.6]). 
	$\QED$}

\begin{lemma}\label{lemma:bounded_diffence_between_laplacians}
	For any point $p\in\ccalM$, the probability of the difference between the point-cloud Laplacian operator $ \bbL_{\lambda,N}^t f_\theta(p)$ defined in equation \ref{def:point_cloud_laplacian} and the continuous heat kernel laplacian is given by
	\begin{align}\label{eqn:bounded_diffence_between_laplacians}
		P(|\tilde\bbL_{\lambda}^t f_\theta (p) - \bbL_{\lambda,N} f_\theta(p) | \geq \epsilon) \leq  2 e^{\frac{-\epsilon^2 n }{8 \bbB^2_t}}+ 2 e^{  -\frac{\epsilon^2 (n-1)}{4\bbB_t a^{3/2}}  } ,
	\end{align}
	where $B_t$ is an upper bound on the norm of the random variable being integrated given by $ \frac{1}{t}\frac{1}{(4\pi t)^{k/2}} \frac{2\bbF}{a}$. 
\end{lemma}
\myproof[of Lemma \ref{lemma:bounded_diffence_between_laplacians}]{
	The proof is similar to section $5.1$ of Theorem $5.2$ from \cite{belkin2005towards}. To begin with, we define a intermediate operator $\hat \bbL_{\lambda,n}f_\theta(p)$ as,
	\begin{align}
		\hat \bbL_{\lambda,n}f_\theta(p) = \frac{1}{nt} \sum_{i=1}^N \frac{G_t(p,x_i)}{\sqrt{d_t(p)}\sqrt{d_t(x_i)}}(f_\theta(p)-f_\theta(x_i))
	\end{align}
	Now we consider equation \eqref{eqn:bounded_diffence_between_laplacians}, adding and subtracting the intermediate operator $\hat \bbL_{\lambda,n}f_\theta(p)$ as follows, 
	\begin{align}
		&P(|\tilde\bbL_{\lambda}^t f_\theta (p) - \bbL_{\lambda,N} f_\theta(p) | \geq \epsilon)\\
		&= P(|\tilde\bbL_{\lambda}^t f_\theta (p) - \hat \bbL_{\lambda,n}f_\theta(p) +\hat \bbL_{\lambda,n}f_\theta(p) - \bbL_{\lambda,N} f_\theta(p) | \geq \epsilon)\\
		&= 1- P(|\tilde\bbL_{\lambda}^t f_\theta (p) - \hat \bbL_{\lambda,n}f_\theta(p) +\hat \bbL_{\lambda,n}f_\theta(p) - \bbL_{\lambda,N} f_\theta(p) | \leq \epsilon) \label{eqn:bounded_diffence_between_laplacians_complement}\\
		&= 1- P(|\tilde\bbL_{\lambda}^t f_\theta (p) - \hat \bbL_{\lambda,n}f_\theta(p)|\leq\epsilon/2\cap |\hat \bbL_{\lambda,n}f_\theta(p) - \bbL_{\lambda,N} f_\theta(p) | \leq \epsilon/2) \label{eqn:bounded_diffence_between_laplacians_intersect}\\
		&=  1- (1-P(|\tilde\bbL_{\lambda}^t f_\theta (p) - \hat \bbL_{\lambda,n}f_\theta(p)|>\epsilon/2 \cup P(|\hat \bbL_{\lambda,n}f_\theta(p) - \bbL_{\lambda,N} f_\theta(p) | > \epsilon/2) \label{eqn:bounded_diffence_between_laplacians_complement_again}\\
		&\leq  P(|\tilde\bbL_{\lambda}^t f_\theta (p) - \hat \bbL_{\lambda,n}f_\theta(p)|>\epsilon/2)+P(|\hat \bbL_{\lambda,n}f_\theta(p) - \bbL_{\lambda,N} f_\theta(p) | > \epsilon/2) \label{eqn:bounded_diffence_between_laplacians_union_bound}
	\end{align}
	where equation \eqref{eqn:bounded_diffence_between_laplacians_complement} holds by taking the complement, equation \eqref{eqn:bounded_diffence_between_laplacians_intersect} holds given that it is a subset of the total probability, equation \eqref{eqn:bounded_diffence_between_laplacians_complement_again} by taking complement again, and equation \eqref{eqn:bounded_diffence_between_laplacians_union_bound} by the union bound. 
	
	We will now focus on the first term of \eqref{eqn:bounded_diffence_between_laplacians_union_bound}, i.e. $P(|\tilde\bbL_{\lambda}^t f_\theta (p) - \hat \bbL_{\lambda,n}f_\theta(p)|>\epsilon/2)$, and by considering the event as follows,
	
	\begin{align}
		&|\tilde\bbL_{\lambda}^t f_\theta (p) - \hat \bbL_{\lambda,n}f_\theta(p)| \\
		&= \bigg|\int_{\ccalM} \frac{1}{t}\frac{G_t(p,z)}{\sqrt{d_t(p)}\sqrt{d_t(z)}}(f_\theta(p)-f_\theta(z))\lambda(z)vol(z) - \frac{1}{tn} \sum_{i=1}^N \frac{G_t(p,x_i)}{\sqrt{d_t(p)}\sqrt{d_t(x_i)}}(f_\theta(p)-f_\theta(x_i))\bigg|\nonumber
	\end{align}
	Now we can note that for every $t>0$, we can bound, 
	\begin{align}
		\bigg|\frac{1}{t}\frac{G_t(p,z)}{\sqrt{d_t(p)}\sqrt{d_t(z)}}(f_\theta(p)-f_\theta(z))\bigg| \leq \frac{1}{t}\frac{1}{(4\pi t)^{k/2}} \frac{2\bbF}{a} = \bbB_t
	\end{align}
	where $a$ as in $\Lambda$, and the degree is lower bounded by Lemma \ref{lemma:min_degree}. By the application of Hoeffding's inequality, we obtain, 
	\begin{align}
		P(|\tilde\bbL_{\lambda}^t f_\theta (p) - \hat \bbL_{\lambda,n}f_\theta(p)|>\epsilon/2) \leq 2 e^{\frac{-\epsilon^2 n }{8 \bbB^2_t}}
	\end{align}
	
	We will now focus on the second term of \eqref{eqn:bounded_diffence_between_laplacians_union_bound}, i.e. $P(|\hat\bbL_{\lambda,N}^t f_\theta (p) -  \bbL_{\lambda,N}f_\theta(p)|>\epsilon/2)$, and by considering the event as follows,
	
	\begin{align}
		&|\hat\bbL_{\lambda,N}^t f_\theta (p) -  \bbL_{\lambda,N}f_\theta(p)| \\
		&= \bigg|\frac{1}{tn} \sum_{i=1}^N \frac{G_t(p,x_i)}{\sqrt{d_t(p)}\sqrt{d_t(x_i)}}(f_\theta(p)-f_\theta(x_i))- \frac{1}{tn} \sum_{i=1}^N \frac{G_t(p,x_i)}{\sqrt{\hat{d_t}(p)}\sqrt{\hat{d_t}(x_i)}}(f_\theta(p)-f_\theta(x_i))\bigg|\nonumber\\
		&\leq \frac{1}{tn}\frac{2\bbF}{(4\pi t)^{k/2}} \sum_{i=1}^N \bigg| \frac{1}{\sqrt{d_t(p)}\sqrt{d_t(x_i)}}-  \frac{1}{\sqrt{\hat{d_t}(p)}\sqrt{\hat{d_t}(x_i)}}\bigg|\nonumber \\
		&\leq \frac{1}{tn}\frac{2\bbF}{(4\pi t)^{k/2}} \sum_{i=1}^N \bigg| \frac{\sqrt{d_t(p)}\sqrt{d_t(x_i)} - \sqrt{\hat{d_t}(p)}\sqrt{\hat{d_t}(x_i)}  }{\sqrt{d_t(p)}\sqrt{d_t(x_i)} \sqrt{\hat{d_t}(p)}\sqrt{\hat{d_t}(x_i)}  }\bigg|\nonumber   \\
		&\leq \frac{1}{tn}\frac{2\bbF}{(4\pi t)^{k/2}} \sum_{i=1}^N \bigg| \frac{\sqrt{d_t(p)}\sqrt{d_t(x_i)} - \sqrt{\hat{d_t}(p)}\sqrt{\hat{d_t}(x_i)} + \sqrt{{d_t}(p)}\sqrt{\hat{d_t}(x_i)}  -\sqrt{{d_t}(p)}\sqrt{\hat{d_t}(x_i)}  }{\sqrt{d_t(p)}\sqrt{d_t(x_i)} \sqrt{\hat{d_t}(p)}\sqrt{\hat{d_t}(x_i)}  }\bigg|\nonumber \\
		&\leq \frac{1}{tn}\frac{2\bbF}{(4\pi t)^{k/2}} \sum_{i=1}^N \bigg| \frac{\sqrt{d_t(p)} ( \sqrt{\hat d_t(x_i)} - \sqrt{d_t(x_i)} ) + \sqrt{\hat{d_t}(x_i)} (\sqrt{{d_t}(p)} -\sqrt{\hat{d_t}(p)}) }{\sqrt{d_t(p)}\sqrt{d_t(x_i)} \sqrt{\hat{d_t}(p)}\sqrt{\hat{d_t}(x_i)}  }\bigg|\nonumber   \\ 
	\end{align}
	Where we can use the Hoeffding's inequality given that the random variables are bounded, as follows,
	\begin{align}
		P[ | \hat{d_t}(x_i) - d_t(x_i) | > \epsilon ] \leq 2e ^{-\frac{\epsilon^2 (n-1)}{\bbB_t}}
	\end{align}
	Finally, we obtain, 
	\begin{align}
		P[|\hat\bbL_{\lambda,N}^t f_\theta (p) -  \bbL_{\lambda,N}f_\theta(p)| > \epsilon /2] \leq 2 e^{  -\frac{\epsilon^2 (n-1)}{4\bbB_t a^{3/2}}  }
	\end{align}
	All together, we obtain, 
	\begin{align}
		& P(|\tilde\bbL_{\lambda}^t f_\theta (p) - \bbL_{\lambda,N} f_\theta(p) | \geq \epsilon)  \\
		&\leq P(|\tilde\bbL_{\lambda}^t f_\theta (p) - \hat \bbL_{\lambda,n}f_\theta(p)|>\epsilon/2)  +  P(|\tilde\bbL_{\lambda}^t f_\theta (p) - \hat \bbL_{\lambda,n}f_\theta(p)|>\epsilon/2)  \\
		& \leq 2 e^{\frac{-\epsilon^2 n }{8 \bbB^2_t}}+ 2 e^{  -\frac{\epsilon^2 (n-1)}{4\bbB_t a^{3/2}}  } 
	\end{align}
	Completing the proof. 
	
	
}

\begin{lemma}\label{lemma:convergence_euclidean_to_laplacian}

	Consider $f_\theta$ with $\theta \in \Theta$, and $\lambda \in \Lambda$, for any point $p\in \ccalM$, there is a uniform bound between the continuous heat kernel laplacian (cf .definition \ref{def:continuous_heat_kernel_laplacian})and the the laplace-beltrami operator $\Delta_\lambda f_\theta= \Delta f_\theta + \langle \nabla_\ccalM f_\theta , \nabla_\ccalM \lambda  \rangle$, where $\Delta$ is the Laplace-Beltrami operator, i.e.
	\begin{align}
		| \tilde \bbL_\lambda^t f_\theta(p) -  \Delta_\lambda f_\theta (p) | \leq \ccalO(t^{1/2})
	\end{align}
	with $\tilde\bbL_\lambda^t f_\theta (p) =\frac{1}{t} \int_{\ccalM} \frac{G_t(p,z)}{\sqrt{d_t(p)}\sqrt{d_t(z)}}(f_\theta(p)-f_\theta(z))\lambda(z)vol(z)$ .
\end{lemma}
\myproof[of Lemma \ref{lemma:convergence_euclidean_to_laplacian}]{
	The proof is as follows, first we will compare the heat kernel laplacian $\tilde \bbL_\lambda^t f_\theta(p) $ with the heat kernel laplacian with $\lambda$ instead of $d_t$. Next, we will consider the integral over a ball $\ccalB$, as opposed to the integral over the whole manifold $\ccalM$. Then, we will exploit the euclidean properties of the manifold at point $p$, which will allow us to convert an integral over a manifold into an integral over the low dimensional structure of the manifold. In this low dimensional manifold, we can compute the integrals, and obtain the desired result.

	To begin with, we will introduce the continuous heat kernel laplacian, but using the probability distribution $\lambda$ as follows, 
	\begin{align}
		&\tilde\bbL_\lambda^t f_\theta (p) =\frac{1}{t} \int_{\ccalM} \frac{G_t(p,z)}{\sqrt{d_t(p)}\sqrt{d_t(z)}}(f_\theta(p)-f_\theta(z))\lambda(z)vol(z) \\
		&\leq \frac{1}{t}\int_\ccalM \frac{G_t(p,z)}{\sqrt{\lambda(p)}\sqrt{\lambda(z)} } (f_\theta(p) - f_\theta(z))\lambda(z)vol(z)\\
		&+| \frac{1}{t} \int_{\ccalM} \frac{G_t(p,z)}{\sqrt{d_t(p)}\sqrt{d_t(z)}}(f_\theta(p)-f_\theta(z))\lambda(z)vol(z) -\\
		&\frac{1}{t}\int_\ccalM \frac{G_t(p,z)}{\sqrt{\lambda(p)}\sqrt{\lambda(z)} } (f_\theta(p) - f_\theta(z))\lambda(z)vol(z)| \\
		&\leq \frac{1}{t}\int_\ccalM \frac{G_t(p,z)}{\sqrt{\lambda(p)}\sqrt{\lambda(z)} } (f_\theta(p) - f_\theta(z))\lambda(z)vol(z)\\
		&+\frac{1}{t} \int_{\ccalM} G_t(p,z) | \frac{1}{\sqrt{d_t(p)}\sqrt{d_t(z)}} - \frac{1}{\sqrt{\lambda(p)}\sqrt{\lambda(z)}}   |  |f_\theta(p)-f_\theta(z)|\lambda(z)vol(z) \label{eqn:degree_lambda_inverse_multiplicative_bound}
	\end{align}

	By compactness of manifold $\ccalM$, and the fact that $\lim_{t\to 0}d_t(p)=\lambda(p)$ the following holds for any point $p\in\ccalM$,
	\begin{align}
		d_t(p) = \lambda(p)+ \ccalO (tg(p))
	\end{align} 
	where $g$ is a smooth function depending upon higher derivatives of $\lambda$, which is bounded by $c$. Therefore, we can write 
	\begin{align}
		d^{-1/2}_t(z)  = (\lambda(z)+ \ccalO (tg(z)))^{-1/2}=\lambda(z)^{-1/2}+\ccalO(t),
	\end{align}
	where $\ccalO(t)\leq c |t|$. Therefore, we can bound the right hand side of \ref{eqn:degree_lambda_inverse_multiplicative_bound} by, 
	\begin{align}
		&\frac{1}{t} \int_{\ccalM} G_t(p,z) | \frac{1}{\sqrt{d_t(p)}\sqrt{d_t(z)}} - \frac{1}{\sqrt{\lambda(p)}\sqrt{\lambda(z)}}   |  |f_\theta(p)-f_\theta(z)|\lambda(z)vol(z) \\
		& \leq \frac{1}{t} \int_{\ccalM} G_t(p,z) c t |f_\theta(p)-f_\theta(z)|\lambda(z)vol(z) \\
		& \leq  \int_{\ccalM} G_t(p,z) c  |f_\theta(p)-f_\theta(z)|\lambda(z)vol(z) \leq  \ccalO(t)
	\end{align}
	where the bound holds uniformly given the compactness of the manifold $\ccalM$, and the upper bound on the derivative of $f_\theta$. We now return to equation \eqref{eqn:degree_lambda_inverse_multiplicative_bound}, by virtue of Lemma \ref{lemma:ball_integral}, we can convert the integral over the manifold \ccalM, to an integral over the ball $\ccalB$ as follows, 
	
	\begin{align}
&\frac{1}{t}\int_\ccalM \frac{G_t(x,y)}{\sqrt{\lambda(x)}\sqrt{\lambda(y)} } (f_\theta(x) - f_\theta(y))\lambda(y)vol(y) \\
&=  \frac{1}{t}\int_\ccalB \frac{G_t(x,y)}{\sqrt{\lambda(x)}\sqrt{\lambda(y)} } (f_\theta(x) - f_\theta(y))\lambda(y)vol(y)\\ 
&\quad+\frac{1}{t}\int_\ccalM\frac{G_t(x,y)}{\sqrt{\lambda(x)}\sqrt{\lambda(y)} } (f_\theta(x) - f_\theta(y))\lambda(y)vol(y)\\
&\quad\quad-\frac{1}{t}\int_\ccalB \frac{G_t(x,y)}{\sqrt{\lambda(x)}\sqrt{\lambda(y)} } (f_\theta(x) - f_\theta(y))\lambda(y)vol(y) \\
& \leq  \frac{1}{t}\int_\ccalB \frac{G_t(x,y)}{\sqrt{\lambda(x)}\sqrt{\lambda(y)} } (f_\theta(x) - f_\theta(y))\lambda(y)vol(y)\\ 
&+\bigg|\frac{1}{t}\int_\ccalM\frac{G_t(x,y)}{\sqrt{\lambda(x)}\sqrt{\lambda(y)} } (f_\theta(x) - f_\theta(y))\lambda(y)vol(y)\\
&\quad\quad-\frac{1}{t}\int_\ccalB \frac{G_t(x,y)}{\sqrt{\lambda(x)}\sqrt{\lambda(y)} } (f_\theta(x) - f_\theta(y))\lambda(y)vol(y)\bigg| \\
& \leq  \frac{1}{t}\int_\ccalB \frac{G_t(x,y)}{\sqrt{\lambda(x)}\sqrt{\lambda(y)} } (f_\theta(x) - f_\theta(y))\lambda(y)vol(y) +  b \bbV \bbF e^{-\frac{r^2}{4t}}
	\end{align}
	We now consider the exponential coordinates $\text{exp}$ around $p$, after introducing a ball $\ccalB$ of radius $r$. Letting the change of variables be $\text{exp}_p:T_p \ccalM \to \ccalM$. Letting $\tilde \ccalB$ be a ball in $T_p \ccalM$, then $\ccalB$ is the image of the ball under the exponential map. We will use the exponential coordinates as $v = \text{exp}_p(z)$, and $\tilde f_\theta (v) = f_\theta (\text{exp}_p(v))  $, and $\tilde f_\theta (0) = f_\theta(\text{exp}_p(0))$.
	\begin{align}
		\frac{1}{t}\int_\ccalB \frac{G_t(x,y)}{\sqrt{\lambda(x)}\sqrt{\lambda(y)} } (f_\theta(x) - f_\theta(y))\lambda(y)vol(y) & =\\
		\frac{1}{t} \frac{1}{\sqrt{\tilde\lambda(0)}}\int_{\tilde\ccalB} \frac{e^{\frac{\|\text{exp}_p(v)-\text{exp}_p(0) \|^2}{4t}}}{(4\pi t)^{k/2}} \sqrt{\tilde\lambda(v)}  (\tilde f_\theta(0) - \tilde f_\theta(v))\text{det}(g(v))dv \label{eqn:integral_in_the_ball}
	\end{align}
	where $\text{det}(g(v))$ is the determinant of the metric tensor in exponential coordinates. We can now expand $\sqrt{\tilde\lambda}(v)$, and $\tilde f_\theta (v)$ by their Taylor expansions as follows,
	\begin{align}
		\tilde{\lambda}^{1/2}(v) &\leq \tilde \lambda^{1/2}(0) +\frac{1}{2}\tilde\lambda^{-1/2}(0) v^\intercal \nabla \tilde \lambda (0) + \ccalO( \|v\|^2 ) \\
		\tilde f_\theta (v) & =  \tilde f_\theta (0) + v^\intercal \nabla\tilde f_\theta (0) + \frac{1}{2}v^\intercal \bbH v + \ccalO( \|v\|^3 )\\
		\tilde f_\theta (0) -\tilde f_\theta (v) &\leq  - v^\intercal \nabla\tilde f_\theta (0) - \frac{1}{2}v^\intercal \bbH v + \ccalO(  \|v\|^3)
	\end{align}
	Where $\nabla \tilde f  = (\frac{\partial \tilde f}{\partial x_1}, \dots, \frac{\partial \tilde f}{\partial x_k})^\intercal$, and $\bbH$ its corresponding hessian matrix. 
	Now note that given that the norm of the second derivative $\lambda$  and third derivative of $f_\theta$ are uniformly bounded for all $p\in\ccalM$, $\theta\in\Theta$ and $\lambda\in\Lambda$.  With these two approximations in hand, we also approximate the heat kernel $e^{\frac{\|\text{exp}_p(v)-\text{exp}_p(0) \|^2}{4t}}$ as follows, 
	\begin{align}
		e^{-\frac{ d_\ccalM(\text{exp}_p(v)-p )^2}{4t}} = e^{-\frac{\|v\|^2}{4t}} ,
	\end{align}
	given that $y$ is close enough to $p$.  We also approximate the determinant of the metric tensor in exponential coordinates $\text{det}(g(v))$ using the fact that 
	\begin{align}
		\text{det}(g(v)) & = 1-\frac{1}{6}z^\intercal \bbR v +\ccalO(\|v\|^3)\\
		& \leq  1 + k_0 \| v\|^2.
	\end{align}
	where $\bbR$ is the Ricci curvature tensor which is uniformly bounded given the compactness of $\ccalM$. The last inequality holds uniformly for all $p \in \ccalM$ given the compactness of $\ccalM$.  
	Returning to \ref{eqn:integral_in_the_ball}, putting all the pieces together we obtain,
	\begin{align}
		& \frac{1}{t}\int_\ccalB \frac{G_t(x,y)}{\sqrt{\lambda(x)}\sqrt{\lambda(y)} } (f_\theta(x) - f_\theta(y))\lambda(y)vol(y)\\
		&\leq \frac{1}{t} \frac{1}{\sqrt{\tilde\lambda(0)}}\frac{1}{(4\pi t)^{k/2}}  \int_{\tilde\ccalB}   e^{\frac{-\|v \|^2}{4t}}   (\tilde \lambda^{1/2}(0) +\frac{1}{2}\tilde\lambda^{-1/2}(0) v^\intercal \nabla \tilde \lambda (v) + k_1 \| v\|^2 )\\
		& \quad ( v^\intercal \nabla\tilde f_\theta (0) + \frac{1}{2}v^t \bbH v + k_2 \|v\|^3  )    dv \\
		&+  \frac{1}{t} \frac{1}{\sqrt{\tilde\lambda(0)}}\frac{1}{(4\pi t)^{k/2}}  \int_{\tilde\ccalB}   e^{\frac{-\|v \|^2}{4t}}  (\tilde f(0)-\tilde f(v)) k_0 \| v\|^2 dv 
	\end{align}
	Which can be split into $5$ integrals as follows:
	\begin{align}
		A_t &= \frac{1}{t} \frac{1}{\sqrt{\tilde\lambda(0)}}\frac{1}{(4\pi t)^{k/2}}  \int_{\tilde\ccalB}   e^{-\frac{\|v \|^2}{4t}}   \tilde \lambda^{1/2}(0)   v^\intercal \nabla\tilde f_\theta (0)  dv \\
		B_t &=  \frac{1}{t} \frac{1}{\sqrt{\tilde\lambda(0)}}\frac{1}{(4\pi t)^{k/2}} \int_{\tilde\ccalB}   e^{-\frac{\|v \|^2}{4t}}   \tilde \lambda^{1/2}(0)    \frac{1}{2}v^t \bbH v dv \\
		C_t &=  \frac{1}{t} \frac{1}{\sqrt{\tilde\lambda(0)}}\frac{1}{(4\pi t)^{k/2}} \int_{\tilde\ccalB}   e^{-\frac{\|v \|^2}{4t}}   \frac{1}{2}\tilde\lambda^{-1/2}(0) v^\intercal \nabla \tilde \lambda (0) \nabla\tilde f_\theta (0) ^\intercal v dv \\
		D_t &= \frac{1}{t} \frac{1}{\sqrt{\tilde\lambda(0)}}\frac{1}{(4\pi t)^{k/2}}  \int_{\tilde\ccalB}   e^{-\frac{\|v \|^2}{4t}}   \frac{1}{2}\tilde\lambda^{-1/2}(0) v^\intercal \nabla \tilde \lambda (0)   \frac{1}{2}v^\intercal \bbH v   dv \\
		E_t &= \frac{1}{t} \frac{1}{\sqrt{\tilde\lambda(0)}}\frac{1}{(4\pi t)^{k/2}}  \int_{\tilde\ccalB}   e^{\frac{-\|v \|^2}{4t}}  (\tilde f(0)-\tilde f(v)) k_0 \| v\|^2 d \reals^k\\
		&\leq \frac{1}{t} \frac{1}{\sqrt{\tilde\lambda(0)}}\frac{1}{(4\pi t)^{k/2}}  \int_{\tilde\ccalB}   e^{\frac{-\|v \|^2}{4t}}  k_0 \| v\|^3 d \reals^k
	\end{align}
	Where for $E_t$ we have used the fact that the first derivative has a uniform bound. Note that given the Gaussian kernel integration, the following conditions are true,
	\begin{align}
		\int_{\tilde B} v_i e^{-\frac{\|v \|^2}{4t}}  dv = 0 \\
		\int_{\tilde B} v_i v_j e^{-\frac{\|v \|^2}{4t}}  dv = 0 \\
		\int_{\tilde B} v_i v_j^2 e^{-\frac{\|v \|^2}{4t}}  dv = 0 \\
		\frac{1}{t}\frac{1}{(4\pi t)^{k/2}}\int_{\tilde B} \| v\|^3 e^{-\frac{\|v \|^2}{4t}}  dv = \ccalO(t^{1/2})\\
		\frac{1}{t}\frac{1}{(4\pi t)^{k/2}}	\int_{\tilde B} v_i^2 e^{-\frac{\|v \|^2}{4t}}  dv = 2
	\end{align}
	for $i\neq j$ given the zero mean gaussian distribution with zero mean and diagonal covariance matrix. Using these conditions, we can conclude that $A_t = D_t = 0 $. From $B_t$ and $C_t$, the only non-zero elements are given by the diagonal elements, therefore, we obtain,
	\begin{align}
		A_t+B_t+C_t+D_t+E_t  &= - \sum_{i=1}^k \frac{\partial^2 \tilde f (0)}{\partial x_i^2} -   \frac{1}{\tilde \lambda (0)}\sum_{i=1}^k [\nabla\tilde \lambda(0) ]_i [\nabla\tilde f_\theta(0) ]_i + \ccalO(t^{1/2} )\\
		& = - \Delta f - \frac{1}{\lambda} \langle \nabla \lambda, \nabla f_\theta \rangle _{T_p} + \ccalO(t^{1/2} )
	\end{align}
	Where we have used that $\Delta_\ccalM f  (p)= \Delta_{\reals^k} \tilde f (0) = - \sum_{i=1}^k \frac{\partial^2 \tilde f}{\partial x_i^2} (0)$ (see Chapter $3$ of \cite{rosenberg_1997}).
}

\myproof[of Proposition \ref{P:pointcloud}]{
	To begin with, utilizing Green's identity, for the laplacian $\Delta_\lambda = \Delta f_\theta + \frac{1}{\lambda} \langle \nabla \lambda,\nabla f_\theta  \rangle$, the following holds 
	\begin{align}
		\int_\ccalM f_\theta (z)\Delta_\lambda f_\theta(z) \lambda(z)dV(z) = \int_\ccalM \|\nabla_\ccalM f_\theta(z) \|^2 \lambda(z)dV(z)
	\end{align}

	An equivalent statement of the proof is that for any $\epsilon >0$, and $\delta>0$, there exist $N_0$ such that, for $N\geq N_0$
	\begin{align}\label{eqn:theo_pointcloud_first_equation}
		P\bigg(\sup_{\lambda\in\Lambda,\theta \in \Theta}\bigg| \frac{1}{N}\sum_{i=1}^N f_{\theta}(x_i)\lambda(x_i) \bbL_{N}^{t_N}f_{\theta}(x_i)-  \int_\ccalM f_{\theta}(z)(- \Delta_1 f_{\theta}(z) )\lambda(x) p(z)dV(z)\bigg|>\epsilon\bigg)\leq \delta
	\end{align}
	By considering the complement of the event in \eqref{eqn:theo_pointcloud_first_equation}, we obtain, 
	\begin{align}\label{eqn:theo_pointcloud_first_equation_change_inequality}
		P\bigg(\sup_{\lambda\in\Lambda,\theta \in \Theta}\bigg| \frac{1}{N}\sum_{i=1}^N f_{\theta}(x_i)\lambda(x_i) \bbL_{N}^{t_N}f_{\theta}(x_i)-  \int_\ccalM f_{\theta}(z)(- \Delta_1 f_{\theta}(z) )\lambda(x) p(z)dV(z)\bigg|\leq \epsilon\bigg)\geq 1- \delta
	\end{align}
	We can now add and subtract $ \int_\ccalM f_{\theta}(z)\bbL_{N}^{t_{N}}(z)\lambda(z) p(z) dV(z)$, 
	
	\begin{align}\label{eqn:theo_pointcloud_add_subtract}
		P\bigg(\sup_{\lambda\in\Lambda,\theta \in \Theta}\bigg|& \frac{1}{n}\sum_{i=1}^n f_{\theta}(x_i)\lambda(x_i) \bbL_{N}^{t_{N}}f_{\theta}(x_i)- \int_\ccalM f_{\theta}(z)\bbL_{N}^{t_{N}}(z)\lambda(z) p(z) dV(z)\nonumber\\
		&+ \int_\ccalM f_{\theta}(z)\bbL_{N}^{t_{N}}(z)\lambda(z) p(z) dV(z)-\int_\ccalM f_{\theta}(z)(- \Delta_1 f_{\theta}(z) )\lambda(x) p(z)dV(z)\bigg|\leq \epsilon\bigg)\nonumber\\&\geq 1- \delta
	\end{align}
	By rearranging we obtain, 
	\begin{align}\label{eqn:theo_pointcloud_rearrange}
		P\bigg(\sup_{\lambda\in\Lambda,\theta \in \Theta}\bigg|& \frac{1}{N}\sum_{i=1}^n  \bbL_{N}^{t_N} f_{\theta}(x_i) f_{\theta}(x_i)\lambda(x_i) - \int_\ccalM f_{\theta}(z)\bbL_{N}^{t_{N}}(z)\lambda(z) p(z) dV(z)\nonumber\\
		&+ \int_\ccalM f_{\theta}(z)\bigg(\bbL_{N}^{t_N}(z)-(- \Delta_1 f_{\theta}(z) )\bigg)\lambda(x) p(z)dV(z) \bigg|\leq \epsilon\bigg)\nonumber\\&\geq 1- \delta
	\end{align}
	
	We denote $C_{N}$ the event that, 
	\begin{align}
		C_{N}=\bigg\{ & |\bbL_{ N}^{t_{ N}}f_{\theta}(x_i)|| \bbL_{N}^{t_{N}}(z)-(- \Delta_1 f_{\theta}(z) )f_{\theta}(z) |\leq \frac{\epsilon}{2} \nonumber\\
		&\cap |\bbL_{N}^{t_{N}}(z)|\bigg|\frac{1}{N}\sum_{i=1}^N f_{\theta}(x_i) \lambda(x_i) - \int_\ccalM f_{\theta}(z)\lambda(z) p(z)dV(z) \bigg|\leq \frac{\epsilon}{2} \bigg\}
	\end{align}
	Given that we are imposing fixed conditions, $C_{N}$ is included in the event described in \eqref{eqn:theo_pointcloud_first_equation_change_inequality}, 
	\begin{align}
		C_{N}\subseteq\bigg \{\bigg| \frac{1}{N}\sum_{i=1}^N f_{\theta}(x_i)\lambda(x_i) \bbL_{N}^{t_N}f_{\theta}(x_i)-  \int_\ccalM f_{\theta}(z)(-\Delta_1f_{\theta}(z))\lambda(x) p(z)dV(z)\bigg|\leq \epsilon\bigg\}.
	\end{align}
	Therefore, the probability satisfies, 
	\begin{align}
		&P\bigg( \bigg| \frac{1}{N}\sum_{i=1}^N f_{\theta}(x_i)\lambda(x_i) \bbL_{N}^{t_N}f_{\theta}(x_i)-  \int_\ccalM f_{\theta}(z)\Delta_1f_{\theta}(z)\lambda(x) p(z)dV(z)\bigg|\leq \epsilon\bigg) \geq P(C_{N})
	\end{align}
	Now we will use the fact that the manifold is compact, and that $f_{\theta}$ is continuous, and denote $\max_{z\in \ccalM}|f_{\theta}(z)|=F$, we can obtain the following bounds:
	
	\begin{align}
		|\bbL_{N}^{t_N}f_{\theta}(x_i)|&\leq  2 F\\
		\bigg|\int_\ccalM f_{\theta}(z)\lambda(x) p(z)dV(z)\bigg| &\leq F.
	\end{align}
	Using the same logic, and considering the event $B_{N}$ as,
	\begin{align}
		B_{N}=\bigg\{ & | \bbL_{N}^{t_{N}}(z)-(- \Delta_1 f_{\theta}(z) )f_{\theta}(z) |\leq \frac{\epsilon}{2F} \nonumber\\
		&\cap |\bbL_{N}^{t_{N}}(z)|\bigg|\frac{1}{n}\sum_{i=1}^n f_{\theta}(x_i) \lambda(x_i) - \int_\ccalM f_{\theta}(z)\lambda(z) p(z)dV(z) \bigg|\leq \frac{\epsilon}{4 F} \bigg\}
	\end{align}
	As $B_{N}\subset C_{N}$, it follows that $P(C_{N})\geq P(B_{N})$. Now, we can use Hoeffding's inequality and consider $N$ to be the number of samples such that,
	\begin{align}\label{eqn:theo_pointclous_hoefdding}
		P\bigg(\bigg|\frac{1}{N}\sum_{i=1}^N f_{\theta}(x_i)\lambda(x_i) - \int_\ccalM f_{\theta}(z)\lambda(z) p(z)dV(z) \bigg|\geq \frac{\epsilon}{4 F}\bigg) \leq \frac{\delta}{2}.
	\end{align}
	where $N$ should be such that $\frac{\delta}{2} \geq 2 e^{-\frac{\epsilon^2 N}{ (8 b F^2)^2}}$, that is
	\begin{align}
	N\geq \frac{64b^2F^4}{\epsilon^2}\log(\delta/4).\label{prop3:conditionN1}
	\end{align}
	Now, we should evaluate the Laplacian term. To do so, we add and substract the continuous $\tilde \bbL_\lambda^t f_\theta(z)$ (cf. definition \ref{def:continuous_heat_kernel_laplacian}) as follows,
	
	\begin{align}\label{eqn:theo_pointclous_laplacian_n_prime}
		P\bigg( | \bbL_{N}^{t_{N}}(z)+\tilde \bbL_\lambda^t f_\theta(z)-\tilde \bbL_\lambda^t f_\theta(z)-(- \Delta_1 f_{\theta}(z) ) |\leq \frac{\epsilon}{2F} \bigg)\geq 1- \frac{\delta}{2}
	\end{align}
	Utilizing the same argument as before, with $\epsilon/4$, and $\delta/4$ we get, 
		\begin{align}
		P\bigg( | \bbL_{N}^{t_{N}}(z)-\tilde \bbL_\lambda^t f_\theta(z)| \frac{\epsilon}{4F} \bigg)\leq  \frac{\delta}{4}\label{eqn_Prop3:bounded_diffence_between_laplacians}
\end{align}
and,
\begin{align}
		P\bigg( |\tilde \bbL_\lambda^t f_\theta(z)-(- \Delta_1 f_{\theta}(z) ) |\geq \frac{\epsilon}{4F} \bigg)\leq  \frac{\delta}{4} \label{eqn_Prop3:convergence_euclidean_to_laplacian}
\end{align}
Now, \eqref{eqn_Prop3:bounded_diffence_between_laplacians} can be bounded utilizing Lemma \ref{lemma:bounded_diffence_between_laplacians}. That is, $N$ needs to verify 
\begin{align}
     \delta/4\geq2 e^{\frac{-\epsilon^2 N }{8 \bbB^2_t}}+ 2 e^{  -\frac{\epsilon^2 (N-1)}{4\bbB_t a^{3/2}}  } \label{prop3:conditionN2}
\end{align}
Note that $\frac{N}{B_t^2}$ needs to tend to $\infty$, as $N\to\infty$, and $t\to 0$. This is analogous to $N t^{2+d} \to \infty$, therefore, $t=N^{-\frac{1}{2+d+\alpha}}$, for any $\alpha>0$ satisfies the condition. With this choice of $t$, all the integral approximations hold by \cite{belkin2005towards}[Lemma 7.1].

For term \eqref{eqn_Prop3:convergence_euclidean_to_laplacian} can be bounded utilizing Lemma \ref{lemma:convergence_euclidean_to_laplacian}, and therefore $N$ needs to verify,
\begin{align}
    t^{1/2} \leq \frac{4}{\delta},\\
    N^{\frac{-1}{2(2+d+\alpha)}}\leq \frac{4}{\delta} \label{prop3:conditionN3}
\end{align}

	We have therefore obtained the number of samples $N$ required to a lower bound each probability by $1-\delta/2$. Therefore, the probability of $B_{N}$ verifies that,
	\begin{align}
		&P\bigg(\sup_{\lambda\in\Lambda,\theta \in \Theta} \bigg| \frac{1}{n}\sum_{i=1}^N f_{\theta}(x_i)\lambda(x_i) \bbL_{N}^{t_N}f_{\theta}(x_i)-  \int_\ccalM f_{\theta}(z)\Delta_1f_{\theta}(z)\lambda(x) p(z)dV(z)\bigg|\leq \epsilon\bigg) \nonumber\\
		&\geq P(C_{N}) \geq P(B_{N}) = 1-P(B^c_{N}) \\
		&\geq 1-P\bigg( | \bbL_{N}^{t_N}(z)-(-\Delta_1f_{\theta}(z)) |\geq \frac{\epsilon}{2\bbPhi} \bigg)\nonumber\\
		&-P\bigg(\bigg|\frac{1}{N}\sum_{i=1}^N f_{\theta}(x_i) \lambda(x_i) - \int_\ccalM f_{\theta}(z)\lambda(z) p(z)dV(z) \bigg|\geq \frac{\epsilon}{4 \bbPhi}\bigg)\\
		&\geq 1-\delta/2-\delta/2 =1-\delta 
	\end{align}
	Where we first used the fact that the complement of an intersection is the union of the complements. Then, we used the take the maximum over $N$ for the events, so as to verify \eqref{eqn:theo_pointclous_hoefdding} and \eqref{eqn:theo_pointclous_laplacian_n_prime} respectively. Therefore, by taking the complement, we obtain that for any $\epsilon>0$, and any $\delta>0$, there exists a number of samples $N$  \eqref{eqn:theo_pointclous_hoefdding} and \eqref{eqn:theo_pointclous_laplacian_n_prime}, such,
	\begin{align}
		&P\bigg(\sup_{\lambda\in\Lambda,\theta \in \Theta}\bigg| \frac{1}{n}\sum_{i=1}^n f_{\theta}(x_i)\lambda(x_i) \bbL_{n^{'}}^{t_n^{'}}f_{\theta}(x_i)-  \int_\ccalM f_{\theta}(z)(-\Delta_1f_{\theta}(z))\lambda(x) p(z)dV(z)\bigg|\geq \epsilon\bigg)\leq\delta 
	\end{align}
	it suffices to pick the total number of samples $N_0$, such that it verifies all three conditions \eqref{prop3:conditionN1}, \eqref{prop3:conditionN2}, and \eqref{prop3:conditionN3}. Moreover, $t$ needs to go to zero as $N\to\infty$, that is why $t=N^{-\frac{1}{d+2+\alpha}}$ for any $\alpha>0$.}

\newpage
\section{Dual Ascent Algorithm} \label{appendix:algorithms}
In this section, we explore the two algorithms that we present to solve problem \ref{prob:manifold_dual}; the gradient based primal dual Algorithm \ref{alg:gradient_based_smooth_learning} in subsection \ref{appendix_subsec:grad_primal_dual}, and the pointcloud laplacian based primal dual \ref{alg:smooth_learning} in subsection \ref{appendix_subsec:pointcloud_laplacian}. The laplacian variant, presents a computational advantage, but requires conditions on the dual variables $\lambda$ that might be difficult to secure in practice. The gradient based method on the other hand, presents a less restrictive algorithm at the expense of a higher computational cost. It is important to note that the sole difference between the two relies on the gradient of the lagrangian with respect to $\theta$, that is to say, the update on the dual variables $\lambda$ and $\mu$ remains the same. In this appendix section, we provide an verbose explanation of the two procedures. Moreover, we elaborate on the estimator of the norm of the gradient \ref{appendix_subsec:norm_grad}.  

\subsection{Gradient Based Primal Dual Ascent}\label{appendix_subsec:grad_primal_dual}
On this subsection we elaborate on the gradient based method to update the primal variable $\theta$. Upon showing that under certain conditions problem \eqref{prob:manifold_lipschitz} and problem \eqref{prob:emp_dual} are close (see Proposition \ref{P:statistical_empirical_bound}, and \ref{P:solution_bound}), we will introduce a dual ascent algorithm to solve the later. The problem that we seek to solve is given by,
\begin{align}
 \max_{\elambda,\emu}&\ \hat d_G(\hat\lambda, \hat\mu):=\underset{\theta}{\min} \quad  \dfrac{1}{N} \sum_{n = 1}^N \bigg( \ell \big( f_\theta(x_n),y_n \big) -\mu\epsilon\bigg)
	+ \hat{\mu}^{-1}\frac{1}{N} \sum_{n=1}^N \elambda(x_n)\|\nabla_\ccalM f_\theta(x_n) \|\\
\text{subject to}&\quad\frac{1}{N}\sum_{n=1}^N \lambda(x_n)=1
\end{align}
We will procede with an iterative process that minimizes over $\theta$, and maximizes over $\lambda$ and $\mu$. In short, for each value of $\mu$, $\lambda$, we seek to minimize the Lagrangian $\hat L(\theta,\emu,\elambda)$ by taking gradient steps as follows, 
\begin{align}
\theta_{k+1} &= \theta_k + \eta_\theta \nabla_\theta \hat L(\theta,\emu,\lambda)\\
&= \theta_k + \eta_\theta \nabla_\theta \bigg(\emu \frac{1}{N} \sum_{n = 1}^N \ell \big( f_\theta(x_n),y_n \big)+\frac{1}{N}\sum_{n=1}^N \elambda(x_n) \|\nabla_\ccalM f_\theta(x_n)\|^2\bigg)
\end{align}
The gradient base approach computes gradients with respect to the loss function $\ell$, and with respect to the norm of the gradient. After the dual function $\hat {d}(\hat\lambda, \hat\mu)$ is minimized, in order to solve \ref{prob:manifold_dual}, we require to maximize the dual function $\hat d(\hat\lambda, \hat\mu)$ over both $\hat\lambda$, and $\hat\mu$, which can be done by evaluating the constraint violation as follows, 
\begin{align}
	\lambda_n&\leftarrow\lambda_n+\eta_\lambda\partial_{\hat\lambda_n}\hat\ccalL(\theta,\hat\mu,\hat\lambda) \text{, with }  \partial_{\hat\lambda_n}\hat\ccalL(\theta,\hat\mu,\hat\lambda)= \| \nabla_\ccalM \phi (x_n) \|^2 \\
	\hat\mu&\leftarrow\hat\mu+\eta_{\hat\mu}\hat\ccalL(\theta,\hat\mu,\hat\lambda)\text{, with }\partial_{\hat\mu} \hat\ccalL(\theta,\hat\mu,\hat\lambda) = \dfrac{1}{N} \sum_{n = 1}^N \ell \big( f_\theta(x_n),y_n\big) - \epsilon ,
\end{align}
where stepsizes $\eta_\lambda,\eta_\mu$ are positive numbers. Note that to evaluate the norm of the gradient on a particular point  $\| \nabla_\ccalM \phi (x_n) \|$ we can look at its neighboring points $\ccalN(x_i)$, and estimate its norm. After updating the dual variables $\lambda$ using gradient ascent, we require $|\lambda|_1=1$, which can be done by either normalizing, or projecting to the simplex \cite{wang2013projection}. 
How to compute the norm of the gradient will be explained in the sequel in subsection \ref{appendix_subsec:norm_grad}. The overall procedure is explained in Algorithm \ref{alg:gradient_based_smooth_learning}.

\begin{algorithm}[tb]
	\caption{Gradient Based Smooth Learning on Data Manifold}
	\label{alg:gradient_based_smooth_learning}
	\begin{algorithmic}[1]
		\STATE   Initialize parametric function $\theta$, define neighborhoods $\ccalN(p)$ for every $p\in \ccalD$
		\REPEAT
		\FOR {primal steps $k$}
		\STATE Estimate maximum norm of gradient for all $x_i$: $\|\nabla_\ccalM  f_\theta(x_i)\|=\max_{z\in\ccalN(x_i)} \frac{\| f_\theta(x_i)-f_\theta(z)\|}{d(x_i,z)})$
		\STATE Update $\theta: \theta\leftarrow\theta - \eta_{\theta}\nabla_\theta\left(\mu  \frac{1}{N}\sum_{i=1}^N\ell(f_\theta(x_i),y_i) + \frac{1}{N}\sum_{i=1}^N  \lambda(x_i)\|\nabla_\ccalM  f_\theta(x_i)\|^2 \right)$
		\ENDFOR
		\STATE Update dual variable $\mu$: $\mu \leftarrow [\mu+\eta_\mu \frac{1}{N}\sum_{i=1}^N\ell(f_\theta(x_i),y_i)-\epsilon) ]_+$ 
		\STATE Update dual variable $\lambda(x_i)$: $\lambda(x_i)\leftarrow [\lambda(x_i)+\eta_\lambda (\max_{z\in\ccalN(x_i)} \frac{\| f_\theta(x_i)-f_\theta(z)\|^2}{d(x_i,z)^2})]_+$
		\STATE Project $\lambda:\lambda= \arg\min_{0\preccurlyeq\tilde\lambda} \|\tilde \lambda-\lambda\|$ s.t. $|\tilde\lambda|_1=N$ 
        \STATE $e= e+1$
		\UNTIL  { convergence}
	\end{algorithmic}
\end{algorithm}

\subsection{point-cloud Laplacian Dual Ascent}\label{appendix_subsec:pointcloud_laplacian}

\begin{algorithm}[t]
	\caption{point-cloud Laplacian Smooth Learning on Data Manifold}
	\label{alg:smooth_learning}
	\begin{algorithmic}[1]
		\STATE   Initialize parametric function $\theta$, fix temperature $t$
		\REPEAT
		\FOR {primal steps $k$}
		\STATE Update $\theta: \theta\leftarrow\theta - \eta_{\theta}\nabla_\theta\left(\mu  \frac{1}{N}\sum_{i=1}^N\ell(f_\theta(x_i),y_i) + \frac{1}{N}\sum_{i=1}^N f_\theta(x_i) \lambda(x_i) \bbL_{\lambda p, N}^{t}f_\theta(x_i) \right)$
		\ENDFOR
		\STATE Update dual variable $\mu$: $\mu \leftarrow [\mu+\eta_\mu \frac{1}{N}\sum_{i=1}^N\ell(f_\theta(x_i),y_i)-\epsilon) ]_+$ 
		\STATE Update dual variable $\lambda(x_i)$: $\lambda(x_i)\leftarrow [\lambda(x_i)+\eta_\lambda (\max_{z\in\ccalN(x_i)} \frac{\| f_\theta(x_i)-f_\theta(z)\| ^2}{d(x_i,z)^2})]_+$
		\STATE Project $\lambda:\lambda= \arg\min_{0\preccurlyeq\tilde\lambda} \|\tilde \lambda-\lambda\|$ s.t. $|\tilde\lambda|_1=N$ 
        \STATE { $e=e+1$}
		\UNTIL{  convergence }
	\end{algorithmic}
\end{algorithm}

 Our algorithm will take advantage of the Laplacian formulation given in Proposition \ref{P:pointcloud}. That is, we will estimate the gradient of the lagrangian with respect to $\theta$ utilizing the point-cloud laplacian formulation. Formally, for a vector $\hat\lambda\in \reals_+^N$, such that $\frac{1}{N}\sum_{n=1}^N\hat\lambda_n=1$, and a constant $\hat\mu\in\reals^+$, the empirical dual function associated with the Lagrangian  $\hat{L}(\theta,\hat{\mu},\hat{\lambda}_n)$ of the empirical dual function associated with \eqref{prob:manifold_dual} is defined as,
\begin{align}\label{prob:dual_function}
	\hat d(\hat\lambda, \hat\mu)&=\underset{\theta}{\min} \quad  \dfrac{1}{N} \sum_{n = 1}^N \ell \big( f_\theta(x_n),y_n \big) 
	+ \hat{\mu}^{-1}\frac{1}{N} \sum_{n=1}^N f_\theta(x_n) \hat{\lambda}_n \bbL_{\lambda p,N}^{t_N} f_\theta(x_n),
\end{align}
Note that in \eqref{prob:dual_function}, we have omitted the term $-\hat\mu \epsilon$ as it is a constant term for a given $\hat\mu$, and we have divided over $\hat\mu$ which renders an equivalent problem as long as $\hat\mu>0$. By taking the maximum of the dual function over $\emu$, and $\elambda$, we recover the dual problem \ref{prob:emp_dual}. For a given choice of dual variables, $\hat\lambda,\hat\mu$, the dual function \eqref{prob:dual_function} is an unconstrained problem that only depends on the parameters $\theta$. Now, the link with Manifold Regularization \eqref{prob:manifold_regularization} is seen; as considering $\tilde\lambda(x_n)=1/N$, and $\tilde \mu = \gamma$, the problems become equivalent. In order to minimize \eqref{prob:dual_function}, we can update the parameters $\theta$ following the gradient,
\begin{align}\label{eqn:update_theta}
	\theta \leftarrow \theta + \eta_\theta\nabla_\theta \bigg( \dfrac{1}{N} \sum_{n = 1}^N \ell \big( f_\theta(x_n),y_n \big) 
	+ \hat{\mu}^{-1}\frac{1}{N} \sum_{n=1}^N f_\theta(x_n) \hat{\lambda}_n \bbL_{\lambda p, N}^{t_N} f_\theta(x_n) \bigg),
\end{align}
where $\eta_\theta>0$ is a step-size. Note that the updates in \ref{eqn:update_theta} have two parts, one that relies on the loss $\ell$ which utilizes labeled data, and another term given by $\hat{\lambda}_n\bbL_{\lambda p,N}^{t_N}$, that penalizes the Lipschitz constant of the function, and can be computed utilizing unlabeled data. A more succinct explanation of the algorithm is described in Algorithm \ref{alg:smooth_learning}. 
\subsection{Gradient Norm Estimate}\label{appendix_subsec:norm_grad}

For both Algorithms \ref{alg:gradient_based_smooth_learning}, and \ref{alg:smooth_learning}, we require to compute the norm of the gradient of $f_\theta$ at sample point $x_n$ with respect to the manifold $\ccalM$. Leveraging the Lipschitz constant definition \ref{def_manifold_lipschitz}, in order to estimate the maximum norm of the gradient, we require to evaluate the Lipschitz over the sample points in our dataset. To do so, we require to set a metric distance over the manifold $d_\ccalM$, with which we will define the neighborhood of sample $x_n$ as the samples that are sufficiently close,i.e.
\begin{align}\label{def:neighborhood_manifold}
    \ccalN(x_n)=\{z\in\ccalM:d_\ccalM(z,x_n)\leq \delta\}.
\end{align}
Note that our definition of neighborhood \ref{def:neighborhood_manifold}, requires a maximum distance $\delta$. In practice, we can either choose $\delta$, or set the degree of the neighborhood,i.e. choose the $k$nearest neighbors. Upon the selection of the neighborhood, the norm of the gradient is computed as follows. 
\begin{align}
\|\nabla_\ccalM  f_\theta(x_i)\|=\max_{z\in\ccalN(x_i)} \frac{\| f_\theta(x_i)-f_\theta(z)\|}{d(x_i,z)}
\end{align}
The selection of distance over the manifold is application dependent. In controls problems that involve physical systems, the metric might involve the work required from two states to reach each other. On computer vision applications, common choices of metrics can be perceptual losses \cite{zhang2018unreasonable}, or distances over embedding \cite{Khrulkov_2020_CVPR}.


\newpage
\section{Experiments}\label{Appendix:DetailsNumericals}

\subsection{Ground Robotic Vehicle Residual Learning}\label{Appendix:GroundRobot}

In this section we explain the residual learning experiment that involves a ground robot vehicle. The data acquisition of this experiment involves an \textit{iRobot} Packbot equipped with high resolution camera. The setting of the data acquisition, is a robot making turns on both \textit{pavement}, and \textit{grass}. The dynamics of the system, are govern by the discrete-time nonlinear state-space system of equations
\begin{align}
x_{k+1} =     f(x_k,u_k)+g(u_k),
\end{align}
where $x_k$ is the state of the system, and $u_k$ is the action taken, $f(x_k,u_k)$ is the model prediction, and $g(u_k)$ is a non-modellable error of the prediction. In this setting, the robot state involves the position, and the action taken is given by its linear, and angular velocities. The dynamics of the system are modeled by $f(x_k,u_k)$, and they involve the mass of the robot, the radius of the wheel, among other known parameters of the robot. 
In practice, the model $f(x_k,u_k)$ is not perfect, and the model mismatch is given by the difference in friction with the ground, delays in communications with the sensors, and discrepancies in the robot specifications. To make matters worse, some of the discrepancies, are difficult to model, or intractable to compute in practice. Therefore, we seek to learn the model mismatch $g(u_k)$. 

Figure \ref{fig:GroundSamples} shows examples of trajectories given by time series $x_k$, each of which is associated with an average error model mismatch, and its corresponding variance, 
\begin{align}
    \mu_i &= \frac{1}{K}\sum_{k=1}^{K} \| x_{k+1} - f(x_k,u_k) \|\text{,}\\
    \sigma_i &= \sqrt {\frac{1}{K-1}\sum_{k=1}^{K} \| \mu_i - (x_{k+1} - f(x_k,u_k)) \|^2}\text{.}
\end{align}
In practice, samples from the grass dataset tend to have larger disturbance mean $\mu$, given that pavement presents a more uniform setting. For example, if we compare sample \ref{subfig:grass100} to \ref{subfig:pav120}, we can see that for similar trajectories, the error is larger in the case of grass. Moreover, trajectories that involve larger magnitudes in velocities, also present larger disturbance means, and variances (cf. \ref{subfig:grass10} vs \ref{subfig:grass120}). This is due to the fact that at high speeds, the non-modeled effects such as drift, and friction are more significant. 

For more details on the data acquisition, please refer to the original paper \cite{7759118}. 
%
%
\begin{figure*}
	\centering
	\begin{subfigure}[b]{0.32\textwidth}
		\centering
		\includegraphics[width=\textwidth]{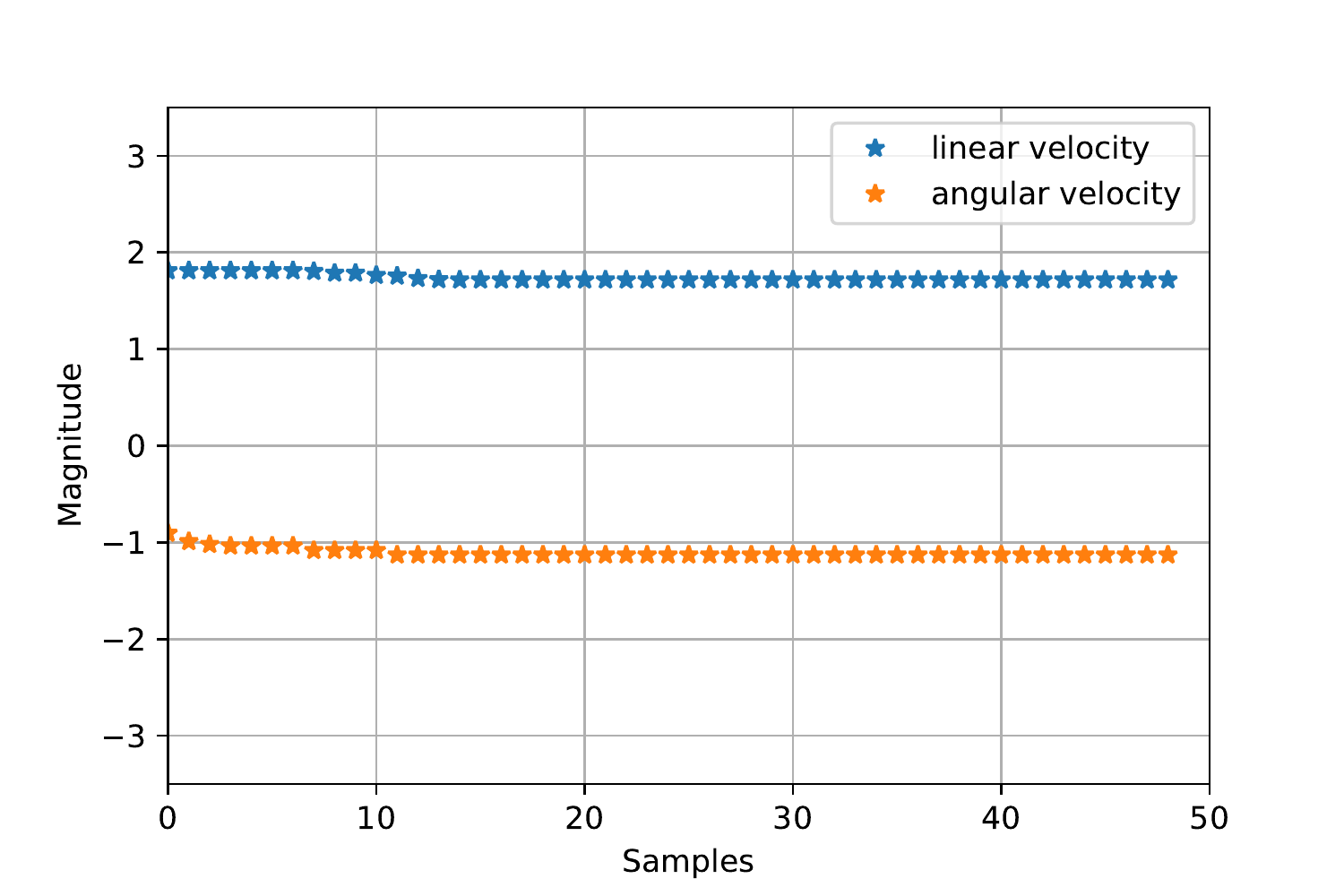}
		\caption{Grass sample $10$ with average error $\mu=0.288$ and variance error  $\sigma=0.088$. }
		\label{subfig:grass10}
	\end{subfigure}
	\hfill
	\begin{subfigure}[b]{0.32\textwidth}
		\centering
		\includegraphics[width=\textwidth]{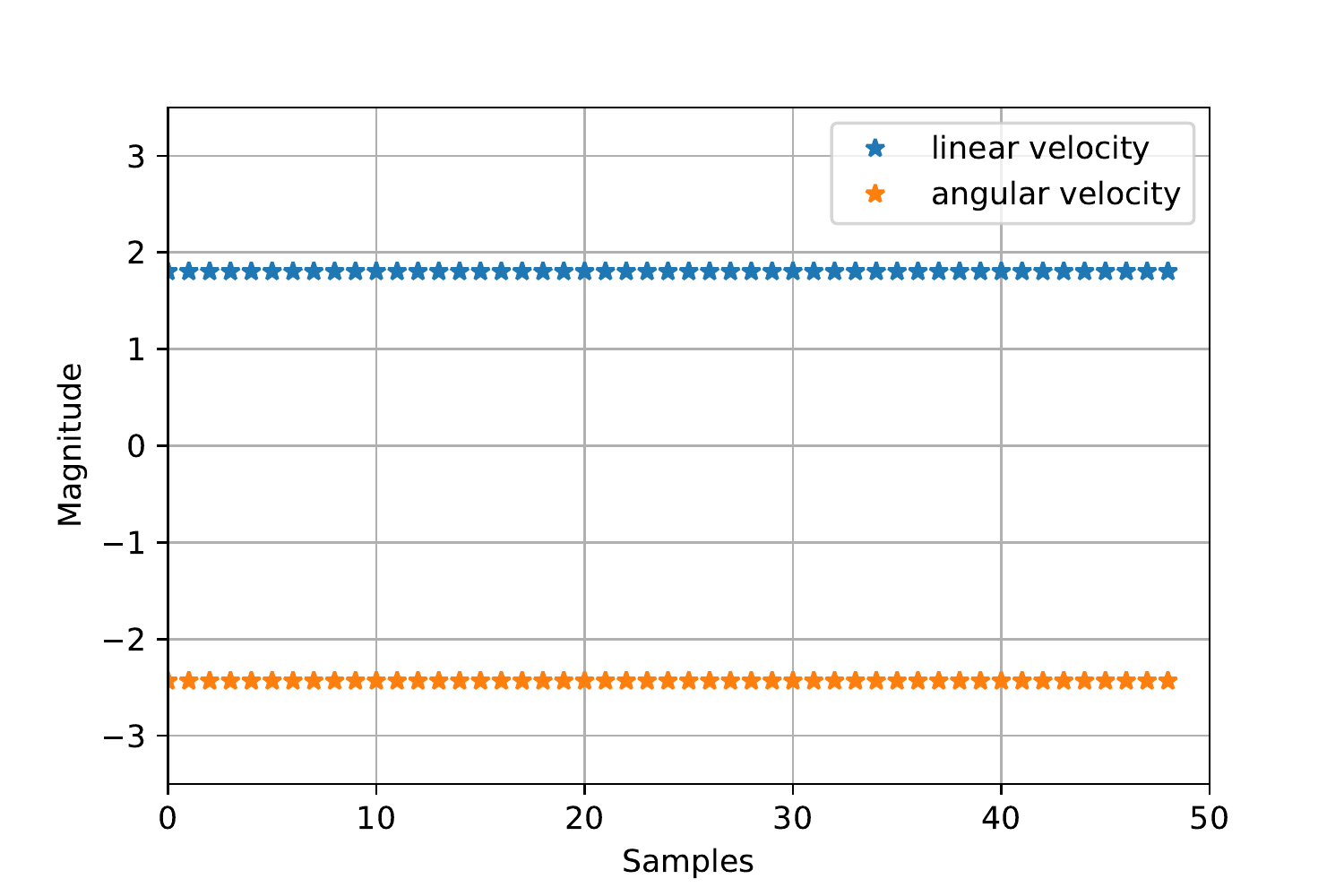}
		\caption{Grass sample $100$ with average error $\mu=0.919$ and variance error  $\sigma=0.113$. }
		\label{subfig:grass100}
	\end{subfigure}
	\hfill
	\begin{subfigure}[b]{0.32\textwidth}
		\centering
		\includegraphics[width=\textwidth]{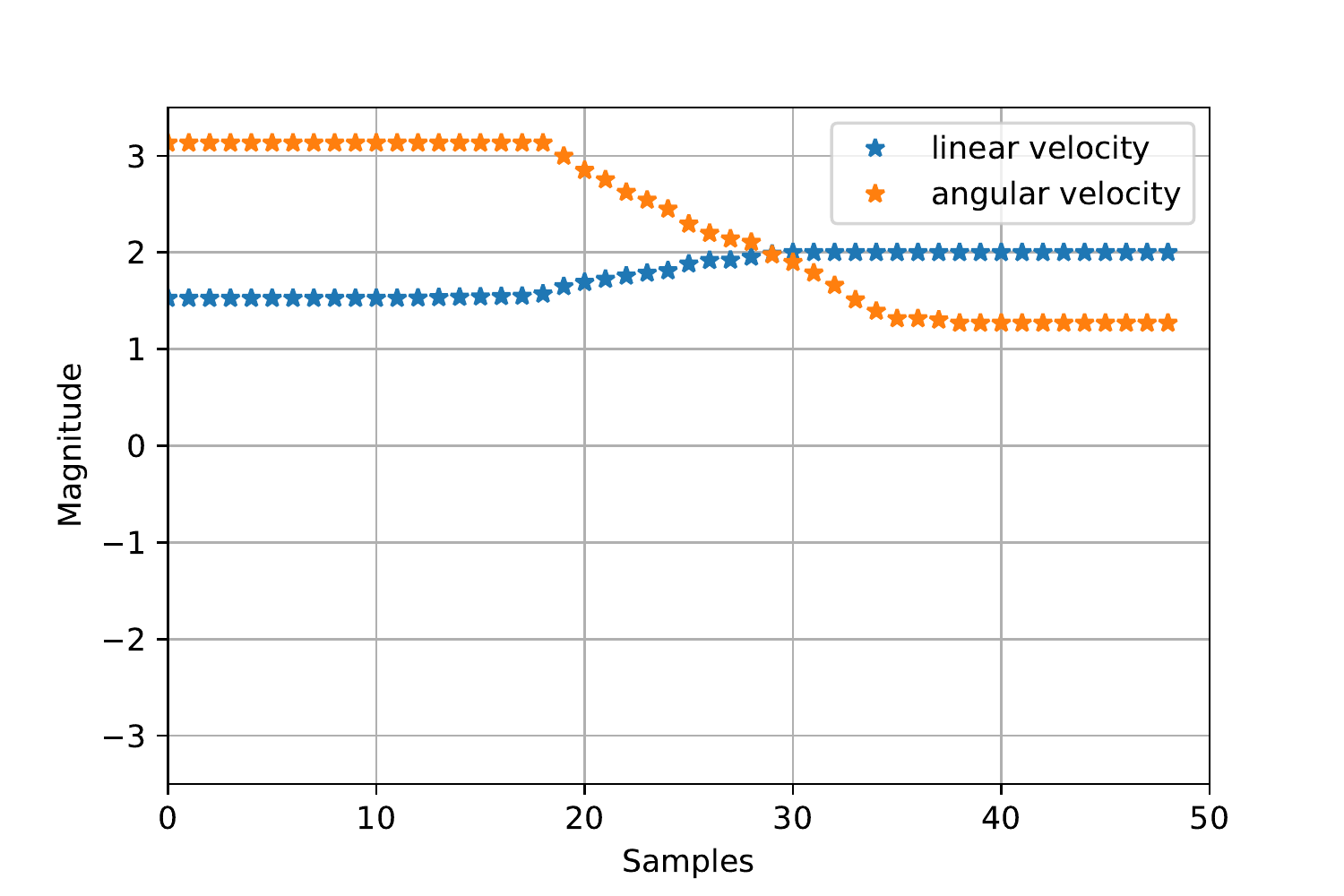}
		\caption{Grass sample $120$ with average error $\mu=1.01$ and variance error  $\sigma=0.298$.}
		\label{subfig:grass120}
	\end{subfigure}
	\hfill
	\begin{subfigure}[b]{0.32\textwidth}
		\centering
		\includegraphics[width=\textwidth]{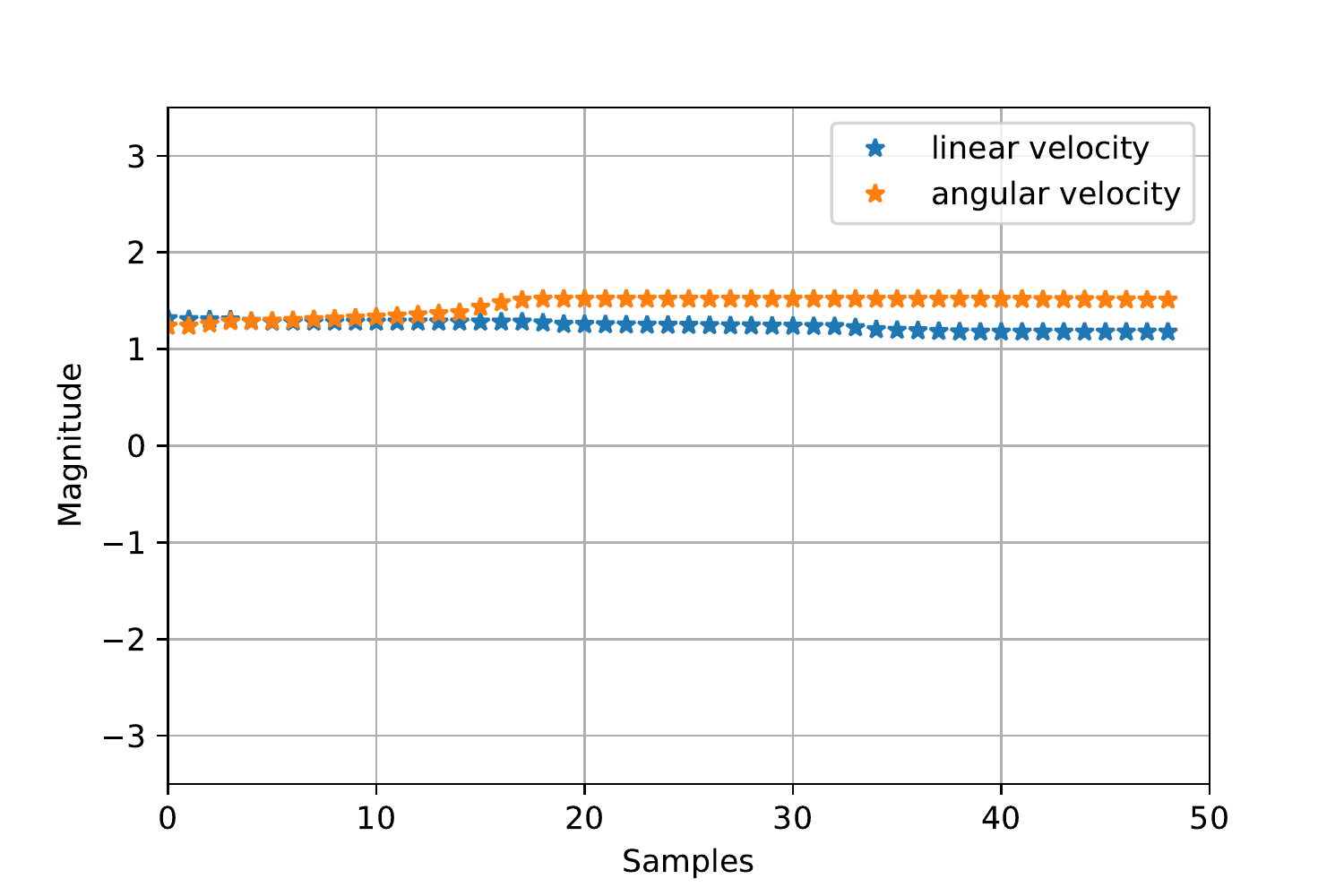}
		\caption{Pavement sample $10$ with average error $\mu=0.370$ and variance error  $\sigma=0.172$.}
		\label{subfig:pav10}
	\end{subfigure}
	\hfill
	\begin{subfigure}[b]{0.32\textwidth}
		\centering
		\includegraphics[width=\textwidth]{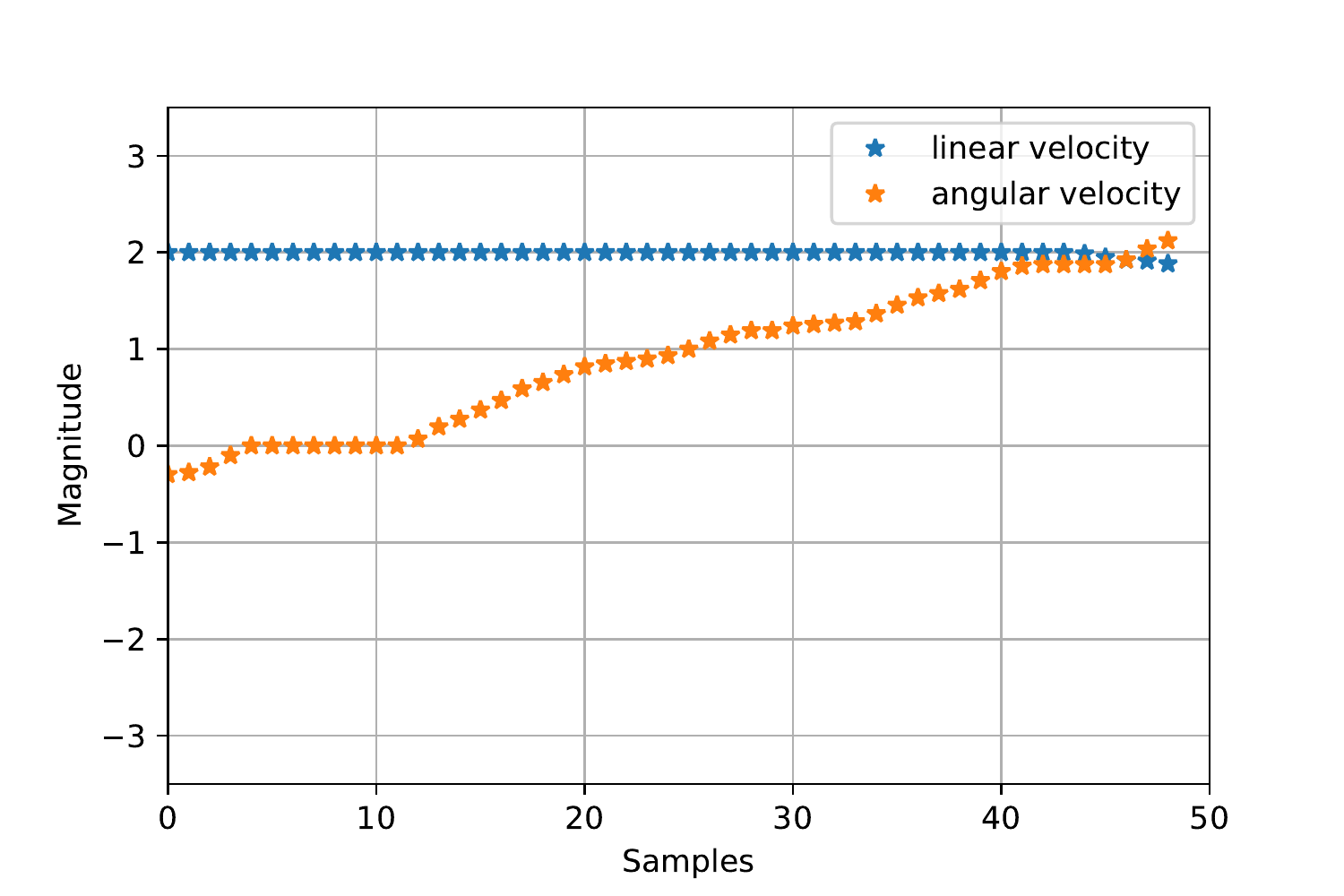}
		\caption{Pavement sample $100$ with average error $\mu=0.424$ and variance error  $\sigma=0.313$.}
		\label{subfig:pav100}
	\end{subfigure}
	\hfill
	\begin{subfigure}[b]{0.32\textwidth}
		\centering
		\includegraphics[width=\textwidth]{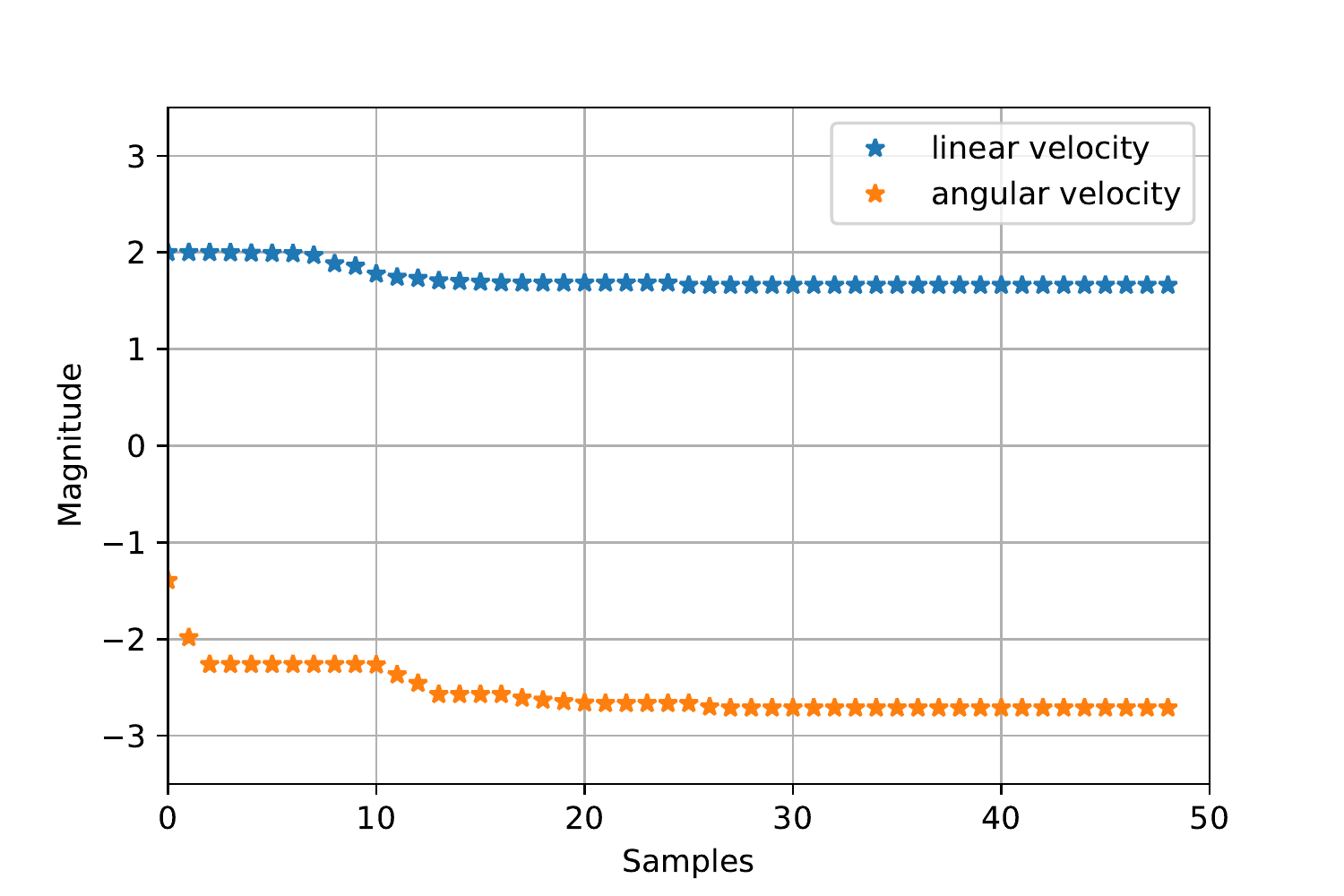}
		\caption{Pavement sample $120$ with average error $\mu=0.652$ and variance error  $\sigma=0.747$. }
		\label{subfig:pav120}
	\end{subfigure}
	\caption{Sample trajectories for grass, and pavement, with their corresponding average error, and variance error. }
	\label{fig:GroundSamples}
\end{figure*}
\subsubsection{Numerical Implementation}
We are equipped with two datasets, one for grass with $195$ samples, and one for pavement with $224$ samples. We partition the datasets into $170$ samples for training and $25$ samples for testing, and $200$ samples for training and $24$ samples for testing, for grass and pavement respectively. Each sample has associated a disturbance mean, and disturbance variance vector $[\mu,\sigma_i]\in\reals^2$. 

Regarding the optimization, we utilize the mean square error as a loss, and we train a $2$ layer fully connected neural network with $256$ hidden dimensions and hyperbolic tangent as the non-linearity. We train for $10000$ epochs, with a learning rate of $0.0015$ in the case of pavement, and $0.00015$ in the case of grass. For batch size, we utilize the whole training set. For the case of ambient Lipschitz, we utilize weight decay $w_D=0.3$, and $w_D=0.1$ for grass and pavement respectively. 

For the point-cloud laplacian, we utilize heat kernel with temperature $t=15$, and $t=35$ for grass and pavement respectively. This gives us a point-cloud laplacian with one connected component. For the Laplacian regularization we utilize $1e^{-3}$, and $1e^{-4}$ for grass and pavement respectively. 

As for the Manifold Lipschitz, we utilize $\eta_\mu = 0.01$, and $\eta_\lambda = 0.1$, and $\epsilon = 0.005$, and $0.01$ for grass and pavement respectively. We initialize $\mu=5$, and $\lambda$ uniform. 

The final results are summarized in Table \ref{table:ground_robot}.

\subsection{Quadrotor state prediction}
In this section we present a state prediction problem based on a real world collected from a quadrotor taking off and flying in a circle for $12$ seconds. The aerial robot is the open-source Crazyflie 2.1 quadrotor (\texttt{https://www.bitcraze.io/products/crazyflie-2-1/}), which has a mass of $32$ g, and a mass of size $9$ cm$^2$. The quadrotor communicates with a computer running on Intel $i7$ CPU, and the communication is established with the Crazyradio PA and at a nominal rate of $500$ Hz ($Ts=1/500$). 
To measure the position of the quadrotor  a VICON is utilized. The position is obtained from the VICON, whereas the accelerations are obtained from the on-board accelerometers and gyroscope sensors. For further information on this setup, please refer to \cite{jiahao2022online,chee2022knode,jiahao2021knowledge}. 

The experimental setup is to target a speed of $0.4$m/s and to track a circular trajectory of radius $0.5$m. We consider $2$ trajectories of $12$ seconds each (each trajectory has $6000$ time stamps), and the starting position of the quadrotor is the same for both trajectories. For each time stamp $t$, we have measurements of position, velocity and acceleration in $\reals^3$, i.e. $[x_t,y_t,z_t,\dot x_t, \dot y_t, \dot z_t, \ddot x_t, \ddot y_t, \ddot z_t]$. 
The problem consists on learning the dynamical system composed by the quadrotor. We consider the dynamics given by the equation,
\begin{align}\label{eqn:quadrotor_output_function}
    [x_{t+1},y_{t+1},z_{t+1},\dot x_{t+1},\dot y_{t+1}, \dot z_{t+1}]=f([x_t,y_t,z_t,\dot x_t, \dot y_t, \dot z_t, \ddot x_t, \ddot y_t, \ddot z_t]).
\end{align}

The learning problem consists of learning the dynamical system, i.e. we consider the dynamics given by the equation \eqref{eqn:quadrotor_output_function}.
For the learning procedure we utilize the $6000$ samples, and we seek to minimize the mean square error loss between the next state and the prediction given the current state (cf. equation \ref{eqn:quadrotor_output_function}). We train a two layer neural network with different methods as seen in Table \ref{table:quadrotor}. To test the neural network, we compute the difference between the predicted state, and the next state on the test trajectory. Note that even though the training, and testing trajectories are not the same, there is a resemblance between the two of them. 
Te begin with, we can conclude that adding regularization not always helps, as ambient regularization does not improve upon the ERM prediction method. 
A salient conclusion of the results shown in Table \ref{table:quadrotor}, is that adding regularization on the manifold space always improves upon ERM. In particular, our method is almost $3$ times better than standard ERM, and more than $2$ times better than standard Laplacian Regularization. 

To conclude with, we show that our method obtains an improvement over all the techniques considered. This allows us to conclude that in predicting the next state of a quadrotor from the current state under noisy measurements utilizing smooth functions improves generalization.

\subsubsection{Details on Numerical Implementation}

For the state prediction problem of a quadrotor we utilized a two layer fully connected neural network with $8192$ hidden units, hyperbolic tangent as the non-linearity, and bias term. For the optimizer, we utilized a learning rate of $10^{-5}$, and the full dataset per batch. We trained until convergence in all cases with number of epochs $e=1000$. For the ambient regularization we used weight decay $0.1$. For the construction of the Laplacian, we utilized temperature coefficient $t=0.01$. For Laplacian regularization we utilized $\gamma = 10^-6$. For our method, we utilized $\mu$ dual step $0.5$, $\epsilon=0.003$, and $\lambda$ dual step $0.1$. In all cases we trained until convergence with $e=10000$ epochs.

\subsection{Two-Moons Dataset}
In this subsection we provide the details of the experiment with the Two-moons data set utilized in \ref{fig:two_moons}. 

To generate the data we utilized \texttt{sklearn} library, and we utilize $1$ labeled, and $200$ unlabeled samples per class (i.e. moon), and we added noise $\sigma = \{0.05,0.1\}$. For the neural network, we utilized a two layer fully connected neural network with $64$ hidden neurons with bias term, and hyperbolic tangent as the non-linearity. For the optimizer, we utilized a learning rate of $0.9$, and no momentum. For the ambient regularization, we added a weight decay of $0.1$. For the construction of the Laplacian of Figure \ref{fig:two_moons} we utilized a heat kernel temperature of $t=0.005$, and we normalize it. For Laplacian regularization, we set $\gamma=0.5$. For Manifold Lipschitz (our method), a $\mu$ dual step of $0.5$, and a $\lambda$ dual step of $0.1$. An ablation study over different values of temperature coefficient $t$ can be found in sections \ref{sec:ablation}.
\begin{figure*}[t!]
    \centering
    \includegraphics[width=0.6\textwidth]{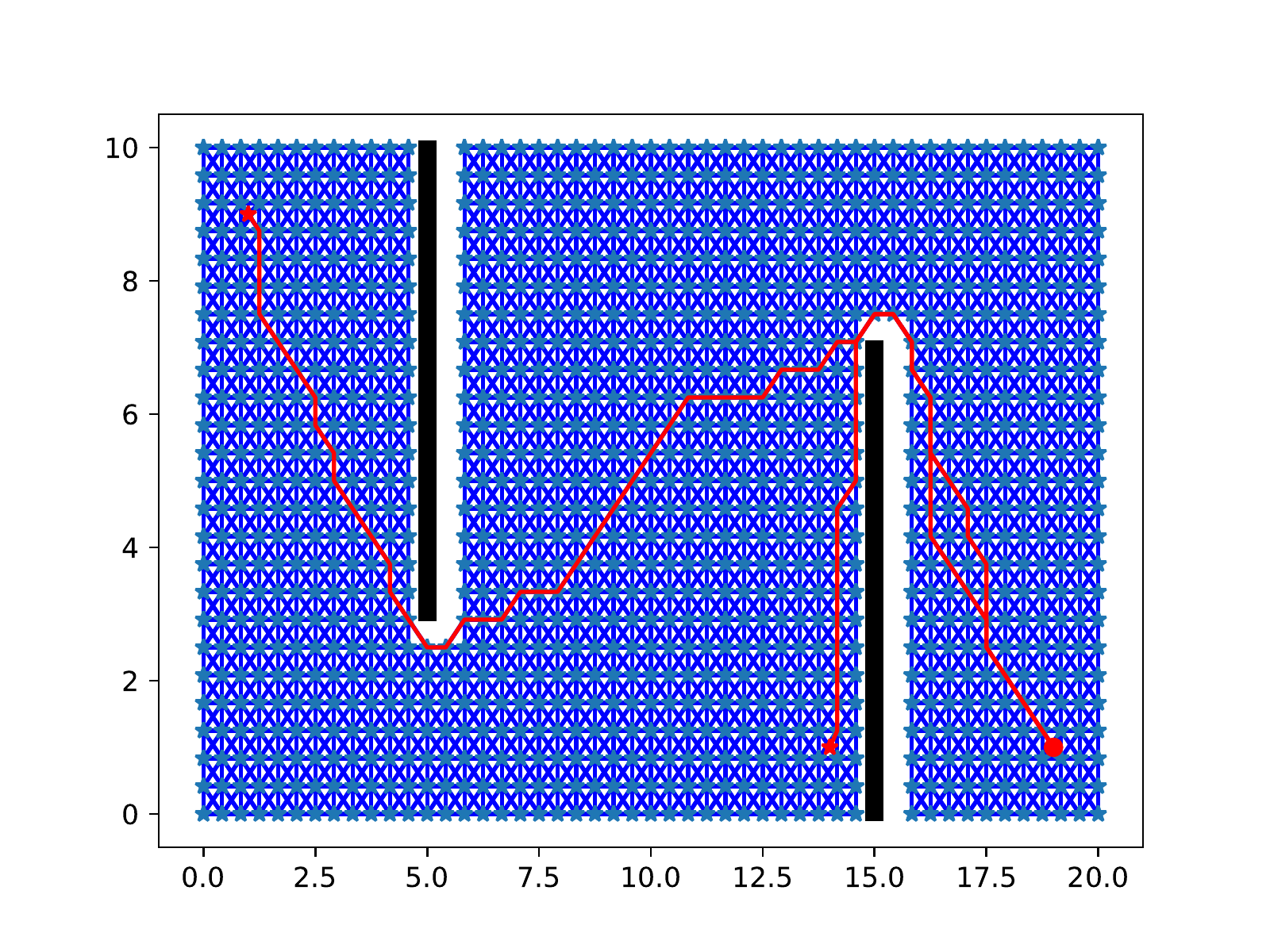}
    \caption{Point-cloud Laplacian}
    \label{fig:navigation-pointcloudLaplacian}
\end{figure*}
As seen in figure \ref{fig:two_moons}, Ambient Regularization fails to classify the unlabeled samples, given that ignores the distribution of samples given by the Manifold. The case in which the manifold has two connected component (cf. Figure \ref{subfig:dataset}), our method works as good as Manifold Regularization, due to the fact that the Lipschitz constant will be made small in both components separately. However, when the manifold is weakly connected, Manifold Regularization fails to recognize the transition between the components, as it will penalize large gradients across the manifold, converging to a plane that connects the two samples. Our Manifold Lipschitz method, as it requires the Lipschitz constant to be small, forces a sharp transition along the point with maximal separation. 

\subsection{Navigation Controls Problem}\label{Appendix:NavigationControls}

In this section, we consider the problem of continuous navigation of an agent. The agent's objective is to reach a goal while avoiding obstacles. The state space of the agent is $\ccalS = [0,20]\times[0,10]$, which represents the $x$ and $y$ axis respectively. The agent navigates by taking actions on the velocity $v\in \reals^2$, and the state evolves according to the dynamics $s_{t+1} = s_t + v_t T_s$ where $T_s=0.1s$. We construct a square grid of points in the environment that are on the free space i.e. outside of the obstacles, and utilize Dijkstra's algorithm to find the shortest path for two starting $[1,9]^T,[14,1]^T$ positions, and goal $[19,1]^T$ along the grid. For those two grid trajectories, we compute the optimal actions to be taken at each point in order to follow the trajectory.

The learner is equipped with both the labeled trajectories, as well as the unlabeled point grid. To leverage the manifold structure of the data, we consider the grid of points, and we construct the point-cloud Laplacian considering adjacent points in the grid. We train a two layer neural network using the mean square error loss over the optimal set of points and actions for ERM, ERM with ambient Lipschitz regularizer, Manifold Regularization, and our method Manifold Lipschitz method. 
\begin{wraptable}{r}{0.40\textwidth}
	\begin{tabular}{||c c||} 
		\hline
		Method &  Trajectories  \\ [0.5ex] 
		\hline\hline
		ERM & $85$ \\ 
		\hline
		Ambient Reg. & $66$ \\
		\hline
		Manifold Reg. & $77$ \\
		\hline
		Manifold Lipschitz & $94$ \\ 
		\hline
	\end{tabular}
	\caption{Number of successful trajectories from $100$ random starting points.\label{table:navigation}}
\end{wraptable}
To evaluate the performance, we randomly chose $100$ starting points and compute the trajectories generated by each learned function. A trajectory is successful if it reaches the goal without colliding with the obstacles or the walls. The results are summarized in table \ref{table:navigation}, and showcase the benefit of implementing manifold lipschitz. Our method outperforms the $3$ other methods due to the fact that it minimizes the gradient of the function over the domain of the data. As opposed to ERM, our method generates a smooth function outside of the labeled trajectory. Ambient Lipschitz regularization fails due to the fact that the euclidean distance ignores the real distance between samples across wall, forcing similar outputs for points that should take different actions. Manifold regularization is able to capture the similarity between points, but it fails to properly capture the sharp turns near the edges off the obstacles. The success of Manifold Lipschtiz, can be explained by its dual variables $\lambda$ shown in figure \ref{subfig:lambdas}. In this figure, the radius of each ball represents the value of the dual variable, which is larger close to the corners of the obstacles due to the fact that the problem requires larger gradients to make sharp turns over it. Besides, the fact that we can disentangle the loss on the labeled data, from the Lipschitz constant, allows us to overfit the data as much as we require. 
\begin{figure*}[b!]
	\centering
	\begin{subfigure}[b]{0.32\textwidth}
		\centering
		\includegraphics[width=\textwidth]{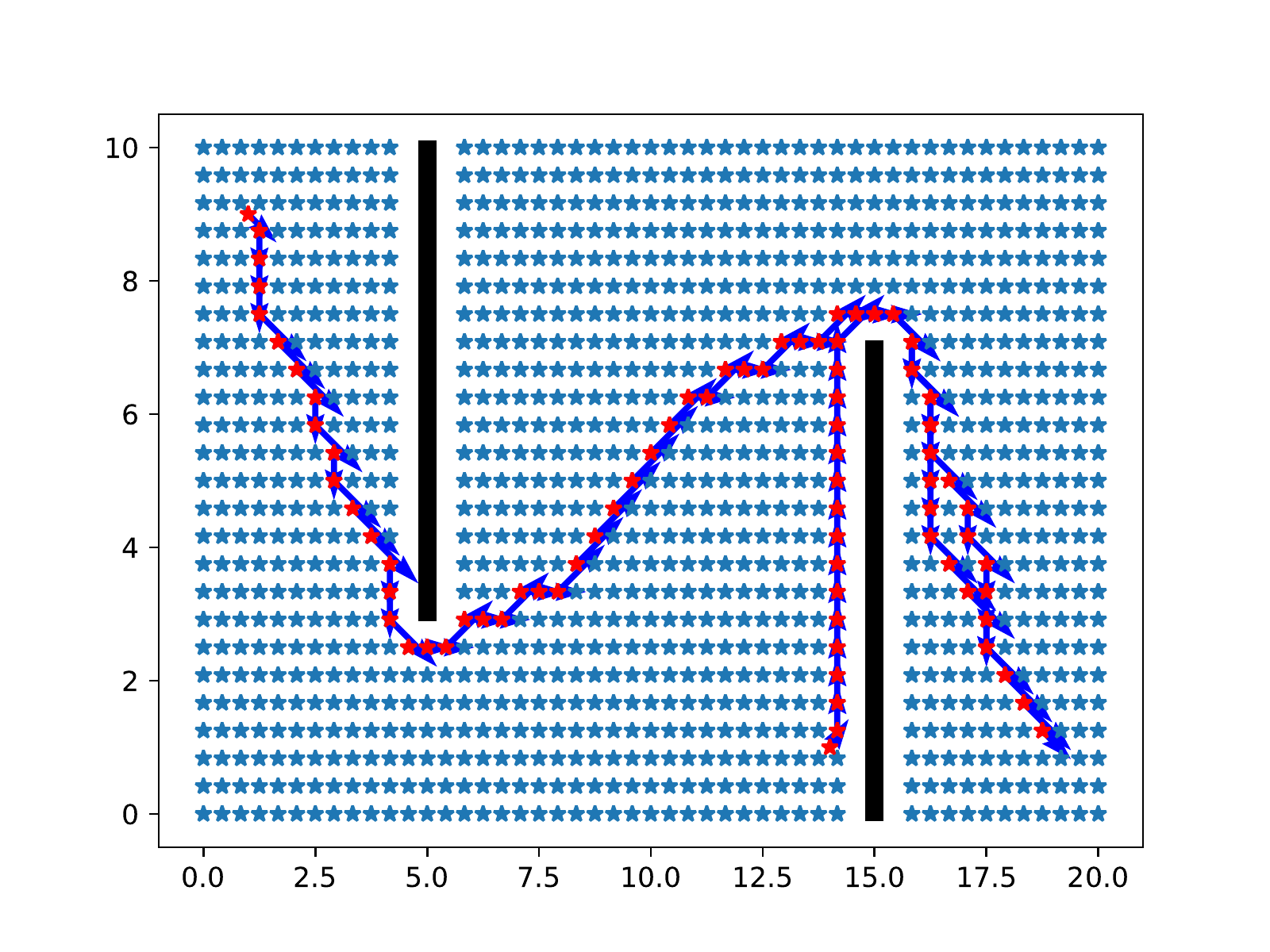}
		\caption{Dataset}
		\label{subfig:Dataset}
	\end{subfigure}
	\hfill
	\begin{subfigure}[b]{0.32\textwidth}
		\centering
		\includegraphics[width=\textwidth]{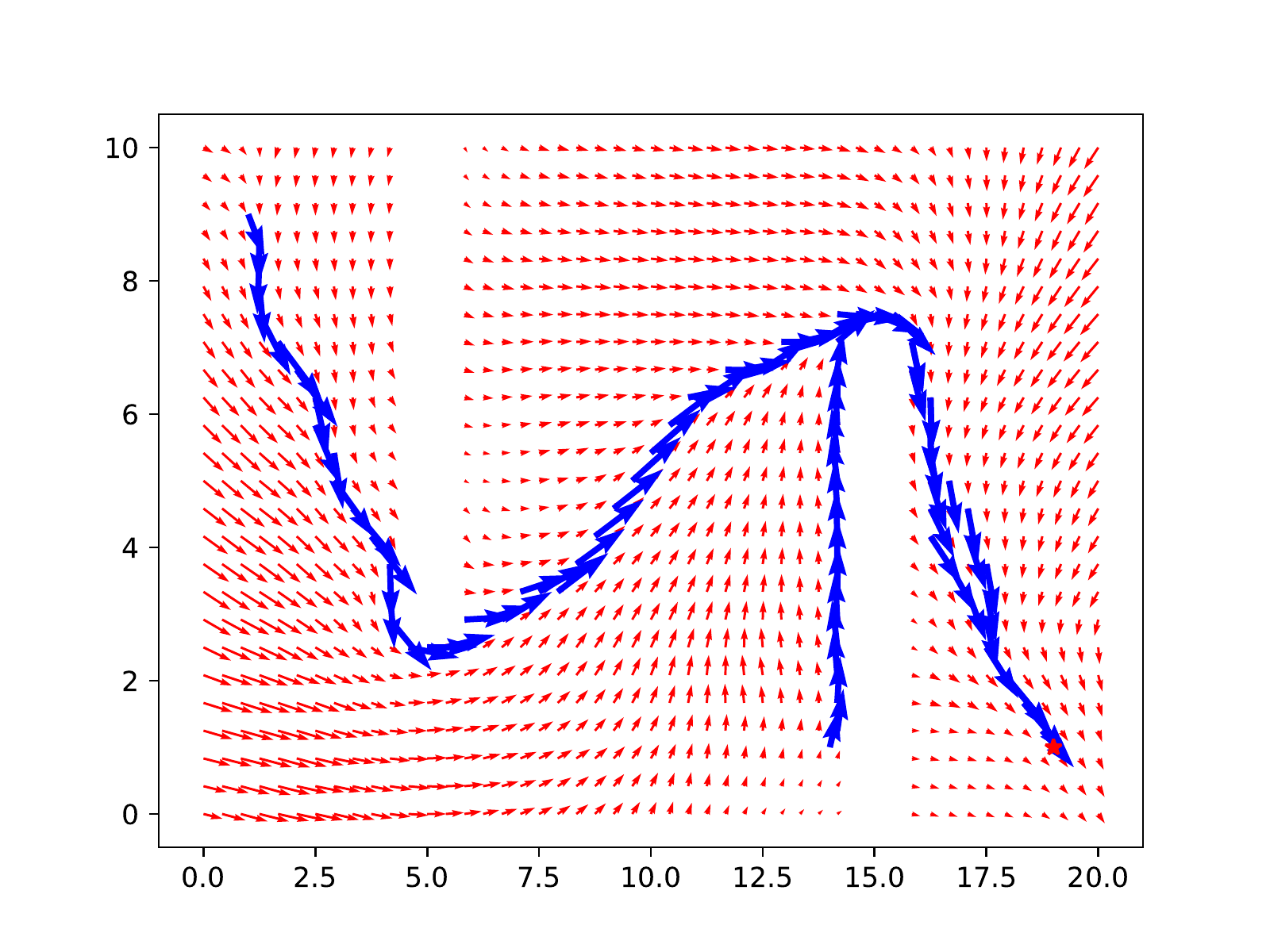}
		\caption{Manifold Lipschitz.}
		\label{subfig:ML}
	\end{subfigure}
	\hfill
	\begin{subfigure}[b]{0.32\textwidth}
		\centering
		\includegraphics[width=\textwidth]{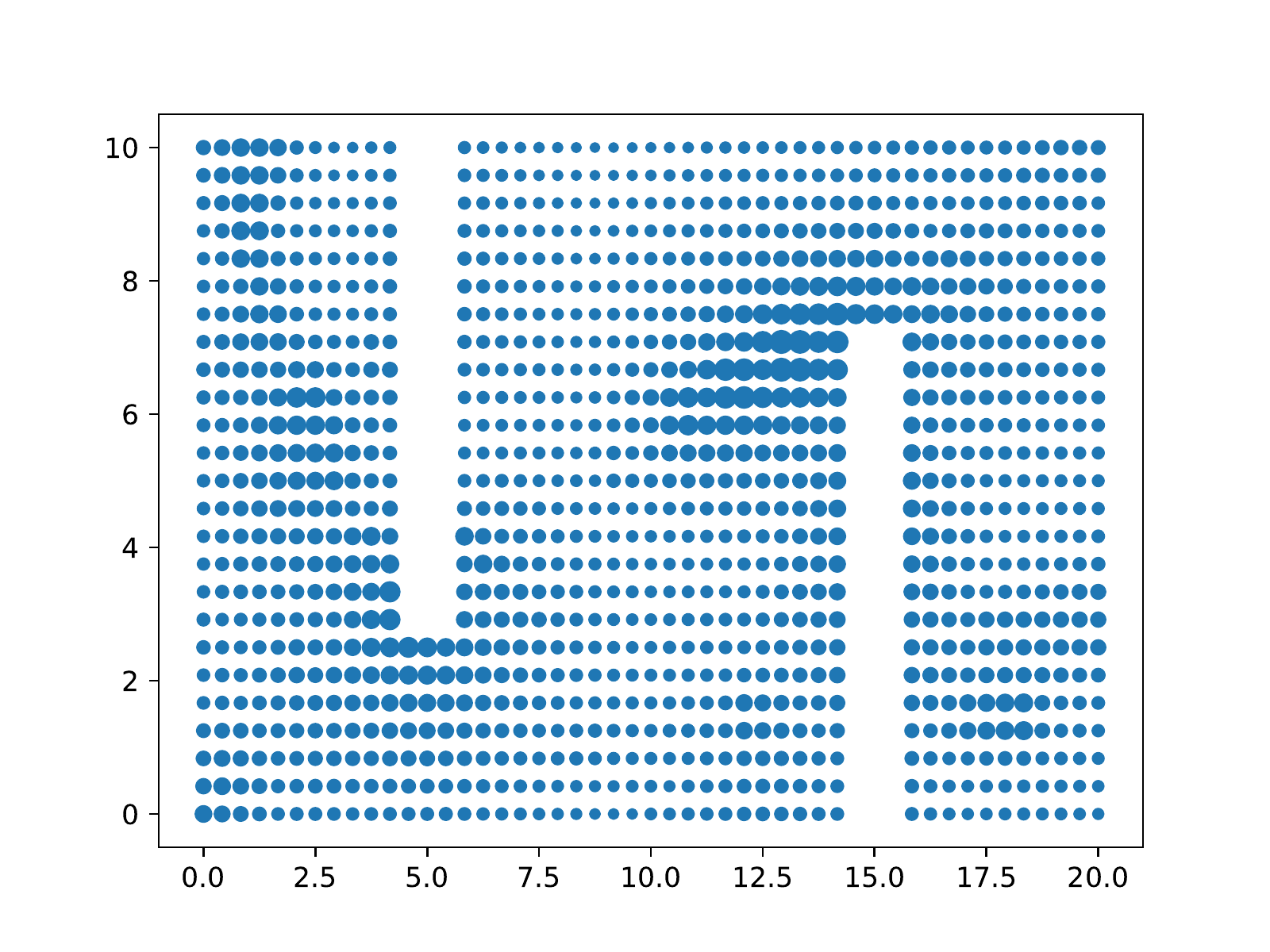}
		\caption{Dual variables $\lambda$.}
		\label{subfig:lambdas}
	\end{subfigure}
	\hfill
	\begin{subfigure}[b]{0.32\textwidth}
		\centering
		\includegraphics[width=\textwidth]{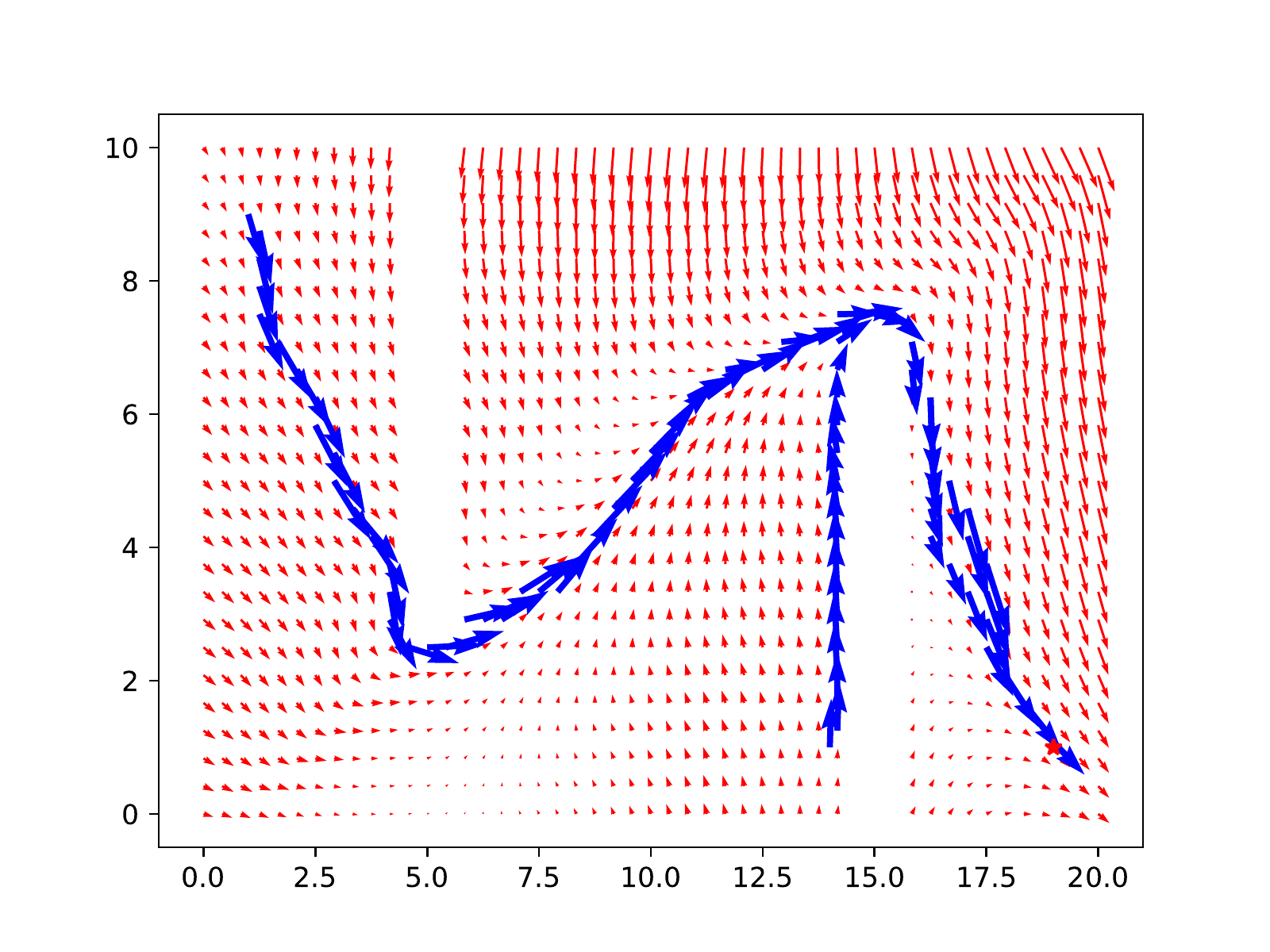}
		\caption{ERM.}
		\label{subfig:ERM}
	\end{subfigure}
	\hfill
	\begin{subfigure}[b]{0.32\textwidth}
		\centering
		\includegraphics[width=\textwidth]{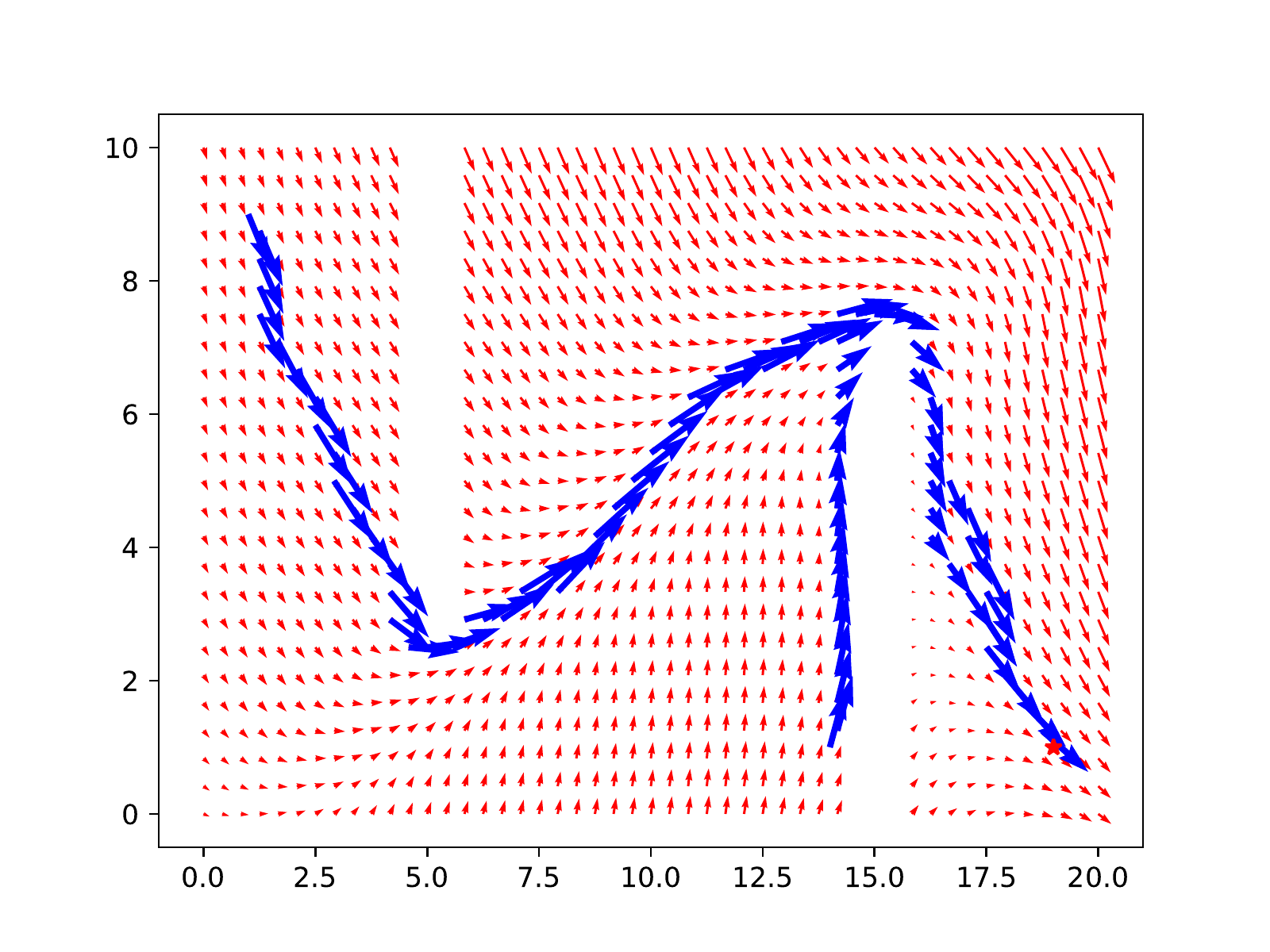}
		\caption{Ambient Regularization.}
		\label{subfig:Ambient}
	\end{subfigure}
	\hfill
	\begin{subfigure}[b]{0.32\textwidth}
		\centering
		\includegraphics[width=\textwidth]{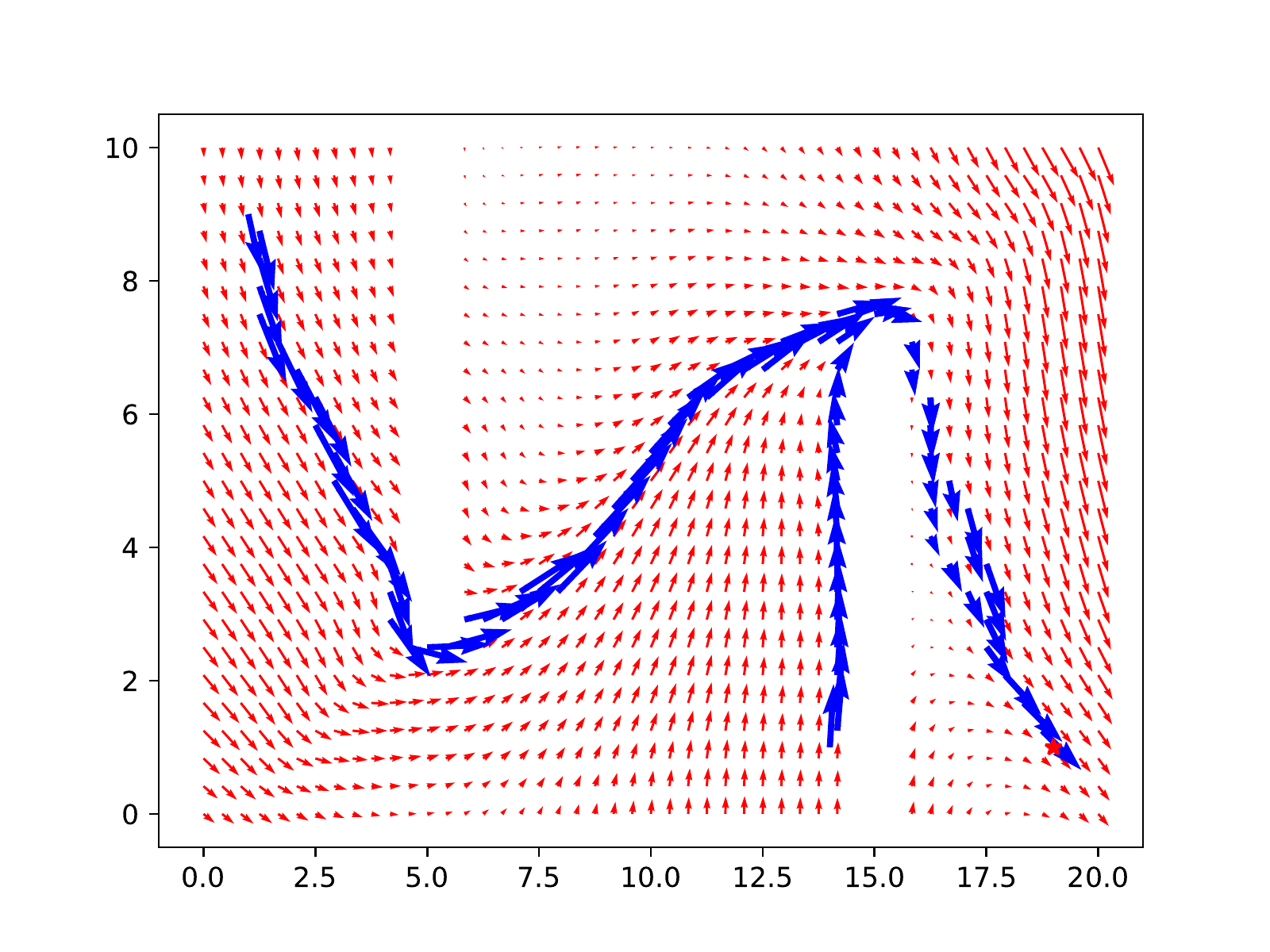}
		\caption{Manifold Regularization.}
		\label{subfig:ManiReg}
	\end{subfigure}
	\caption{Figure \ref{subfig:Dataset} shows the training dataset, blue stars depict unlabeled point, and blue arrow the optimal action at the red star. Figure \ref{subfig:ML} shows the learned function using Manifold Learning, and \ref{subfig:lambdas} its associated dual variables associated. Figures \ref{subfig:ERM}, \ref{subfig:Ambient}, \ref{subfig:ManiReg} show the functions learned using ERM, ambient regularization, and Manifold regularization respectively. }
	\label{fig:Navigation}
\end{figure*}
As measure of merit, we take $100$ random points and we compute the trajectories. A trajectory is successful if it achieves the goal without colliding. The results are shown in Table \ref{table:navigation}, and the learned functions in Figure \ref{fig:Navigation}.

\newpage

\section{Further References}

Since we introduce the Lipschitz constant as a constraint to the learning problem our reformulation and solution methodologies are framed within the constrained learning paradigm \citep{chamon2020probably, chamon2022constrained, yang2019advancing}. Central to the solution of constrained learning problems is the use of dual formulations and dual ascent learning algorithms. These are finding increasing applicability as evidenced by their use in, e.g., adversarial robustness \citep{robey2021adversarial}, graph neural networks \citep{cervino2022training,arghal2021robust}, federated learning \citep{shen2021agnostic}, active learning \citep{elenter2022lagrangian}, reinforcement learning \citep{paternain2019constrained, paternain2022safe, castellano2021reinforcement, bai2021achieving, hasanbeig2018logically}, and wireless communications \citep{eisen2019learning}.

In the context of adversarial attacks to neural networks, manifold based regularization techniques have shown a vast amount of empirical and theoretical evidence of its utility, improving its adversarial robustness\citep{ZHANG2021282,khoury2018geometry,ma2018characterizing,moosavi2019robustness,jin2020manifold,lassance2021laplacian}. Some works seek to obtain manifold attacks, which are more realistic attacks than utilizing the norm-$\infty$ ball, given the high dimensionality of the input and the low dimensional structure of the data \citep{stutz2019disentangling}. Smooth function have also been studied in the context of robustness \citep{pmlr-v137-rosca20a,bubeck2021a,bubeck2021law}. 

Our work, is based on previous results that show convergence of graph laplacians to Laplace-Beltrami operators. There exists a vast amount of work on that validates the convergence results for point-cloud operators over Manifolds \citep{hein2005graphs,hein2007graph,dunson2021spectral,wu2018think}.

Regarding Lipschitz constant estimation for neural networks,  \cite{fazlyab2019efficient} has formulated the problem as a convex optimization problem.

\newpage
\section{Ablation Study on Laplacian Construction}\label{sec:ablation}

In this section we study the impact of the temperature coefficient $t$ in the construction of the Laplacian. To do so we repeat the setting of Figure \ref{fig:two_moons}, and we vary the value of the temperature $t$. We consider the two moons dataset problem with $1$ labeled, and $200$ unlabeled samples per class. We vary the value of the temperature coefficient, which varies the number of cross-manifold edges, and therefore makes the problem more challenging.

\begin{table}[H]
\begin{tabular}{||cc c c c||}
\hline
\begin{tabular}[c]{@{}c@{}}Value of \\ Heat Kernel\end{tabular} & \begin{tabular}[c]{@{}c@{}}Connected \\ Components\end{tabular} & \begin{tabular}[c]{@{}c@{}}Number of \\ Cross-Manifold Edges\end{tabular} & \begin{tabular}[c]{@{}c@{}}Manifold \\ Regularization\end{tabular} & \begin{tabular}[c]{@{}c@{}}Manifold Gradient \\ (Ours)\end{tabular} \\ \hline\hline
$0.0040$                                                        & $2$                                                             & $0$                                                                       & $100$\%                                                            & $100$\%                                                             \\ \hline
$0.0050$                                                        & $1$                                                             & $2$                                                                       & N/A                                                                & $100$\%                                                             \\ \hline
$0.0060$                                                        & $1$                                                             & $3$                                                                       & N/A                                                                & $100$\%                                                             \\ \hline
$0.0070$                                                        & $1$                                                             & $10$                                                                      & N/A                                                                & $100$\%                                                             \\ \hline
$0.0080$                                                        & $1$                                                             & $20$                                                                      & N/A                                                                & $100$\%                                                             \\ \hline
$0.0090$                                                        & $1$                                                             & $27$                                                                      & N/A                                                                & $100$\%                                                             \\ \hline
$0.0100$                                                        & $1$                                                             & $42$                                                                      & N/A                                                                & $100$\%                                                             \\ \hline
$0.0150$                                                        & $1$                                                             & $124$                                                                     & N/A                                                                & $100$\%                                                             \\ \hline
$0.0175$                                                        & $1$                                                             & $176$                                                                     & N/A                                                           & $100$\%                                                                \\ \hline
$0.0180$                                                        & $1$                                                             & $192$                                                                     & N/A                                                                & $100$\%                                                                 \\ \hline
$0.0190$                                                        & $1$                                                             & $217$                                                                     & N/A                                                                & $100$\%                                                                 \\ \hline
$0.0200$                                                        & $1$                                                             & $251$                                                                     & N/A                                                                & N/A                                                                 \\ \hline
\end{tabular}
\caption{Ablation study on the temperature of the heat kernel $t$. We plot the accuracy when it achieves $100$\%, and N/A otherwise, given that an accuracy of less than $100$\% is not representative as the method fails to capture the manifold structure of the problem.  \label{table:ablation_heat_temperature}}
\end{table}

As seen in table \ref{table:ablation_heat_temperature}, Laplacian regularization fails to achieve a perfect accuracy when the number of connected components is less than $2$. That is to say, manifold regularization achieves a perfect accuracy when each class has a connected component. However, once the components become connected, Laplacian regularization smoothness the integral of the gradient, and therefore does no properly identify the transition between components. 

As can be seen in table \ref{table:ablation_heat_temperature}, our method is more robust to non-exact manifolds. Which means that if the manifold is not calculated perfectly, and as a result we obtain $1$ connected component as opposed to $2$ separate moons, our method still works. 

It is important to remark that our method still works when the connected components have cross-manifold edges in different places of the manifold. As an example, take the Laplacian with heat kernel $t=0.007$, there are edges on both side of the manifold. Moreover, our method is able to distinguish between the two classes even with edges in the middle of the two manifold as can be seen with $t=0.0150$, and $t=0.0175$ (cf. \ref{subfig:ablation_0150}).

Our method brakes once the manifold structure vanishes and most of the points are interconnected, as can be seen with $t=0.02$. In this case there are $251$ cross manifold edges, and the low-dimensional structure of the problem disappears. 

In all, our proposed solution is more robust to imperfect estimation of the manifold. Even when the number of cross-edges is large, our method is able to create a partition between classes given that it finds the points with maximal separation, and allows the function to change values between them. 

\begin{figure*}
	\centering
	\begin{subfigure}[b]{0.32\textwidth}
		\centering
		\includegraphics[width=\textwidth]{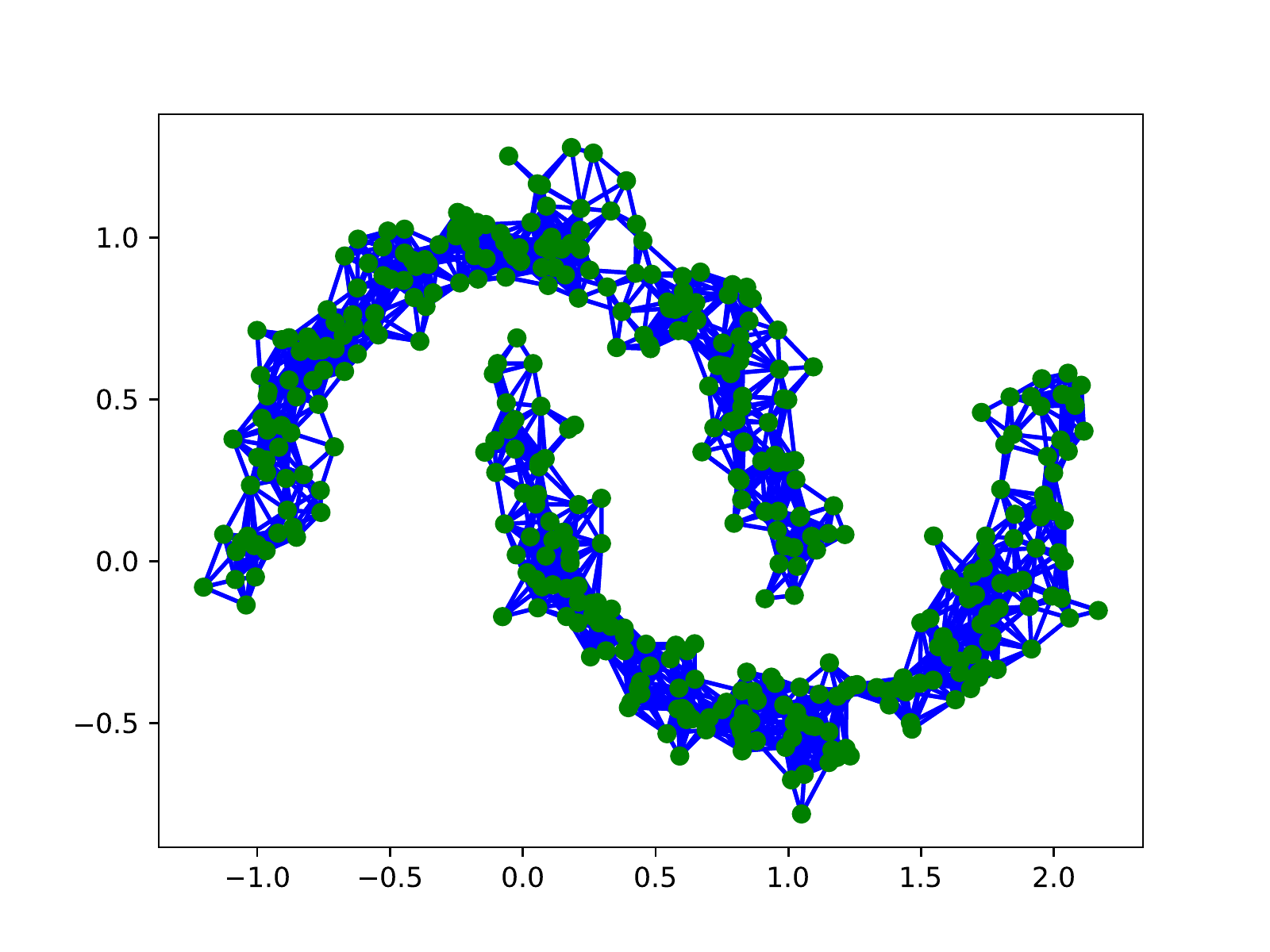}
		\caption{$t=0.004$. }
		\label{subfig:0040}
	\end{subfigure}
	\hfill
	\begin{subfigure}[b]{0.32\textwidth}
		\centering
		\includegraphics[width=\textwidth]{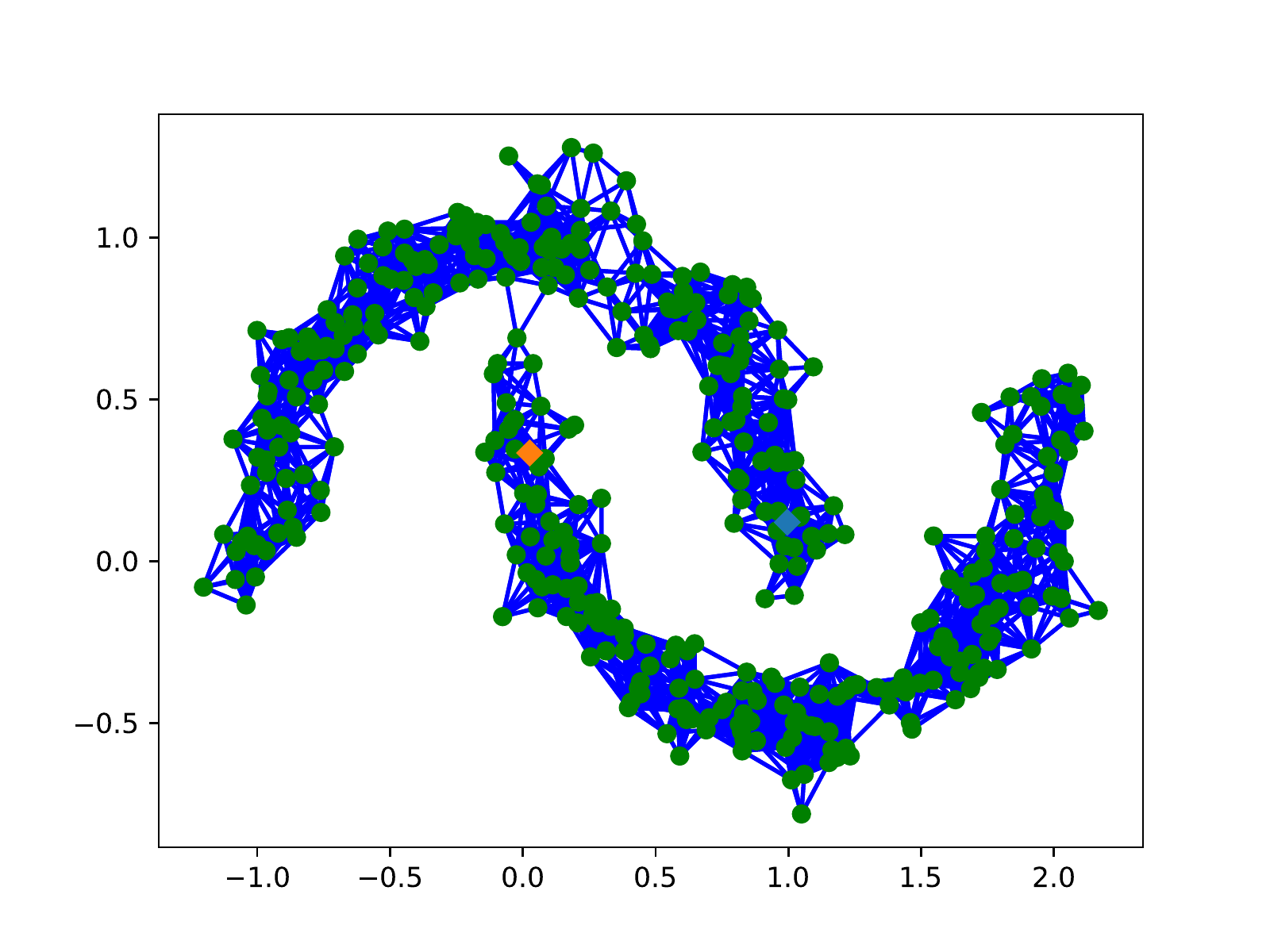}
		\caption{$t=0.005$. }
		\label{subfig:0050}
	\end{subfigure}
	\hfill
	\begin{subfigure}[b]{0.32\textwidth}
		\centering
		\includegraphics[width=\textwidth]{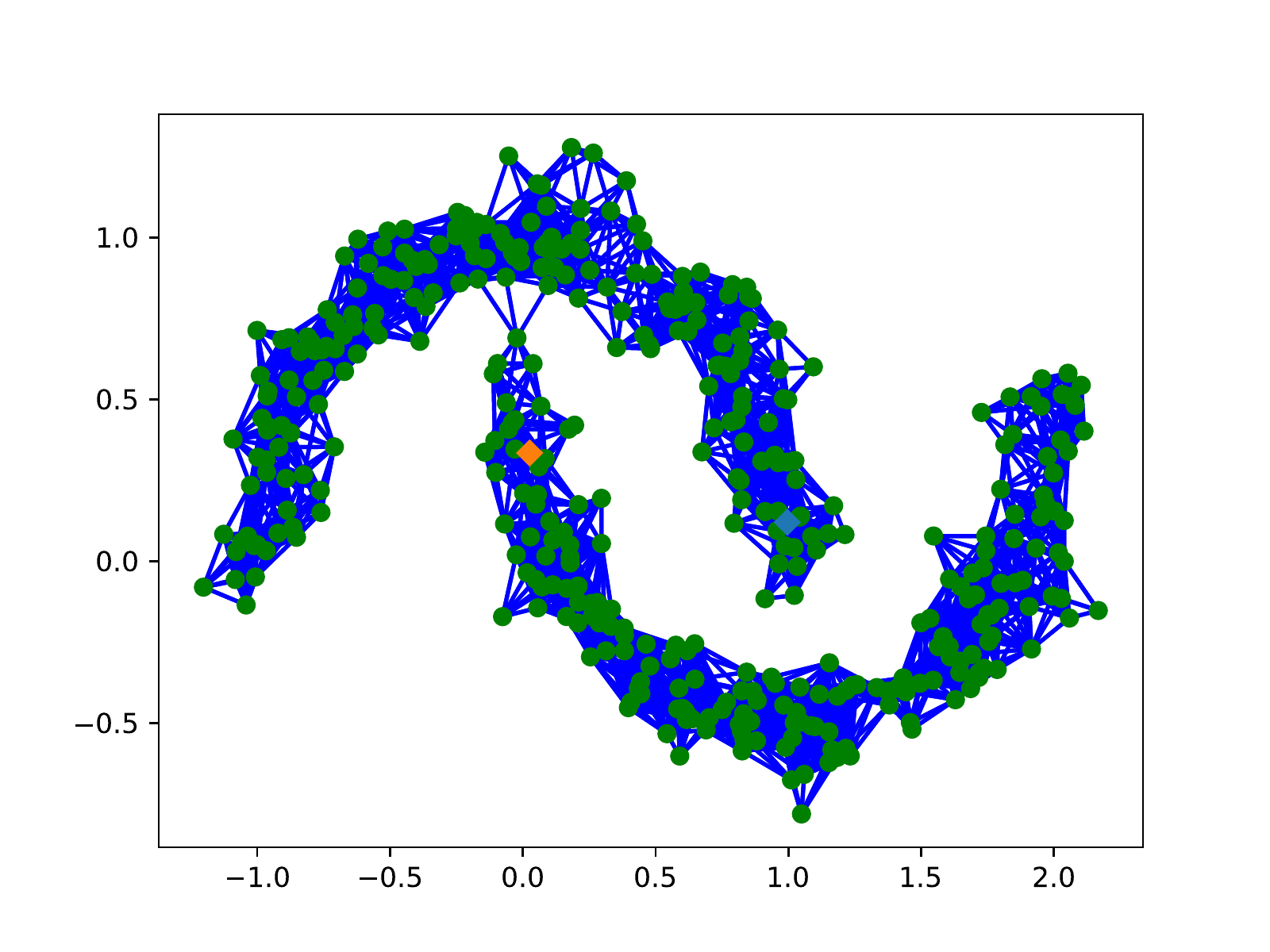}
		\caption{$t=0.006$.}
		\label{subfig:ablation_0060}
	\end{subfigure}
	\hfill
	\begin{subfigure}[b]{0.32\textwidth}
		\centering
		\includegraphics[width=\textwidth]{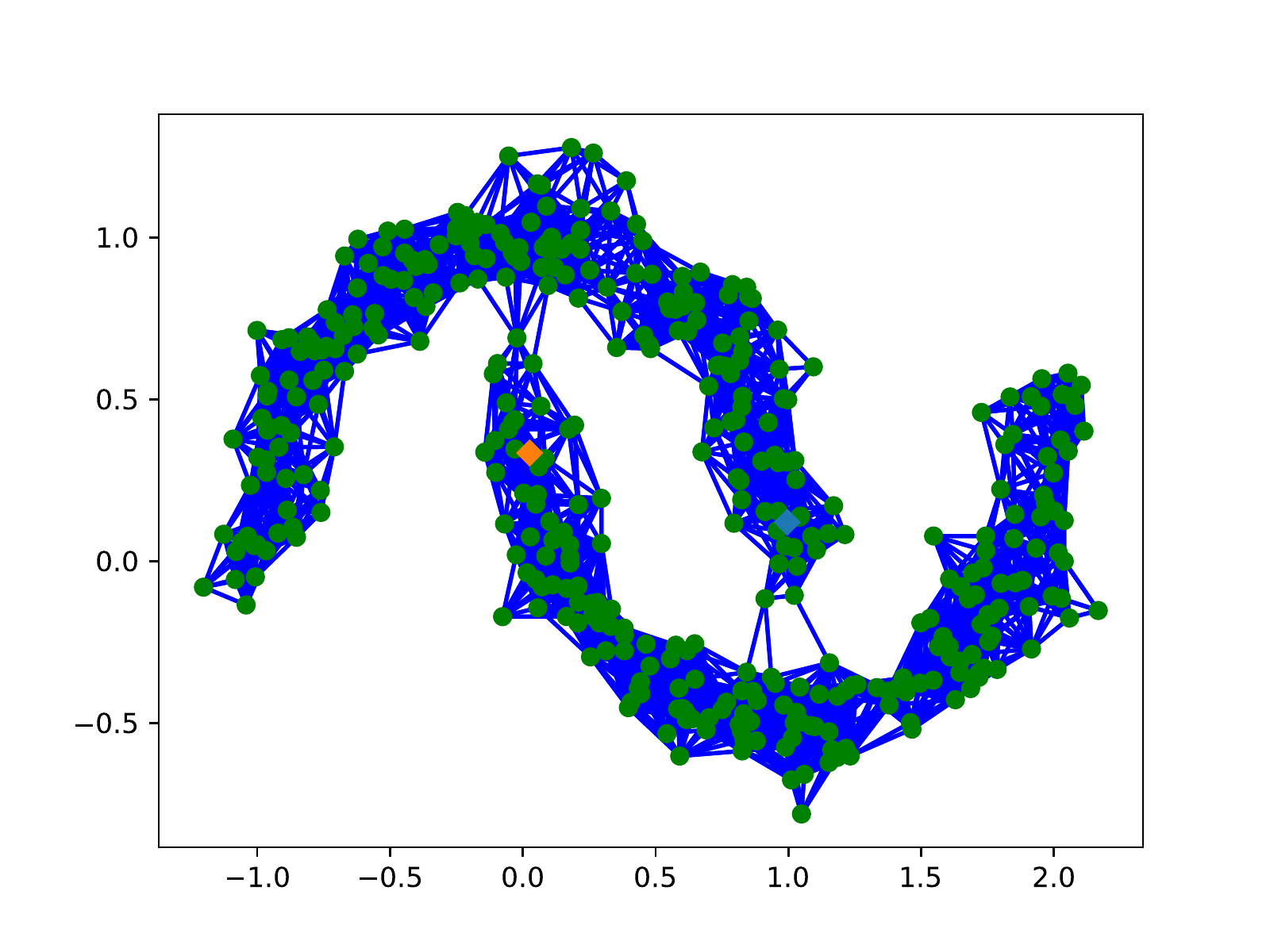}
		\caption{$t=0.007$.}
		\label{subfig:ablation_0070}
	\end{subfigure}
	\hfill
	\begin{subfigure}[b]{0.32\textwidth}
		\centering
		\includegraphics[width=\textwidth]{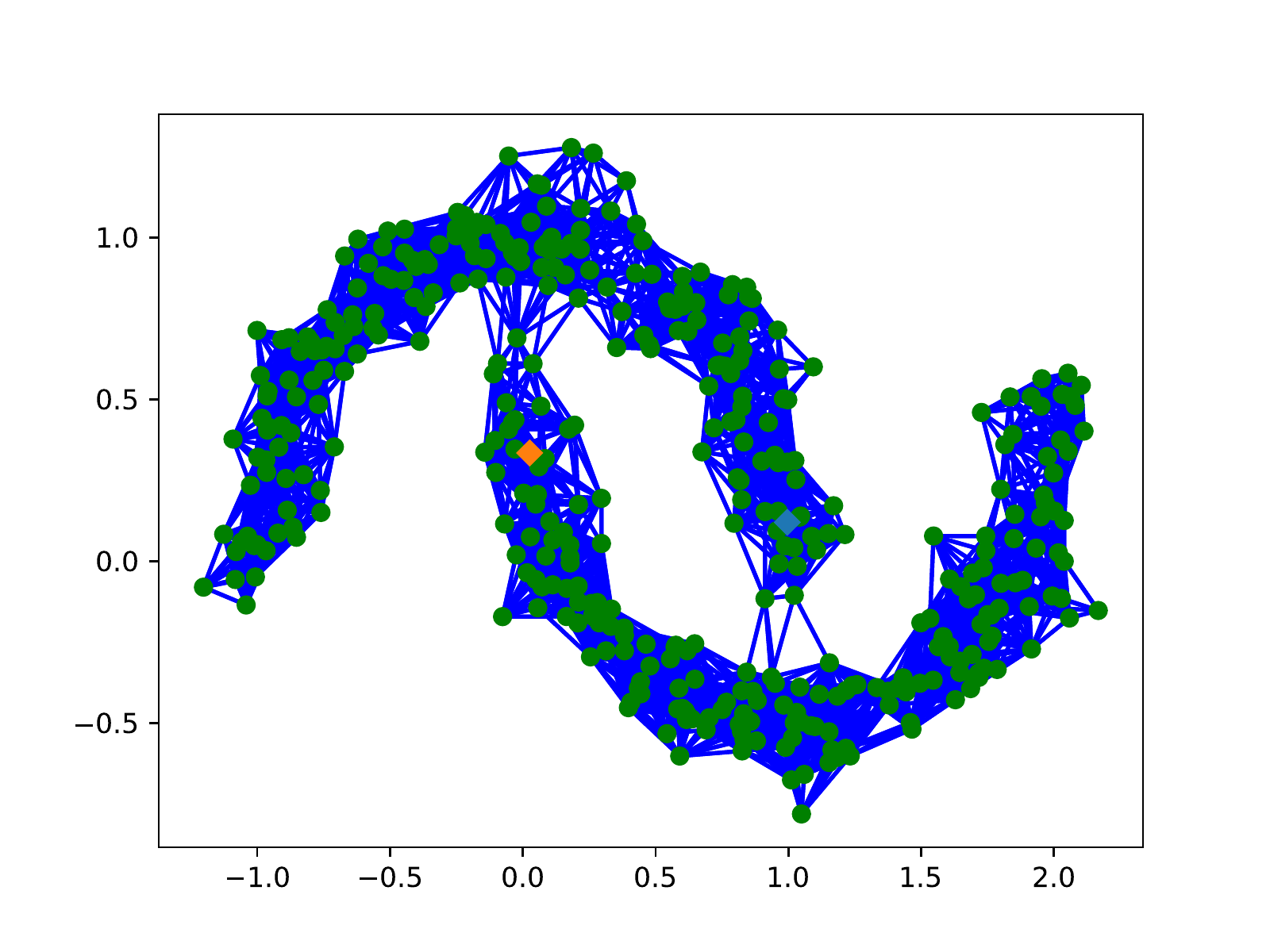}
		\caption{$t=0.008$.}
		\label{subfig:ablation_0080}
	\end{subfigure}
	\hfill
	\begin{subfigure}[b]{0.32\textwidth}
		\centering
		\includegraphics[width=\textwidth]{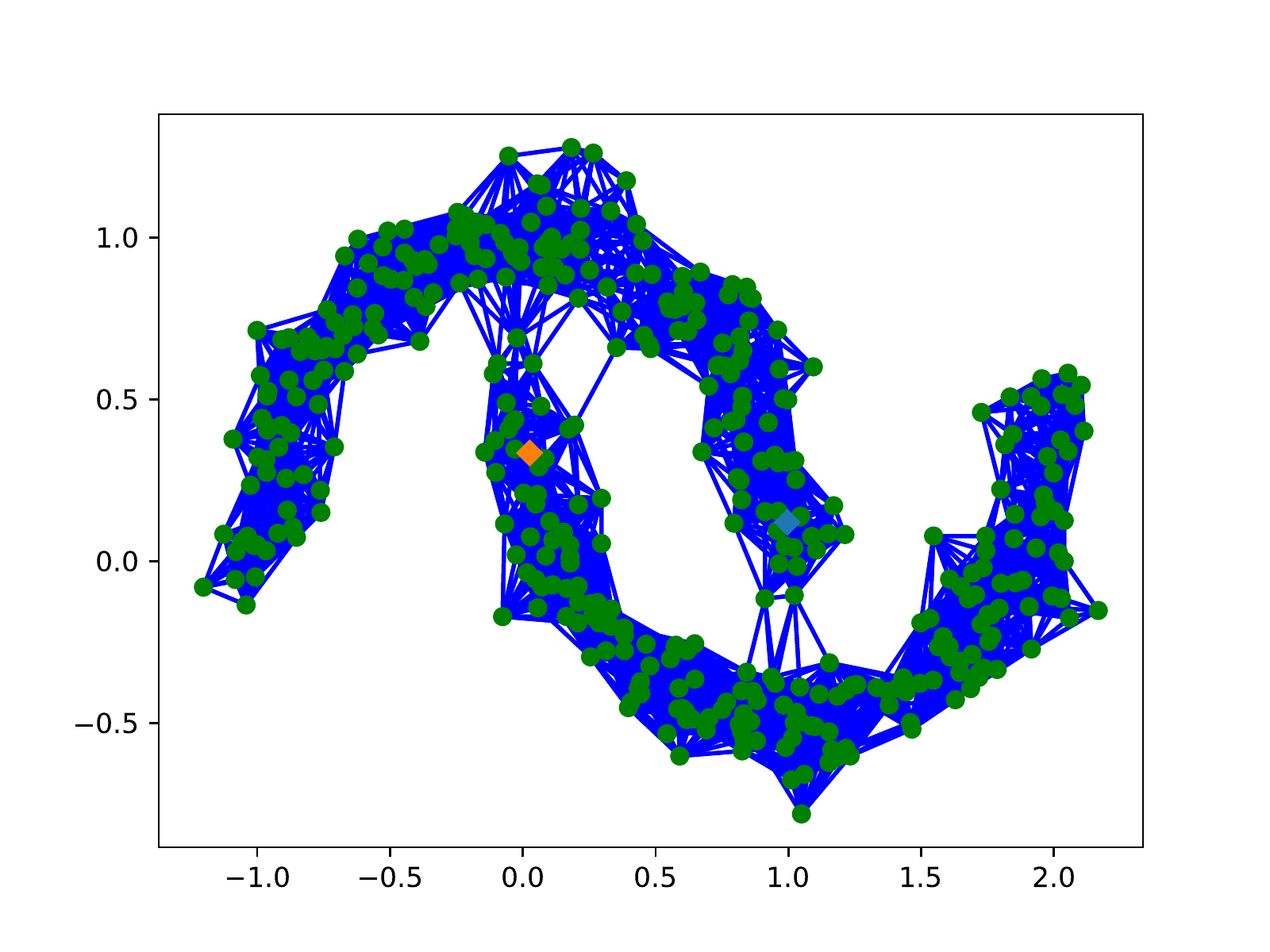}
		\caption{$t=0.009$.}
		\label{subfig:ablation_0090}
	\end{subfigure}
 \hfill
	\begin{subfigure}[b]{0.32\textwidth}
		\centering
		\includegraphics[width=\textwidth]{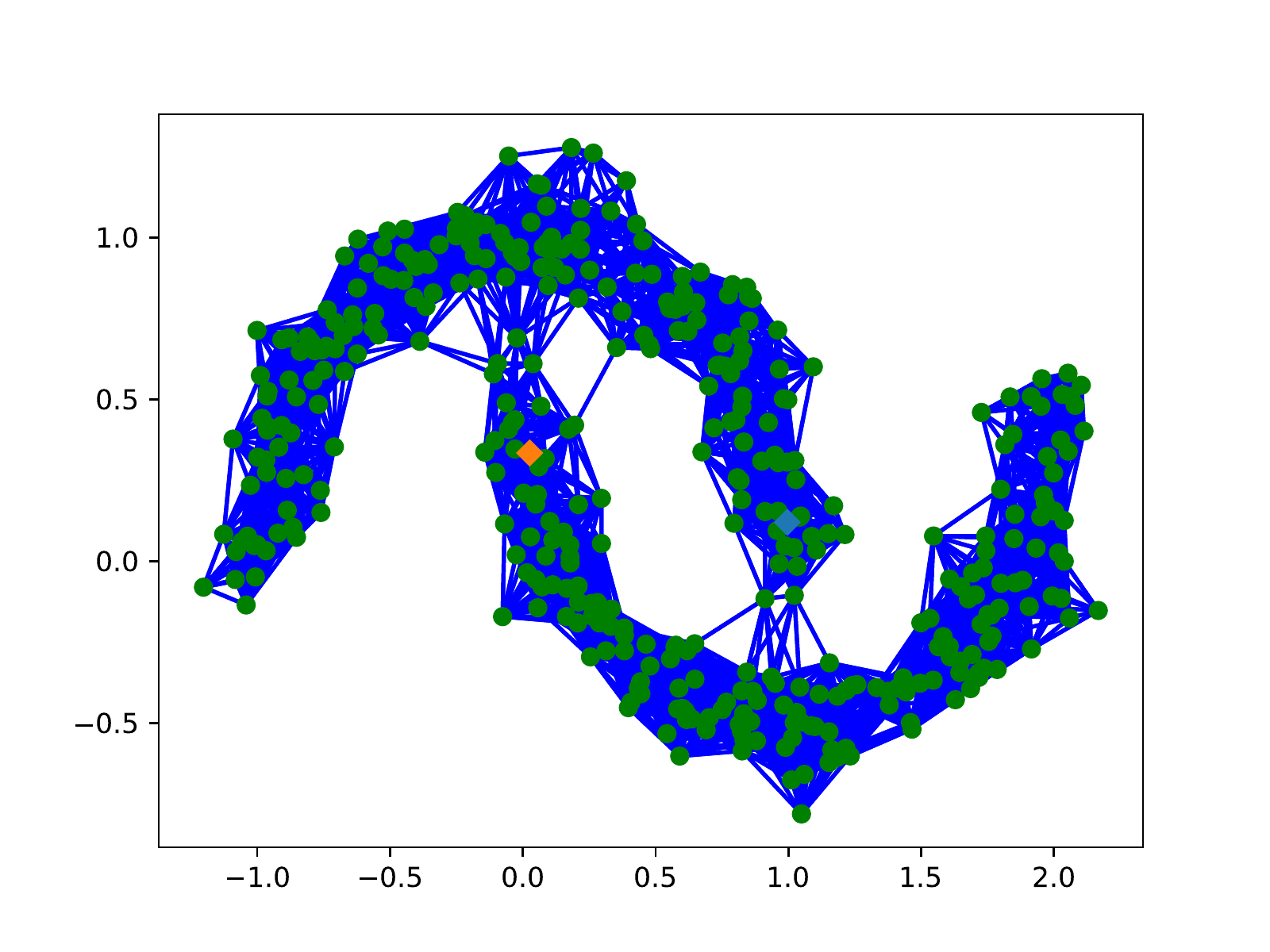}
		\caption{$t=0.01$.}
		\label{subfig:ablation_0100}
	\end{subfigure}
	\hfill
	\begin{subfigure}[b]{0.32\textwidth}
		\centering
		\includegraphics[width=\textwidth]{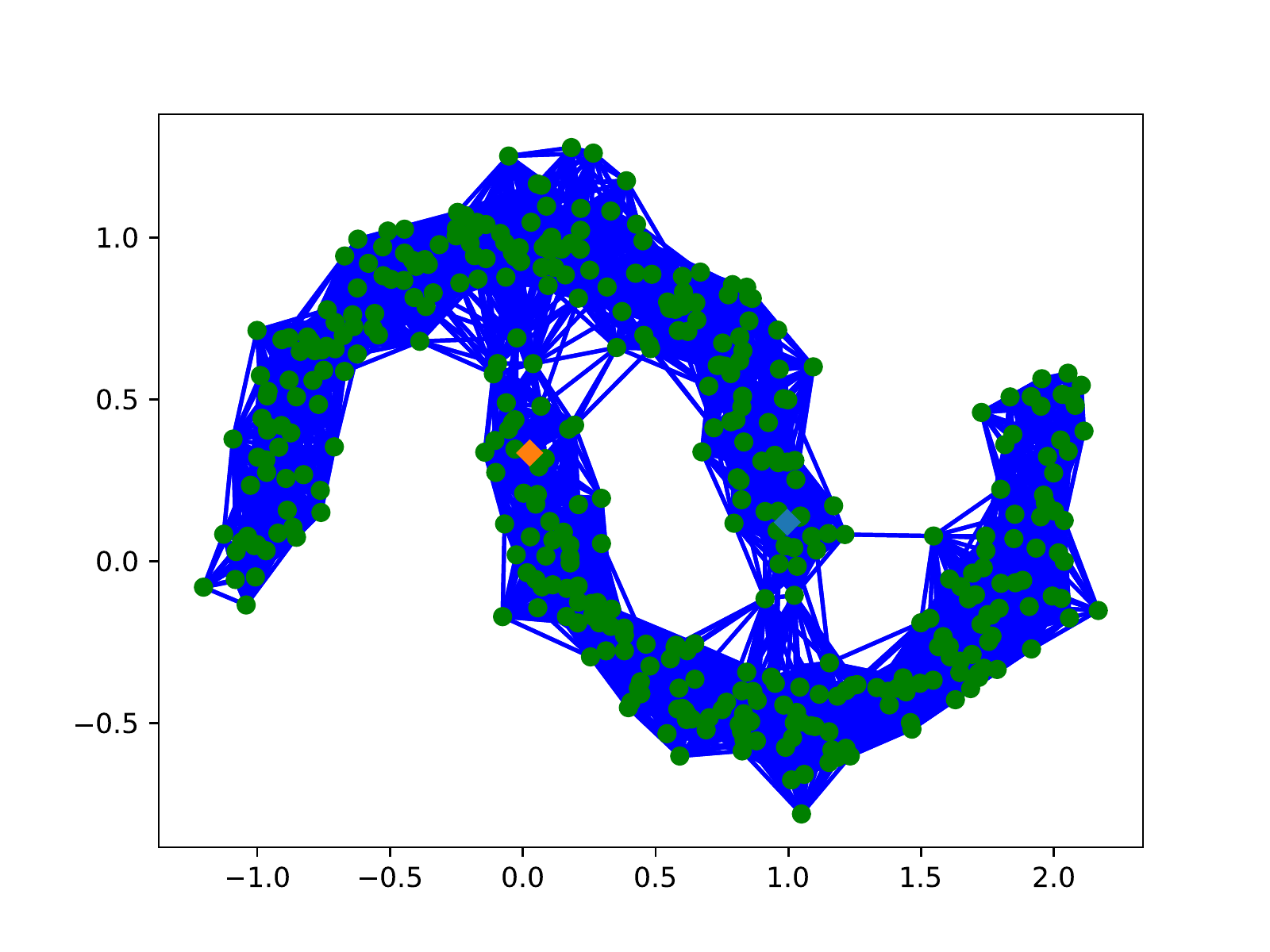}
		\caption{$t=0.015$.}
		\label{subfig:ablation_0150}
	\end{subfigure}
	\hfill
	\begin{subfigure}[b]{0.32\textwidth}
		\centering
		\includegraphics[width=\textwidth]{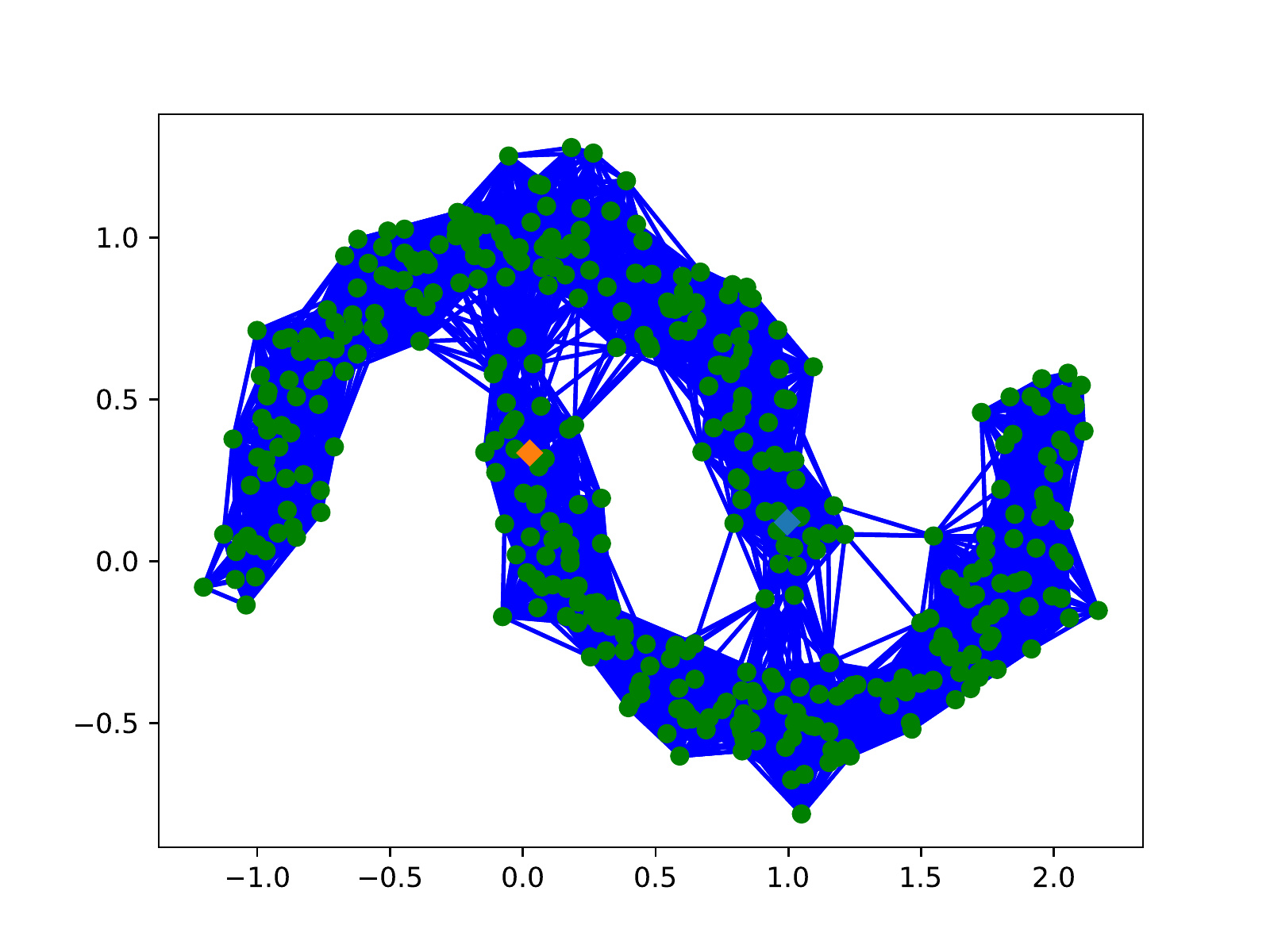}
		\caption{$t=0.0175$. }
		\label{subfig:ablation_0175}
	\end{subfigure}
 \hfill
	\begin{subfigure}[b]{0.32\textwidth}
		\centering
		\includegraphics[width=\textwidth]{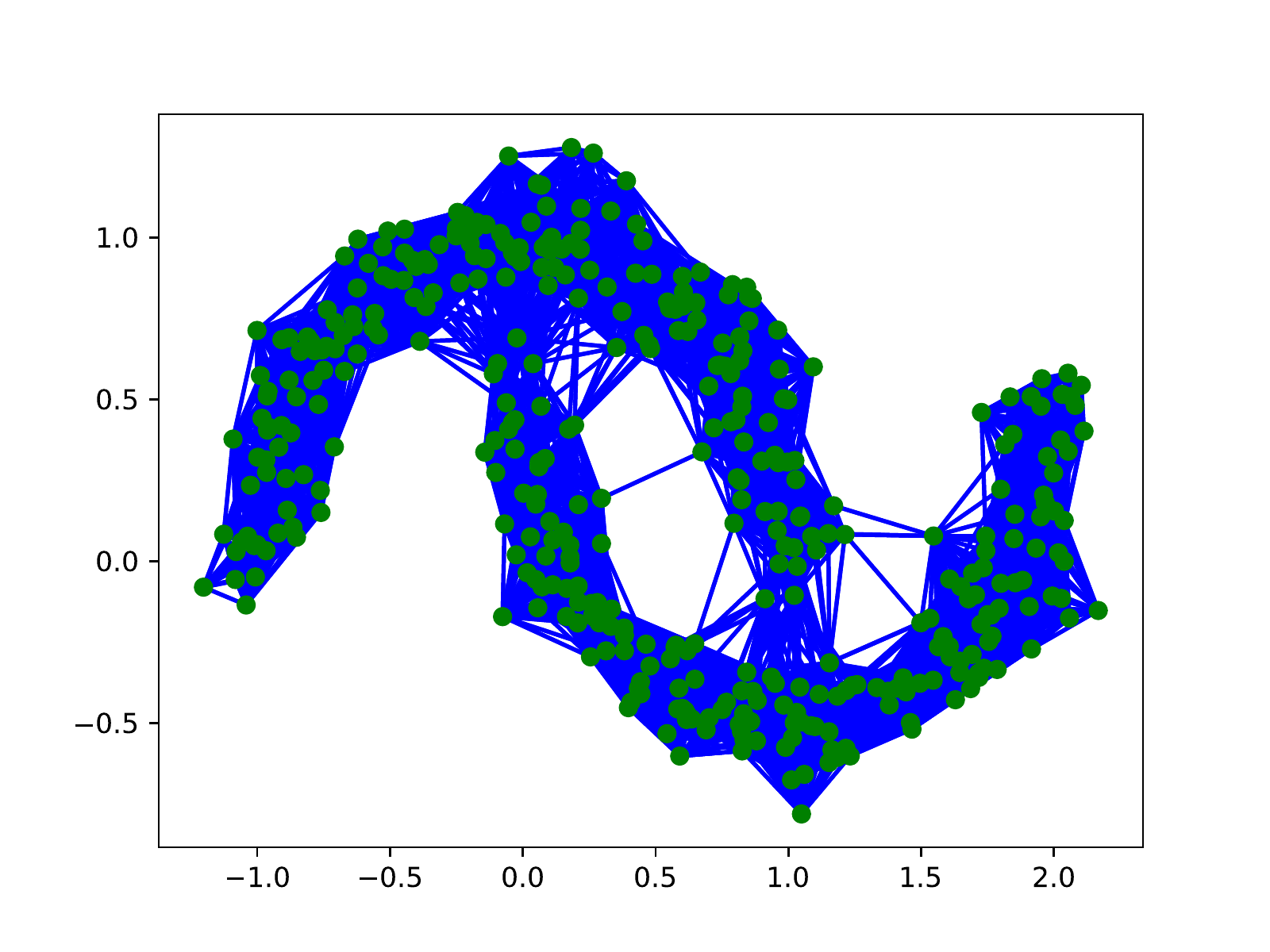}
		\caption{$t=0.018$. }
		\label{subfig:ablation_018}
	\end{subfigure}
 \hfill
 	\begin{subfigure}[b]{0.32\textwidth}
		\centering
		\includegraphics[width=\textwidth]{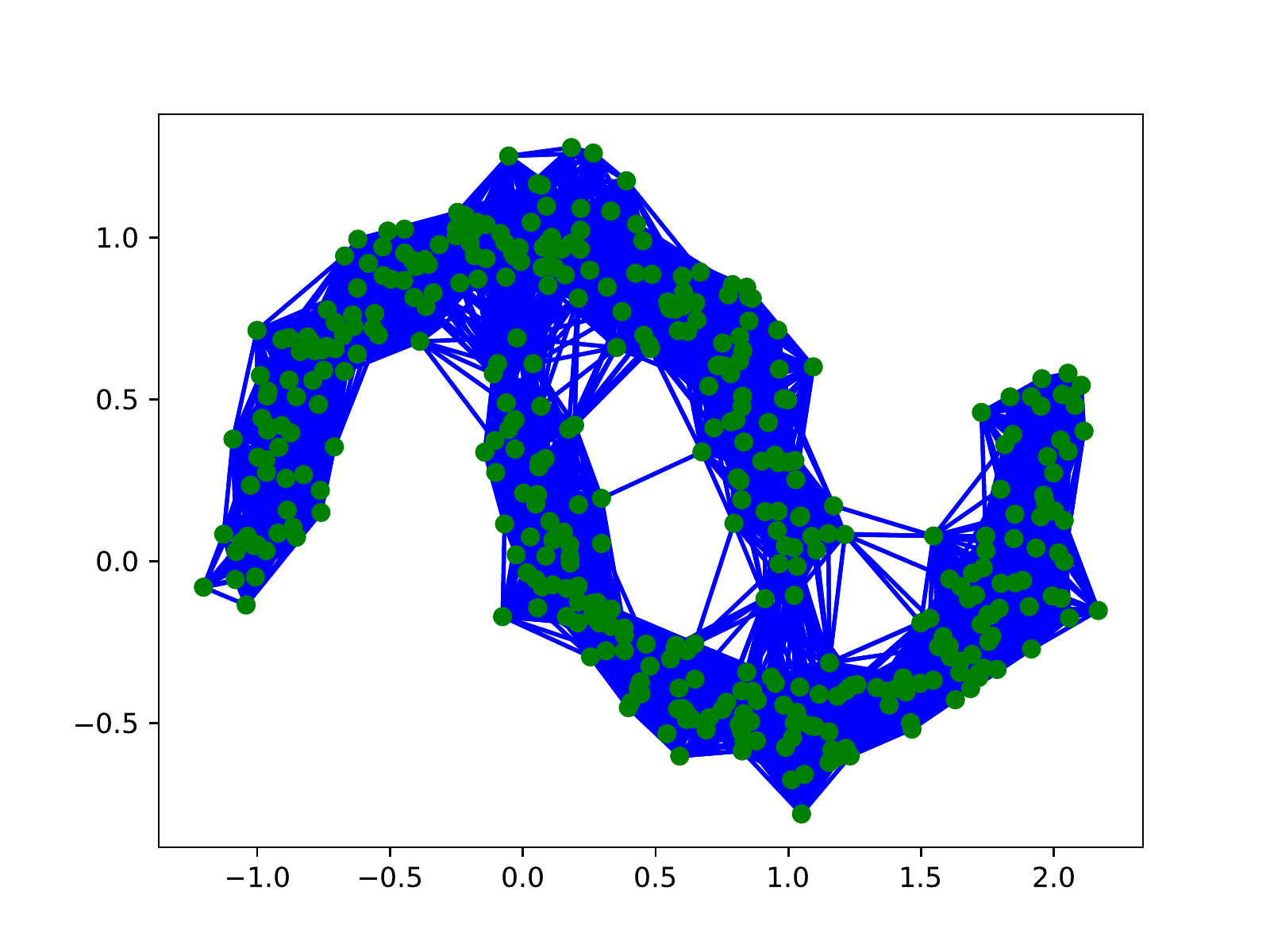}
		\caption{$t=0.019$. }
		\label{subfig:ablation_019}
	\end{subfigure}
 \hfill
	\begin{subfigure}[b]{0.32\textwidth}
		\centering
		\includegraphics[width=\textwidth]{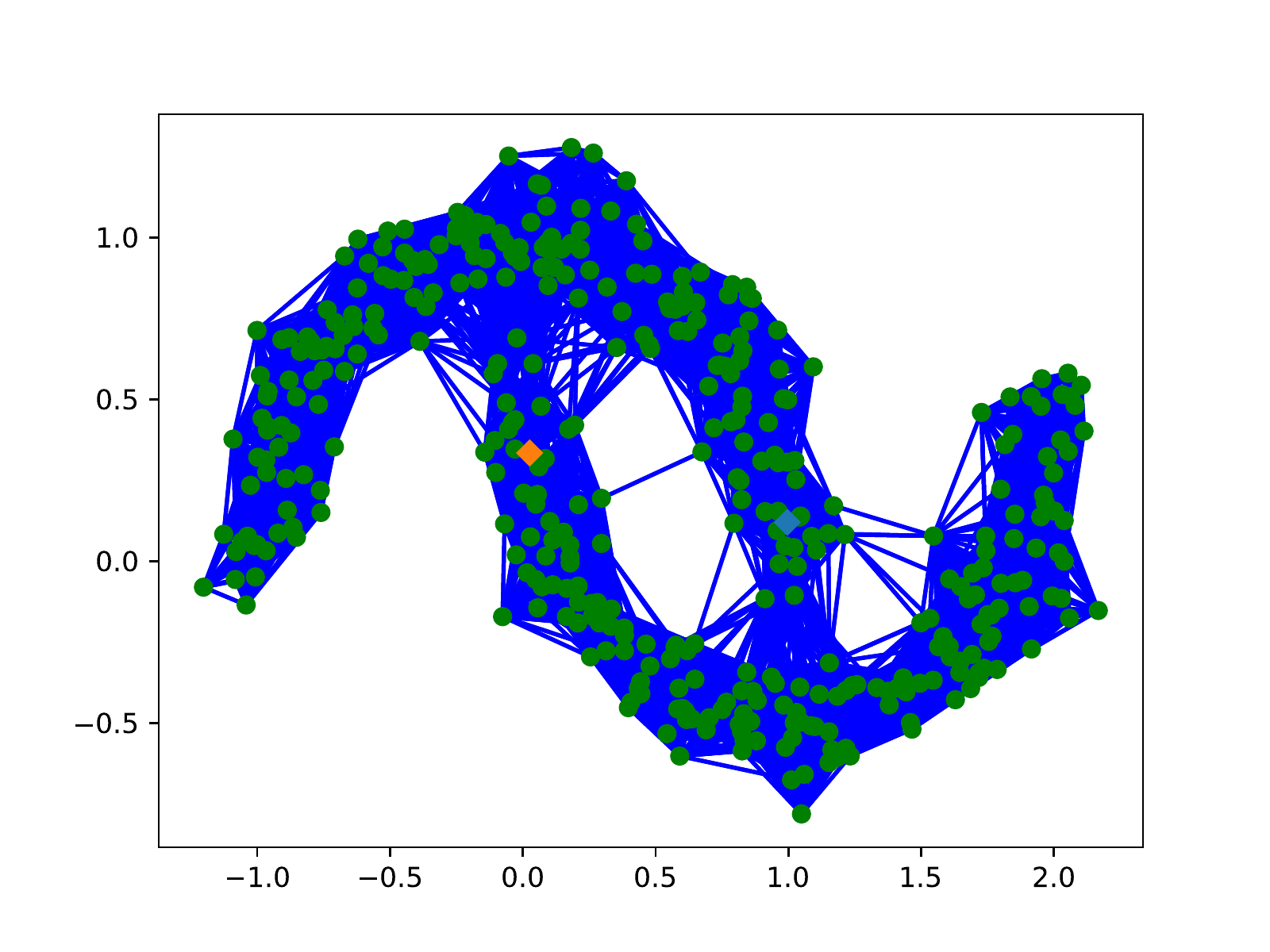}
		\caption{$t=0.02$. }
		\label{subfig:ablation_02}
	\end{subfigure}
	\caption{Laplacian for different values of temperature coefficient $t$. }
	\label{fig:ablation_laplacian}
\end{figure*}


\end{document}

%% file: alejandro_intro.tex
\section{Introduction}

Learning smooth functions has been shown to be advantageous in general and is of particular interest in physical systems. This is because of the general observation that close input features tend to be associated with close outputs and of the particular fact that in physical systems, Lipschitz continuity of input-output maps translates to stability and safety \citep{oberman2018lipschitz, finlay2018improved,couellan2021coupling, finlay2018lipschitz, pauli2021training, krishnan2020lipschitz, shi2019neural, lindemann2021learning, arghal2021robust}. 

To learn smooth functions one can require the parameterization to be smooth. Such is the idea, e.g., of spectral normalization of weights in neural networks \citep{miyato2018spectral, zhao2020spectral}. Smooth parameterizations have the advantage of being globally smooth, but they may be restrictive because they impose smoothness for inputs that are not necessarily realized in the data. This drawback motivates the use of Lipschitz \textit{penalties} in risk minimization \citep{oberman2018lipschitz, finlay2018improved, couellan2021coupling, pauli2021training, bungert2021clip}, which offers the opposite tradeoff. Since penalties encourage but do not enforce small Lipschitz constants, we may learn functions that are smooth on average, but with no global guarantees of smoothness at every point in the support of the data~\citep{bubeck2021a,bubeck2021law}. Formulations that guarantee \emph{global} smoothness can be obtained if the risk minimization problem is modified by the addition of a Lipschitz constant \textit{constraint} \citep{krishnan2020lipschitz, shi2019neural, lindemann2021learning, arghal2021robust}. This yields formulations that guarantee Lipschitz smoothness in all possible inputs without the drawback of enforcing smoothness outside of the input data distribution. 
Several empirical studies \citep{krishnan2020lipschitz, shi2019neural, lindemann2021learning, arghal2021robust} demonstrated the advantage of imposing global smoothness constraints on observed inputs.

\begin{figure*}
	\centering
		\begin{subfigure}[b]{0.24\textwidth}
		\centering
		\includegraphics[width=\textwidth]{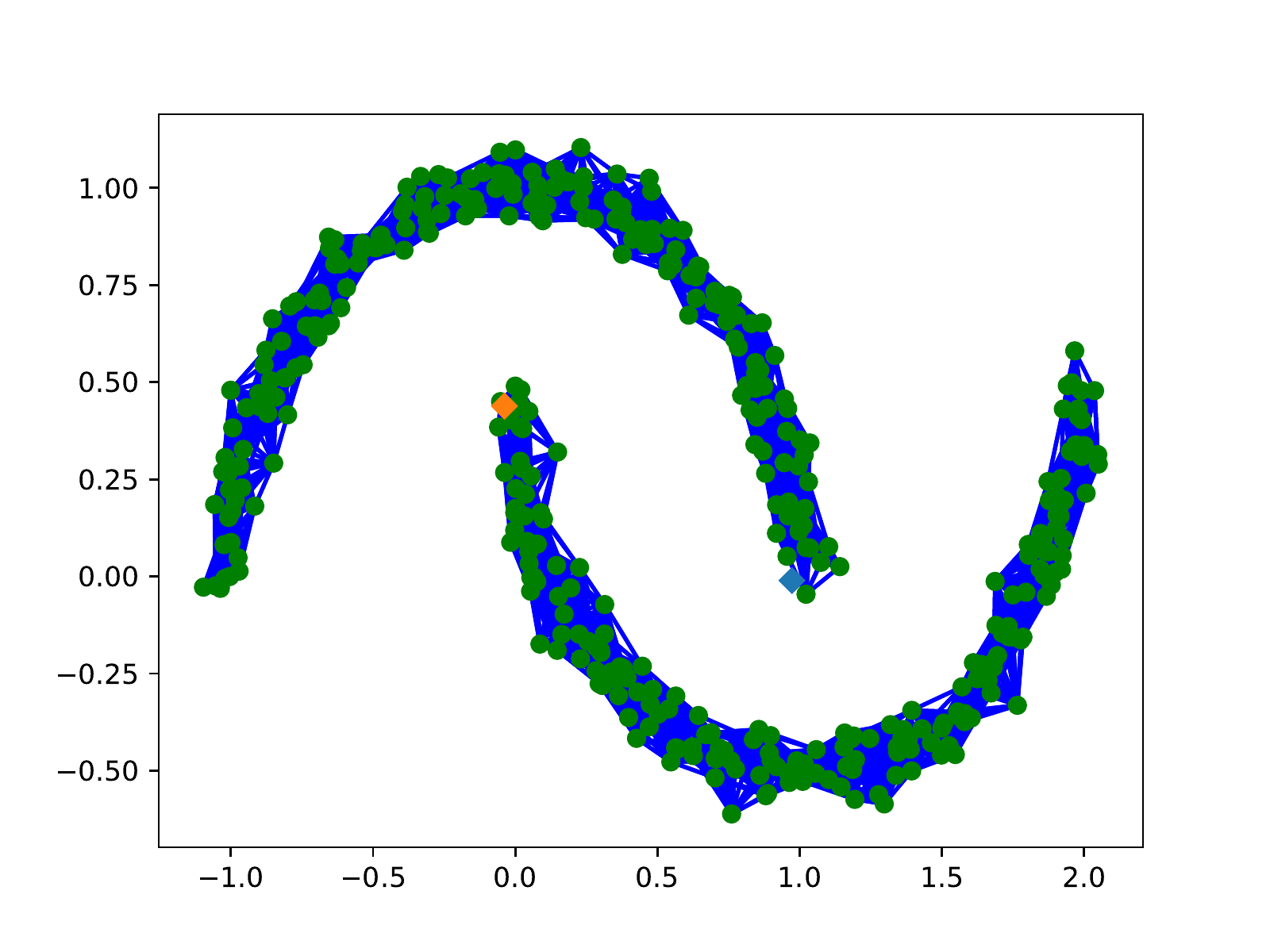}
	\end{subfigure}
	\hfill
	\begin{subfigure}[b]{0.24\textwidth}
		\centering
		\includegraphics[width=\textwidth]{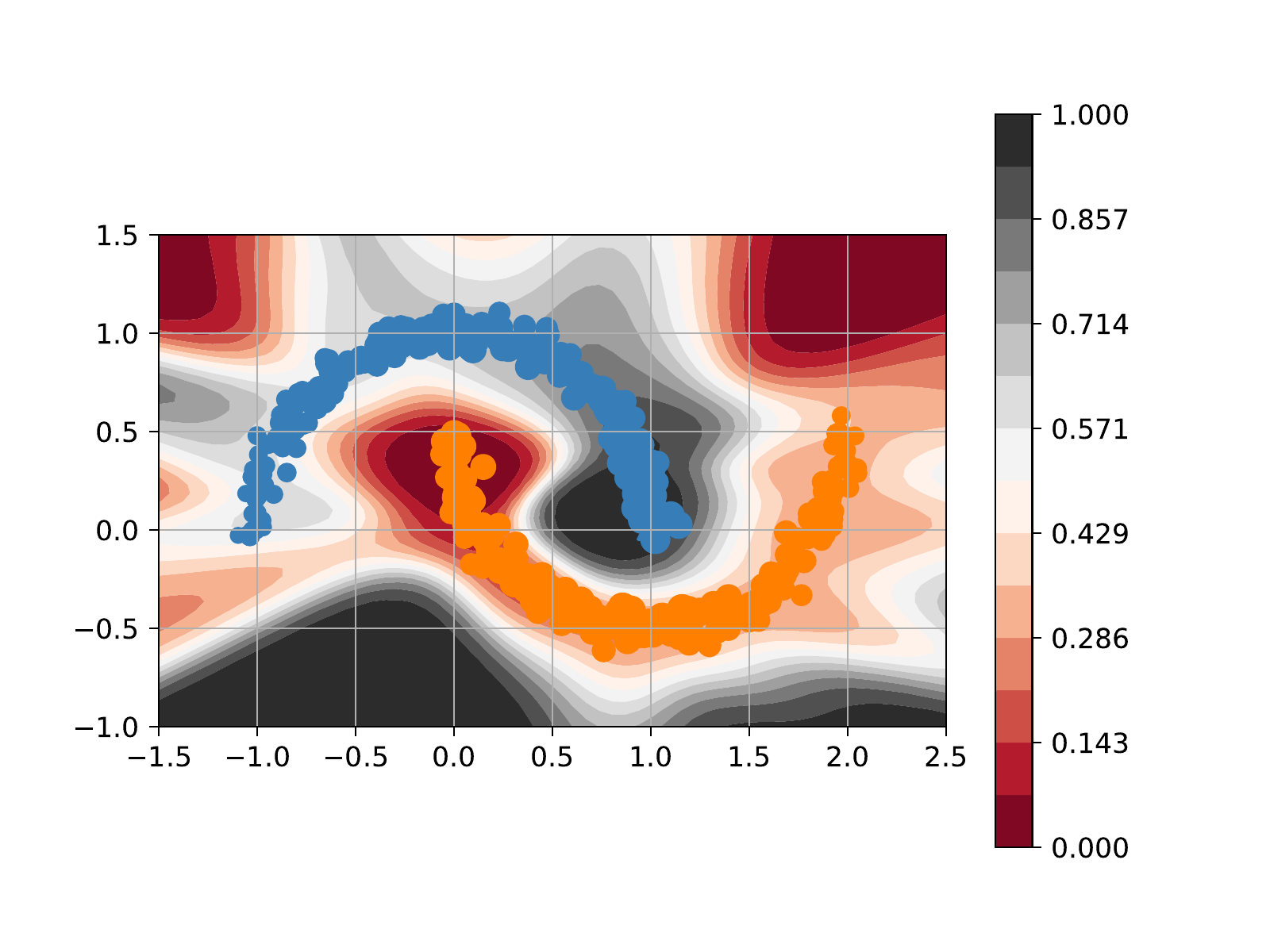}
	\end{subfigure}
	\hfill
	\begin{subfigure}[b]{0.24\textwidth}
		\centering
		\includegraphics[width=\textwidth]{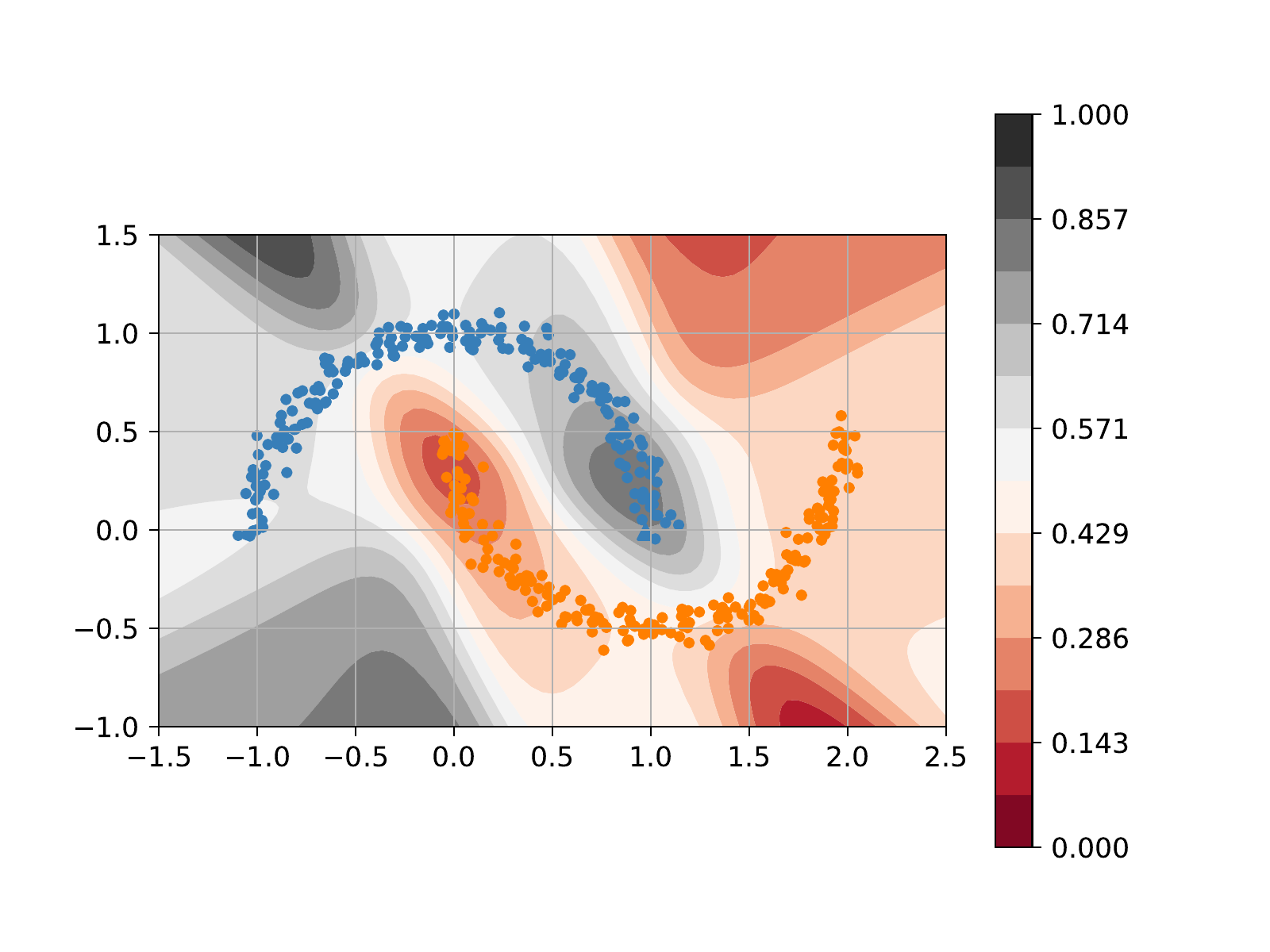}
	\end{subfigure}
	\hfill
	\begin{subfigure}[b]{0.24\textwidth}
		\centering
		\includegraphics[width=\textwidth]{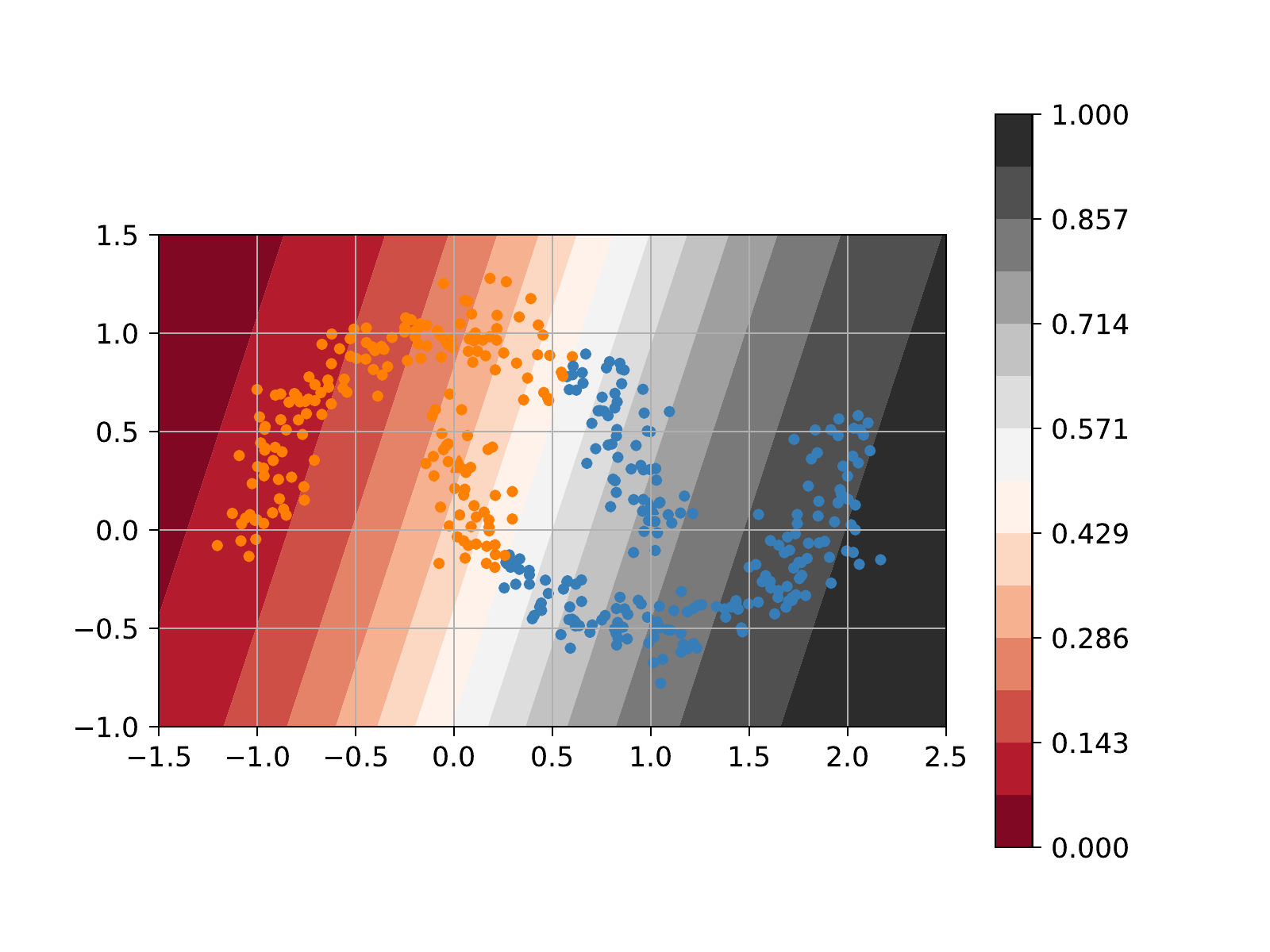}
	\end{subfigure}
	\hfill
	\begin{subfigure}[b]{0.24\textwidth}
		\centering
		\includegraphics[width=\textwidth]{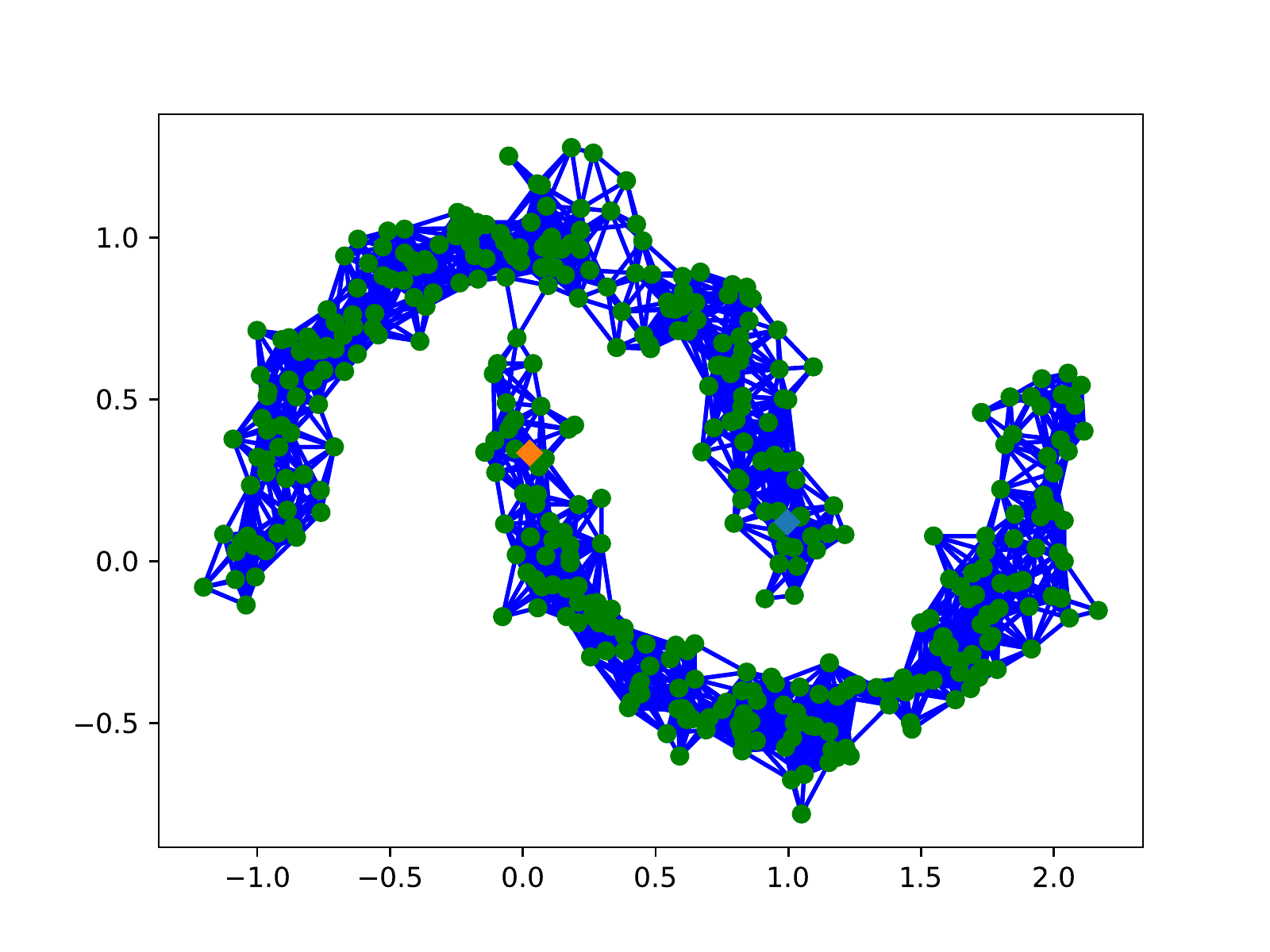}
		\caption{\tiny Dataset\label{subfig:dataset}}
	\end{subfigure}
	\hfill
	\begin{subfigure}[b]{0.24\textwidth}
		\centering
		\includegraphics[width=\textwidth]{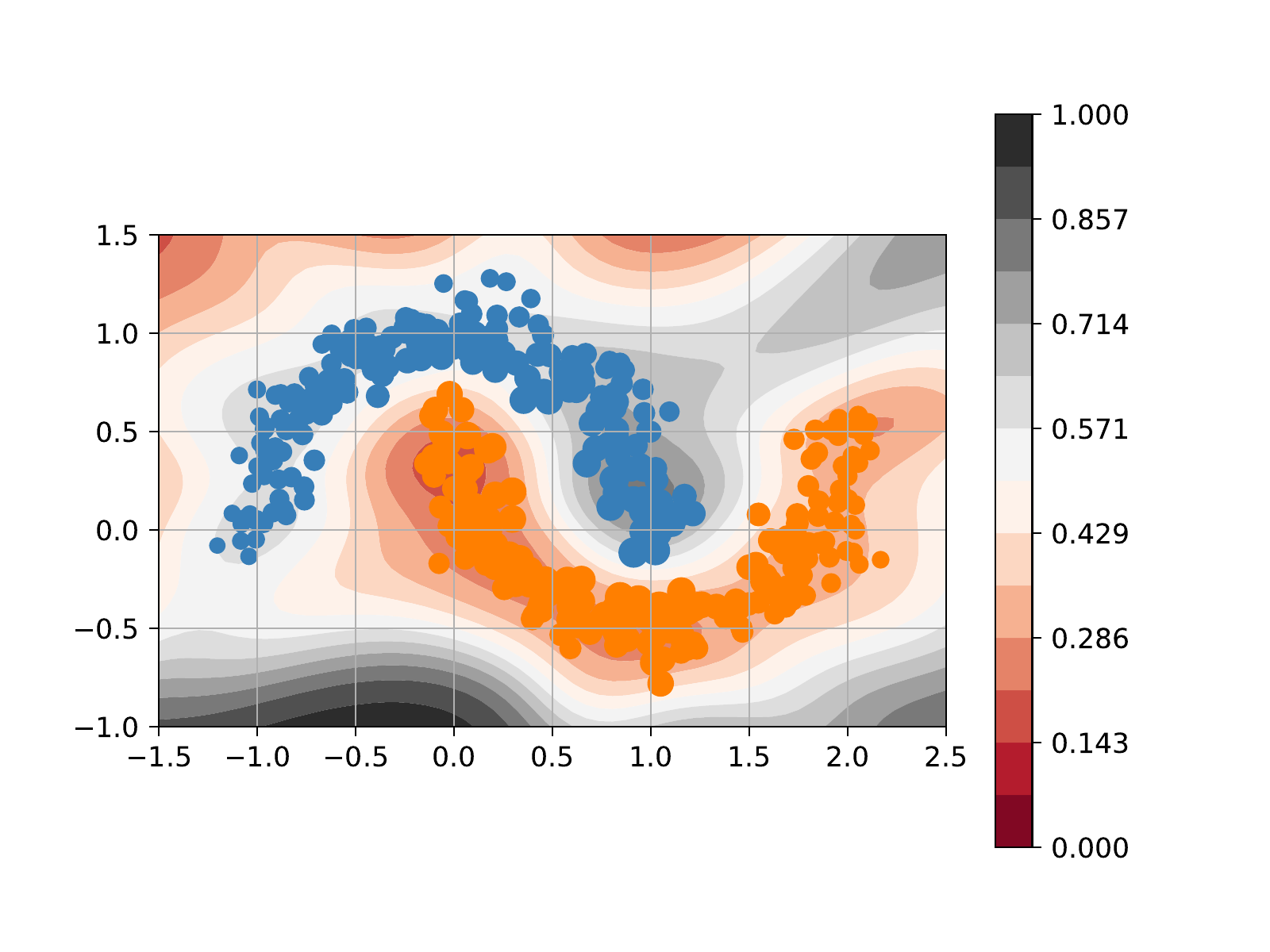}
		\caption{\tiny Manifold Lipschitz (ours)}
	\end{subfigure}
	\hfill
	\begin{subfigure}[b]{0.24\textwidth}
		\centering
		\includegraphics[width=\textwidth]{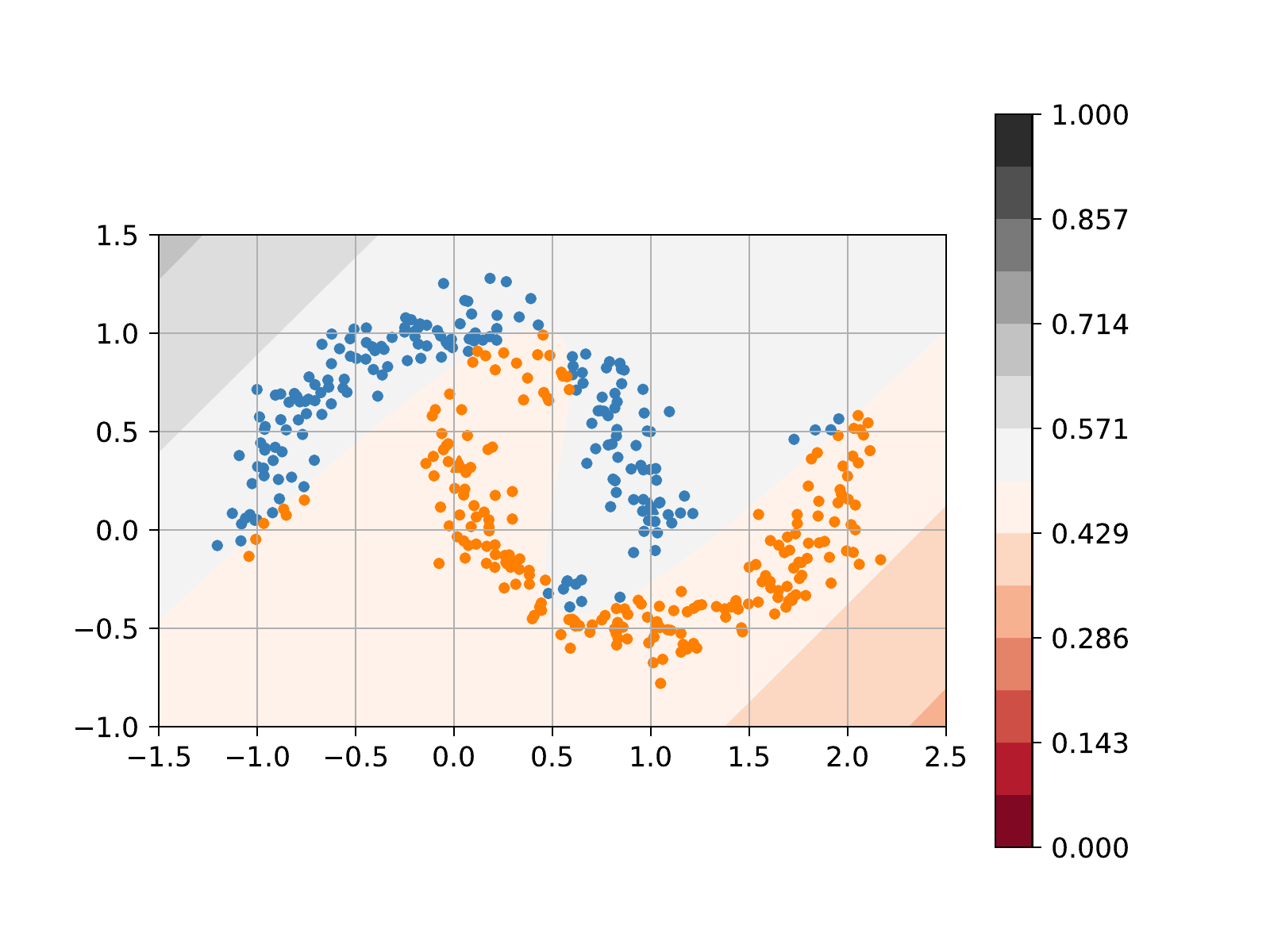}
		\caption{\tiny Manifold Regularization}
	\end{subfigure} 
	\hfill
	\begin{subfigure}[b]{0.24\textwidth}
		\centering
		\includegraphics[width=\textwidth]{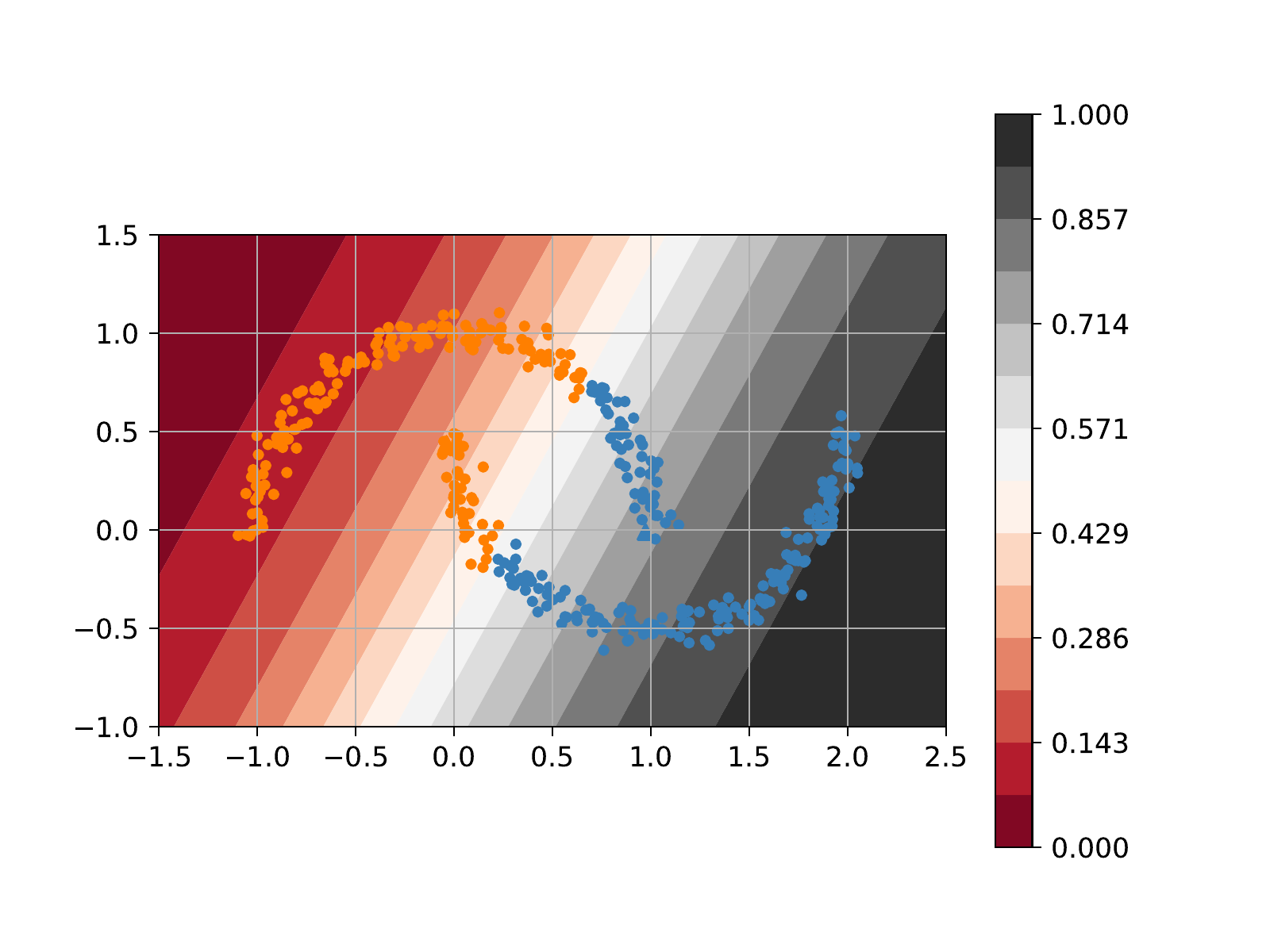}
		\caption{\tiny Ambient Regularization}
	\end{subfigure} 
	\caption{Two moons dataset. The setting consists of a two dimensional classification problem of two classes with $1$ labeled, and $200$ unlabeled samples per class. The objective is to correctly classify the $200$ unlabeled samples. We consider two cases (top) the estimated manifold has two connected components, and (bottom) the manifold is weakly connected (cf. Figure \ref{fig:two_moons}). We plot the output of a one layer neural network trained using Manifold Regularization, Manifold/Ambient Lipschitz.
	}
	\label{fig:two_moons}
\end{figure*}

In this paper we 
exploit the fact
that data can be often modeled as 
points in a low-dimensional
manifold. 
We therefore consider manifold Lipschitz constants in which function smoothness is assessed with respect to distances measured over the data manifold (Definition \ref{def_manifold_lipschitz}). Although this looks like a minor difference, controlling Lipschitz constants over data manifolds is quite different from controlling Lipschitz constants in the ambient space. In Figure \ref{fig:two_moons}, we look at a classification problem with classes arranged in two separate half moons. Constraining Lipschitz constants in the ambient space effectively assumes the underlying data is uniformly distributed in space [cf. Figure \ref{fig:two_moons}(d)]. Constraining Lipschitz constants in the data manifold, however, properly accounts for the data distribution [cf. Figure \ref{fig:two_moons}(a)]. 

This example also illustrates how constraining manifold Lipschitz constants is related to manifold regularization \citep{belkin2005manifold, niyogi2013manifold, li2022neural}. The difference is that manifold regularization penalizes the average norm of the manifold gradient. This distinction is significant because regularizing is more brittle than imposing constraints. In the example in Figure \ref{fig:two_moons}, manifold regularization fails to separate the dataset when the moons are close [cf. Figure \ref{fig:two_moons}-(c), bottom]. Classification with a manifold Lipschitz constant constraint is more robust to this change in the data distribution [cf. Figure \ref{fig:two_moons}-(a), bottom]. 

The advantages of constraining Lipschitz constants on manifolds that we showcase in the synthetic example of Figure \ref{fig:two_moons} manifest in the wild. To illustrate it, we empirically demonstrate this fact with two physical experiments. The first experiment entails learning a model for a differential drive ground robot in a mix of complex terrains. The second experiment consists of learning a dynamical model for a quadrotor. In both of these experiments constraining Lipschitz constants on manifolds improves upon standard risk minimization, manifold regularization, and the imposition of Lipschitz constraints in ambient space.

Global constraints on the manifold gradient yield a statistical constrained learning problem with an infinite and dense number of constraints. This is a challenging problem to approximate and solve. Here, we approach the solution of this problem in the Lagrangian dual domain and establish connections with manifold regularization that allow for the use of point cloud Laplacians. Our contributions include:


\begin{enumerate}[leftmargin=25pt,label=(C\arabic*)]
\item We introduce a constrained statistical risk minimization problem in which we learn a function that: (i) attains a target loss and (ii) attains the smallest possible manifold Lipschitz constant among functions that achieve this target loss (Section \ref{sec_problem_formulation}).

\item We introduce the Lagrangian dual problem and show that its empirical version is a statistically consistent approximation of the primal. These results do \emph{not} require the learning parametrization to be linear (Section \ref{sec_empirical_dual}).

\item We generalize results from the manifold regularization literature to show that under regularity conditions, the evaluation of manifold Lipschitz constants can be recast in a more amenable form utilizing a weighted point cloud Laplacian (Proposition \ref{P:pointcloud} in Section \ref{sec_laplacian}). 

\item We present a dual ascent algorithm to find optimal multipliers. The function that attains the target loss and minimizes the manifold Lipschitz constant follows as a byproduct (Section \ref{subsec:dual_ascent}).

\item We illustrate the merits of learning with global manifold smoothness guarantees with respect to ambient space and standard manifold regularization through two physical experiments: (i) learning model mismatches in differential drive steering over non-ideal surfaces and (ii) learning the motion dynamics of a quadrotor  (Section \ref{sec_numerical_results}).

\end{enumerate}